\documentclass[aos,preprint]{imsart}

\RequirePackage[OT1]{fontenc}
\RequirePackage{amsthm,amsmath}
\RequirePackage[numbers]{natbib}
\RequirePackage[colorlinks,citecolor=blue,urlcolor=blue]{hyperref}

\arxiv{arXiv:0000.0000}

\startlocaldefs
\numberwithin{equation}{section}
\theoremstyle{plain}

\endlocaldefs

\usepackage{times}		
\usepackage{minibox}		
\usepackage{amsfonts} 	
\usepackage{bbm}			
\usepackage[bf, format=plain, font={small, it}]{caption}
\usepackage{natbib}
\usepackage{graphicx}
\usepackage{subcaption}
\usepackage{algorithm2e}

\newcommand{\thisHypName}{}
\newtheorem*{genericHyp}{\thisHypName}
\newenvironment{namedHyp}[1]
  {\renewcommand{\thisHypName}{#1}%
   \begin{genericHyp}}
  {\end{genericHyp}}

\newcommand{\hideall}[1]{}

\newcommand{\spa}{\mathcal{X}}
\newcommand{\labSpa}{\mathcal{Y}}
\newcommand{\dist}{\rho}
\newcommand{\vcDim}{\mathcal{V}_{\mathcal{B}}}

\newcommand{\cBall}[2]{B(#1 , #2 )}
\newcommand{\ballColl}{\mathcal{B}}

\newcommand{\tJoinProb}{Q}
\newcommand{\sJoinProb}{P}
\newcommand{\featVar}{X}
\newcommand{\labVar}{Y}
\newcommand{\tProb}{\tJoinProb_{\featVar}}
\newcommand{\sProb}{\sJoinProb_{\featVar}}
\newcommand{\sCondProb}{\sJoinProb_{\labVar | \featVar}}
\newcommand{\tCondProb}{\tJoinProb_{\labVar | \featVar}}
\newcommand{\sReCondProb}[1]{\sJoinProb_{\labVar | #1}}
\newcommand{\tReCondProb}[1]{\tJoinProb_{\labVar | #1}}

\newcommand{\sDom}{\mathcal{X}_{\sJoinProb}}
\newcommand{\tDom}{\mathcal{X}_{\tJoinProb}}
\newcommand{\regFct}{\eta}

\newcommand{\family}{\mathcal{T}}
\newcommand{\dmFam}{\family_{\dm}}
\newcommand{\bcnFam}{\family_{\bcn}}

\newcommand{\totN}{n}
\newcommand{\nn}{k}

\newcommand{\sN}{n_{\sJoinProb}}
\newcommand{\tN}{n_{\tJoinProb}}
\newcommand{\sampleBase}{({\bf X, Y})}
\newcommand{\sample}{\sampleBase}
\newcommand{\sSample}{\sampleBase_{\sJoinProb}}
\newcommand{\tSample}{\sampleBase_{\tJoinProb}}
\newcommand{\featVectBase}{\bold{\featVar}}
\newcommand{\featVect}{\featVectBase}
\newcommand{\sFeatVect}{\featVectBase_{\sJoinProb}}
\newcommand{\tFeatVect}{\featVectBase_{\tJoinProb}}
\newcommand{\labVectBase}{\bold{\labVar}}
\newcommand{\labVect}{\labVectBase}
\newcommand{\sLabVect}{\labVectBase_{\sJoinProb}}

\newcommand{\err}{\text{err}_{\tJoinProb}}
\newcommand{\exErr}{\mathcal{E}_{\tJoinProb}}
\newcommand{\h}{h}
\newcommand{\hStar}{h^{*}}
\newcommand{\hHat}{\hat{\h}}

\newcommand{\hPQ}{\hat{h}_{k} }

\newcommand{\empReg}{\hat{\regFct}}

\newcommand{\nnFeat}[2]{\featVar_{(#1)}^{#2}}
\newcommand{\nnLab}[2]{\labVar_{(#1)}^{#2}}

\newcommand{\holderExp}{\alpha}
\newcommand{\holderCoeff}{C_{\holderExp}}
\newcommand{\tTsyCoeff}{C_{\tTsyExp}}
\newcommand{\tTsyExp}{\beta}

\newcommand{\domBnd}{D}

\newcommand{\tCovDim}{d}

\newcommand{\tCovCoeff}{C_{\tCovDim}}

\newcommand{\sDmDim}{d}
\newcommand{\tDmDim}{d}
\newcommand{\tDmCoeff}{C_{\tDmDim}}
\newcommand{\transMarginCoeff}{C_{\transMarginExp}}
\newcommand{\transMarginExp}{\gamma}

\newcommand{\bcn}{\text{(BCN)}}
\newcommand{\dm}{\text{(DM)}}

\newcommand{\diamDom}{\Delta_\spa}

\newcommand{\bigElement}{\Phi}

\newcommand{\dmLbConst}{c}
\newcommand{\s}{s}

\newcommand{\diff}{\, \text{d}}
\newcommand{\upConst}{C}
\newcommand{\lowConst}{c}
\newcommand{\rates}{d_0}

\newcommand{\M}{M}
\newcommand{\m}{m}
\newcommand{\alp}{\kappa}
\newcommand{\semiMetric}{\bar \rho}
\newcommand{\semiDist}[2]{\semiMetric \left( #1 , #2 \right)}
\newcommand{\KLDiv}[2]{\mathcal{D}_{\text{kl}} \left(#1 | #2 \right)}
\newcommand{\radius}{r}
\newcommand{\tempRadius}{r'}
\newcommand{\w}{w}
\newcommand{\wConst}{\lowConst_{\w}}
\newcommand{\rConst}{\lowConst_{\radius}}
\newcommand{\mConst}{\lowConst_{\m}}
\newcommand{\grid}{\mathcal{Z}}
\newcommand{\p}{z}
\newcommand{\ball}{B}

\newcommand{\lbda}{q}
\newcommand{\PDensity}{p}

\newcommand{\holderFct}{u}
\newcommand{\sig}{\sigma}
\newcommand{\hammingDist}[2]{\rho_{H}\left(#1,#2 \right)}
\newcommand{\sampleDist}{\Pi}
\newcommand{\modelClass}{\mathcal{H}}

\newcommand{\coverSet}{R}
\newcommand{\nnZero}{\nn_{0}}
\newcommand{\lowerBd}{\hat \eta^\text{-}}
\newcommand{\upperBd}{\hat \eta^\text{+}}
\newcommand{\nR}{\totN_{\coverSet}}
\newcommand{\est}{\hat \eta }

\DeclareMathOperator*{\Expectation}{\mathbb{E}}

\newcounter{myremark}
\newenvironment{myremark}[1][]{\refstepcounter{myremark}\medskip
 \noindent \textbf{Remark~\themyremark} {\hspace{0.1mm} (#1)} \rmfamily}{\medskip}

 \newtheorem{mydefinition}{Definition}
 \newtheorem{assumption}{Assumption}
\newtheorem{example}{Example} 
\newtheorem{theorem}{Theorem}
\newtheorem{lemma}{Lemma}
 \newtheorem{myproposition}{Proposition}

\begin{document}

\begin{frontmatter}
\title{Marginal Singularity, and the Benefits of Labels in Covariate-Shift} 
\runtitle{Minimax Transfer}

\begin{aug}
\author{\fnms{Samory} \snm{Kpotufe}\thanksref{t1,t2,m1}\ead[label=e1]{skk2175@columbia.edu}},
\author{\fnms{Guillaume} \snm{Martinet}\thanksref{t1, t2}\ead[label=e2]{ggm2@princeton.edu}}

\thankstext{t1}{Authors are listed in alphabetic order.}
\thankstext{t2}{Author was at Princeton University, ORFE, for a major portion of the project.}
\runauthor{Kpotufe and Martinet}

\affiliation{Columbia University, Statistics\thanksmark{m1} and Princeton University, ORFE\thanksmark{t2}}

\address{Address of the First and Second authors\\
Usually a few lines long\\
\printead{e1}\\
\phantom{E-mail:\ }\printead*{e2}}

\end{aug}

\begin{abstract}
\emph{Transfer Learning} addresses common situations in Machine Leaning where little or no labeled data is available for a target prediction problem -- corresponding to a distribution $Q$, but much labeled data is available from some related but different data distribution $P$. 
This work is concerned with the fundamental limits of transfer, i.e., the limits in target performance in terms of (1) sample sizes from $P$ and $Q$, and (2) differences in data distributions $P, Q$. In particular, we aim to address practical questions such as how much \emph{target} data from $Q$ is sufficient given a certain amount of related data from $P$, and how to optimally \emph{sample} such target data for labeling.

We present new minimax results for transfer in \emph{nonparametric} classification (i.e. for situations where little is known 
about the target classifier), under the common assumption that the marginal distributions of covariates differ between $P$ and $Q$ (often termed \emph{covariate-shift}). Our results are first to concisely capture the relative benefits of source and target labeled data in these settings through information-theoretic limits. Namely, we show that the benefits of target labels are tightly controlled by a \emph{transfer-exponent} $\gamma$ that encodes how \emph{singular} $Q$ is locally with respect to $P$, and interestingly paints a more favorable picture of transfer than what might be believed from insights from previous work. In fact, while previous work rely largely on refinements of traditional metrics and divergences between distributions, and often only yield a coarse view of when transfer is possible or not, our analysis -- in terms of $\gamma$ -- reveals a \emph{continuum of new regimes} ranging from easy to hard transfer. 

We then address the practical question of how to efficiently sample target data to label, by 
showing that a recently proposed semi-supervised procedure -- based on $k$-NN classification, can be refined to adapt to unknown $\gamma$, and therefore requests target labels only when beneficial, while achieving nearly minimax-optimal transfer rates without knowledge of distributional parameters. Of independent interest, we obtain new minimax-optimality results for vanilla $k$-NN classification in regimes with non-uniform marginals. 
\end{abstract}

\begin{keyword}[class=MSC]
\kwd[Primary ]{60K35}
\kwd{60K35}
\kwd[; secondary ]{60K35}
\end{keyword}

\begin{keyword}
\kwd{sample}
\kwd{\LaTeXe}
\end{keyword}

\end{frontmatter}

\section{Introduction}
Transfer learning addresses the many practical situations where much labeled data is available from a \emph{source} distribution $P$, but relatively little labeled data is available from a \emph{target} distribution $Q$. The aim is to harness 
source data to improve prediction on the target $Q$, assuming the source $P$ is informative about $Q$. Therefore, the main goal in transfer is to use as few target labels as possible, as these are typically expensive or hard to obtain in motivating applications: 
in Speech or Image Processing, much data might be available from a given population, while collecting and labeling speech or image data from a new target population is typically expensive. 
Typically, practitioners do not know a priori how related the two data distributions are, and therefore are left guessing how much labeled target data is needed to attain a desired prediction performance. Such basic questions motivate this work: we aim to quantify the benefits of labeled target data, given a certain amount of \emph{related} source data, and furthermore yield advice on how to efficiently select target data to label.

Naturally, a main theoretical question is in understanding relations (or the amount of divergence) between $P$ and $Q$ that allow information transfer, and in particular, which characterize the relative benefits of source and target labeled samples and thus help inform practice.   

We focus on the problem of classification, i.e., predicting labels $Y$ of future $X$ drawn from $Q$, in \emph{nonparametric} regimes, i.e., assuming little knowledge of the classification patterns encoded by $P$ and $Q$. We adopt the most common transfer setting in the literature, i.e., that of \emph{covariate-shift} where $P_{Y|X} = Q_{Y|X}$, allowing $Q_X$ to differ from $P_X$. While equal conditionals may seem restrictive, it is well motivated by driving applications of transfer (e.g. image, speech, or document classification) where covariates determine the label distribution. 

Our first theoretical question is then how to capture those differences between marginals $P_X, Q_X$ that are relevant to transfer.  While this question is not new, much of the existing work has focused on refinements of traditional metrics and divergences between distributions, which unfortunately have so far been unable to capture the relative benefits of labeled source and target data. In fact, as we will show, traditional measures of change between distributions (e.g. total-variation and common refinements, Wasserstein distance, KL-divergence) paint an overpessimistic view of transfer (at least in the settings of interest here) as they are inherently designed for other purpose (see examples and Remark \ref{remark:divergences} in Section \ref{sec:transferexponent}). 

We present new minimax results that concisely capture the relative benefits of source and target labeled data, under {covariate-shift}. Namely, we show that the benefits of target labels are controlled by a \emph{transfer-exponent} $\gamma$ that encodes how \emph{singular} $Q_X$ is locally with respect to $P_X$, and interestingly allows situations where transfer did not seem possible under previous insights. In fact, our new minimax analysis -- in terms of $\gamma$ -- reveals a \emph{continuum of regimes} ranging from situations where target labels have little benefit, to regimes where target labels dramatically improve classification. 

The notion of transfer-exponent follows a natural intuition, also present in prior work, that transfer is hardest if $P_X$ does not properly cover regions of large $Q_X$ mass. In particular, $\gamma$ parametrizes the behavior of ball-mass ratios 
$Q(B_r)/P(B_r)$ as a function of neighborhood size $r$ (see Definition \ref{def:transferCoefficient}), namely, that these ratios behave like $r^{-\gamma}$. We will see, through both lower and upper-bounds, that transfer is easiest as $\gamma \to 0$ and hardest as $\gamma \to \infty$. There are two essential departures from more traditional measures: first, $\gamma$ {\color{black} is not symmetric}, i.e., $P$ might have information about $Q$ but not the other way around -- which is natural to expect in hindsight (e.g. if the support of $Q_X$ is a proper subset of that of $P_X$), and exemplifies the inadequacy of traditional \emph{metrics} between probability measures at capturing transfer; second, but more subtle, is that $\gamma$ accounts for neighborhood size (through $r$), which negates artifacts of the resolutions at which distributions are compared.

Interestingly, $\gamma$ is well defined even when $Q$ is singular with respect to $P$ -- in which case common notions 
of \emph{density-ratio} and information-theoretic divergences (KL or Renyi) fail to exist. We note that singularity of $Q$ with respect to $P$ is likely common in practice where high-dimensional data is often very structured, and transfer often involves going from a generic dataset from a domain $P$ to a more structured subdomain $Q$ {\color{black}(e.g. going from a generic image repository -- say of a city, to an application with less variety in images -- say mostly stop signs)}. Here, our results can directly inform practice: target labels yield greater performance the lower the dimension of $Q_X$ w.r.t. that of $P_X$; if $Q$ were of higher dimension than $P$, the benefits of source labels quickly saturate. Now when $Q$ and $P$ are of the same dimension, even 
sharing the same support, the notion of $\gamma$ reveals yet a rich set of regimes where transfer is possible at different rates, while more traditional notions might indicate otherwise. 

As stated earlier, the practical question motivating this work, is whether, given a large database of source data, acquiring additional target data might further improve classification; this is usually difficult to test given the costs and unavailability of target data. Here, by capturing the interaction of source and target sample sizes in our rates, in terms of $\gamma$, we can sharply characterize those sampling regimes where target or source data are most beneficial. We then show that it is in fact possible to \emph{adapt} to unknown $\gamma$, i.e., request target labels only when beneficial, while also attaining near- optimal rates in terms of unknown distributional parameters. 

\subsection*{Detailed Results and Related Work}
Many interesting notions of divergence have been proposed that successfully capture a general sense of when transfer is possible. In fact, the literature on transfer is by now expansive, and we cannot hope to truly do it justice. 

A first line of work considers refinements of total-variation that encode changes in error over the classifiers being used (as defined by a hypothesis class $\mathcal{H}$). The most common such measures are the so-called $d_{\mathcal{A}}$-divergence \citep{ben2010theory, david2010impossibility, germain2013pac} and $\mathcal{Y}$-discrepancy \citep{mansour2009domain, mohri2012new, cortesadaptation}. These notions are the first to capture -- through \emph{differences} in mass over space -- the intuition that transfer is easiest when $P$ has sufficient mass in regions of substantial $Q$-mass. Typical excess-error bounds on classifiers learned from source data (and perhaps some target data) are of the form $$o_p(1) + C\cdot \text{distance}(P, Q).$$  In other words, 
transfer seems impossible when these divergences are large; this is certainly the case in very general situations. However, as we show, there are ranges of reasonable situations ($0\leq \gamma < \infty$) where transfer is possible, even at fast rates $o(1/\sqrt{n})$ without the benefit of target data, while the above metrics remain uncharacteristically large (see Remark \ref{remark:divergences} of Section \ref{sec:transferexponent}). Also, as discussed earlier, \emph{metrics} on $P, Q$ 
carry the wrong intuition that transfer is symmetric, i.e., a metric treats the difficulty of transfer equally in both directions. 

Another prominent line or work, which has led to many practical procedures, considers so-called ratios of densities $f_{Q_X}/f_{P_X}$ or similarly Radon-Nikodym derivatives $d{Q_X}/d{P_X}$ as a way to capture the similarity between $P$ and $Q$ 
\citep{quionero2009dataset, sugiyama2012density}. It is often assumed in such work that $d{Q_X}/d{P_X}$ is bounded 
which corresponds to the regime $\gamma = 0$ in our case (see Example \ref{ex:boundedDensity} of Section \ref{sec:transferexponent}). Typical excess-error bounds are dominated by the estimation rates for $d{Q_X}/d{P_X}$ (see e.g. rates for $\alpha$-H\"older $d{Q_X}/d{P_X}$, $\alpha\to 0$, in \cite{kpotufe2017lipschitz}), which unfortunately could be \emph{arbitrarily} higher than the achievable rates we establish for the corresponding setting with $\gamma = 0$. Furthermore, as previously mentioned, $d{Q_X}/d{P_X}$ is ill-defined in common scenarios with structured data, 
or can be unbounded even while $\gamma$ remains small (see Example \ref{ex:unboundedDensity} of Section \ref{sec:transferexponent}). 

Another line of work, instead considers information-theoretic measures such as KL-divergence or Renyi divergence 
\citep{sugiyama2008direct, mansour2009multiple}. In particular, such divergences are closer in spirit to our notion of 
transfer-exponent $\gamma$ (viewing it as roughly characterizing the log of ratios between $Q_X$ and $P_X$). However, similarly to density ratios, these divergences are undefined in typical scenarios with structured data. 

Our upper-bounds are established for the now classical nonparametric settings of \cite{audibert2007fast}, which parametrize the interaction between label variable $Y$ and covariates $X$ via smoothness and noise conditions; this allows us to understand the interaction between $\gamma$ and traditional classification-complexity parameters, and capture regimes where target classification remains easy despite large $\gamma$. Our minimax upper-bounds are achieved by a generic $k$-NN classifier defined over the combined source and target sample. In particular, our results imply new convergence rates of independent interest for vanilla $k$-NN in the context of non-uniform marginals (see Remark \ref{rem:newKNNbounds}, Section \ref{sec:upperbounds}). 

Our lower-bounds are established over any learner with access to both source and target samples, and interestingly, allow the learner \emph{access to infinite unlabeled source and target data} (i.e., is allowed to know $P_X$ and $Q_X$). In other words, our lower-bounds imply that, at least in a minimax sense, unlabeled data only has marginal benefits in transfer, which is interesting as much research efforts has gone into leveraging unlabeled in various aspects of learning (see e.g. \citep{huang2007correcting, ben2012hardness} in the context of transfer). 

Finally, we address efficient target sampling in the context of \emph{semisupervised} or \emph{active} transfer, 
where, given labeled source data and unlabeled target data, the goal is to request as few target labels as possible to improve 
over using source data alone \citep{saha2011active, chen2011co, pmlr-v28-chattopadhyay13}. An early theoretical treatment can be found in \citep{yang2013theory}, but which however considers a fundamentally different setting with fixed marginal ($P_X = Q_X$) but varying conditionals ($P_{Y|X} \neq Q_{Y|X}$). The recent work of \cite{berlind2015active} gives a nice first theoretical treatment of the problem under similar nonparametric conditions as ours; however their work is less concerned with understanding the information-theoretic limits of the problem, but rather in deriving algorithmic strategies towards reducing label requests. We will build on their algorithmic insights, and show how to refine their main procedure to achieve near minimax transfer rates, \emph{without prior knowledge} of distributional parameters, while requesting target labels only when necessary, i.e., when the unknown $\gamma$ is large with respect to source sample size.  

We note that a conference abstract of the work appeared earlier \cite{pmlr-v75-kpotufe18a}. The abstract gives a combined \emph{sketch} of Theorems \ref{thm:minimaxLowerBound} and \ref{thm:expErrRates}, but provides neither rigorous statements of these results, nor their analysis; furthermore, it does not provide any of the adaptive rates and sampling results of Theorems \ref{thm:genericadaptivity} and \ref{thm:labelingComplexity}.

\section*{Paper Outline}
We start with definitions in Section \ref{sec:prelim}, followed by an overview of results in 
Section \ref{sec:overview}. We provide discussions and detailed proofs of some of the lower-bounds in Section \ref{sec:lowerboundAnalysis}, followed by the essential ingredients of the upper-bound analysis in Section \ref{sec:upper-bounds}. Remaining proofs, along with extended settings, are covered in the appendix (supplementary material).

\section{Preliminaries}
\label{sec:prelim}
\subsection{Basic Distributional Setting}
We consider a classification setting where the input variable $X$ belongs to a compact metric space $(\spa, \dist)$ of diameter $\diamDom$, and the label variable $Y$ belongs to $\labSpa \equiv \{0,1\}$. We consider a {\em source} distribution $\sJoinProb$ and a {\em target} distribution $\tJoinProb$ over $\spa \times \labSpa$. We let $\sProb, \tProb,\sCondProb, \tCondProb$ denote the corresponding marginal and conditional distributions. 

We assume the common {\em covariate-shift} setting, where marginals \emph{shift} from source to target, although conditionals remain the same. This is formalized below. 

\begin{mydefinition}[Covariate-shift]\label{def:covShiftSet}
There exists a measurable $\regFct : \spa \rightarrow [0,1]$, called {\em regression function}, such that
$\sReCondProb{x}= \tReCondProb{x} = \regFct(x)$ a.s.~$\sProb$ and $\tProb$.%
\end{mydefinition}


\subsection{Classifiers under Transfer}
{The learner has access to labeled data 
$$\sSample \equiv \left \{ (\featVar_i , \labVar_i) \right \}_{i = 1}^{\sN}\sim P^{\sN},\text{ and }\tSample \equiv \left \{ (\featVar_i , \labVar_i) \right \}_{i = \sN+1}^{\sN + \tN}\sim Q^{\tN},$$ independent of $\sSample$. We let $\sample \equiv \sSample \cup \tSample$. 
We only assume that $(\sN\vee \tN) \geq 1$, although the regime $0\leq \tN \leq \sN$ is often most meaningful in applications of transfer learning. 

For any classifier $\h: \spa \rightarrow \{0,1\}$ learned over $\sample$, we are interested in the target error 
$\err(h) \equiv \mathbb{E}_{\tJoinProb} \mathbbm{1}\{h(X)\neq Y\}$. This is minimized by the {\em Bayes} classifier $\hStar(x) \equiv \mathbbm{1} \{\regFct(x)\geq 1/2\}$. Our results concern the best error achievable by any classifier $\h$ \emph{in excess} over the error of $\hStar$. }

\begin{mydefinition} \label{def:excessError}
The {\bf excess error} of a classifier $h$, under $Q$, is defined as: 
\begin{equation}
\exErr(\h) \equiv \err(\h)-\err(\hStar) = 2 \mathbb{E}_{Q}  \left|\regFct(X)-\frac{1}{2}\right| \cdot\mathbbm{1}\{\h(X) \neq \hStar(X)\}. \label{eq:excesserror}
\end{equation}
\end{mydefinition}

Our minimax analysis aims to upper and lower-bound $\exErr(\hat h)$ -- in expectation over $P^{\sN}\times Q^{\tN}$, over any possible learner $\hat h$\footnote{We will at times conflate the learner $\hat h:\sample \mapsto 2^\spa$ with its output classifier $\hat h\in 2^\spa$.}, so as to capture the separate contributions of $\sN$ and $\tN$ to the rates. 

\subsection{Transfer-exponent (from $P_X$ to $Q_X$)}
\label{sec:transferexponent}
\begin{figure} 

\centering 
\includegraphics[height=4.5cm]{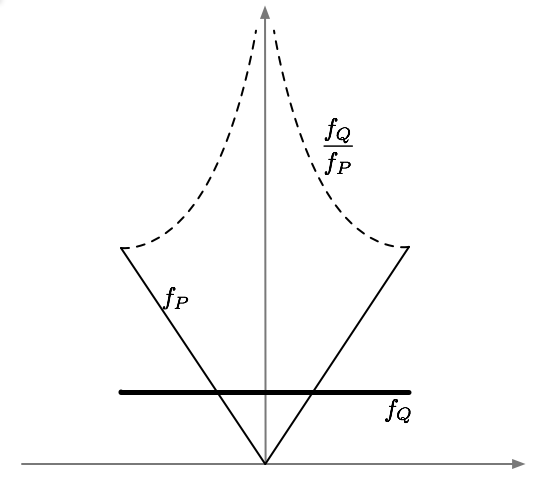}
\includegraphics[height=4.5cm]{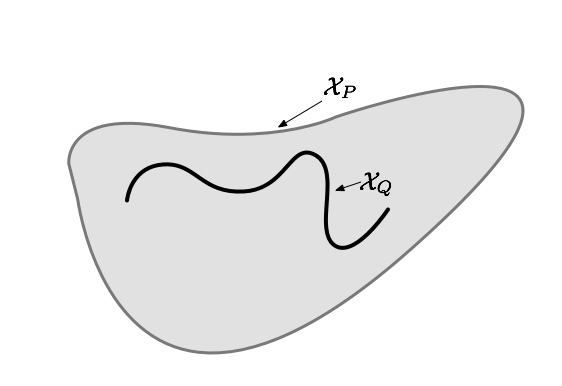}
\caption{\noindent Some settings with $0< \gamma < \infty$. Left: the density $f_P\propto |x|^\gamma$ goes fast to $0$, while $f_Q$ is uniform on the same support; $f_Q/f_P$ then diverges, but $\gamma$ is well-defined. Right: $Q_X$ has lower-dimensional support $\mathcal{X}_Q$; $\gamma$ then captures the difference in dimensions. This last case also illustrates the interesting fact that \emph{transfer} might be possible from $P$ to $Q$ but not from $Q$ to $P$ ($\gamma = \infty$ when $P$ is the target).}\label{fig:gamma}
\end{figure}
Intuitively, transfer is harder if we are likely to see little data from $P_X$ near typical points $X\sim Q_X$. In other words, for easy transfer from $P$ to $Q$, we want $P_X$ to give reasonable mass to those regions of non-negligible $Q_X$ mass. We aim to parametrize this intuition.  

Let $\cBall{x}{r}$ denote the closed ball $\{x'\in \spa: \dist(x, x')\leq r\}$. Let $\sDom$ denote the support of $\sProb$, 
i.e., $\sDom \doteq \{ x \in \spa: \sProb(\cBall{x}{r})>0, \forall r >0 \}$, and similarly define $\tDom$ as the support of $\tProb$. Remark that because $(\spa,\dist)$ is compact and hence separable, we have $\sProb(\sDom) = \tProb(\tDom) = 1$. 

\begin{mydefinition} 
\label{def:transferCoefficient}
$(P, Q)$ has {\bf transfer-exponent} 
$\transMarginExp\in \mathbb{R}_+ \cup \{0, \infty \}$, if there exists a constant $\transMarginCoeff \in (0,1]$, and 
a region $\tDom^\gamma \subset \tDom$, $Q_X (\tDom^\gamma) = 1$, such that:
\begin{equation} \label{eq:ass1equation}
\forall x \in \tDom^\gamma, \forall  r \in (0,\diamDom], \quad \sProb(\cBall{x}{r}) \geq \tProb(\cBall{x}{r}) \cdot \transMarginCoeff  \left( \frac{r}{\diamDom} \right)^{\transMarginExp}.
\end{equation}
\end{mydefinition}

First, notice that every pair $(P, Q)$ satisfies the above with at least $\gamma = + \infty$, since the condition in \eqref{eq:ass1equation} then just defaults to $P_X(B(x, r))\geq 0$. Second, if \eqref{eq:ass1equation} holds for some $\gamma$, then it holds for any $\gamma' > \gamma$; our results are therefore to be understood as holding for the smallest admissible such transfer-exponent $\gamma$. We will see that transfer-learning gets easier with smaller $\gamma$, i.e., achievable rates depend more on $n_P$ and less on $n_Q$ as $\gamma \to 0$. In particular for $\gamma = 0$,  
we need a number of target labels $\tN \gg \sN$ to get any speedup beyond the rates achievable with $\tN = 0$ target labels. For $\gamma = \infty$, we have nearly no transfer, i.e., $\sN$ has little effect on achievable rates. 

Next, to get a sense of the applicability of the above definition, let's consider some examples of situations with different transfer-exponents, including the boundary cases $\gamma = \infty$ and $\gamma = 0$.  As it turns out, these boundary cases encompass much of the usual regimes covered by previous analyses.

\begin{example}[Disjoint supports, or higher-dimensional target]\label{ex:disjointSupports}
Suppose $\tDom\setminus\sDom \neq \emptyset$. Then $\gamma = \infty$ since for any $x\in\tDom\setminus\sDom$, 
$\exists r>0$ s.t. $P(B(x, r)) = 0$ while $Q(B(x, r))>0$.  An important such case in practice is when the support $\tDom$ is of higher dimension than $\sDom$. As we'll see, source-labeled data have minimal benefits in such cases (beyond improving constants) as discussed above. 
\end{example}

\begin{example}[Bounded density ratio $dQ_X/dP_X$] \label{ex:boundedDensity}
Let $Q_X$ be absolutely continuous with respect to $P_X$ and therefore admit a density (Radon-Nikodym derivative) 
$dQ_X/dP_X$ with respect to $P_X$. If $dQ_X/dP_X \leq C$, we then have $\gamma = 0$, since for any ball $B$ we have 
$Q_X(B) = \int_B \frac{dQ_X}{dP_X}dP_X \leq C\cdot P_X(B)$. 
Arguably, this is the most studied case in transfer under covariate-shift. 
\end{example}

\begin{example}[Unbounded density ratio $dQ_X/dP_X$] \label{ex:unboundedDensity}
Again, let $Q_X$ admit a density $dQ_X/dP_X$ with respect to $P_X$. However we now allow $dQ_X/dP_X$ to diverge at some points or regions in space; how fast it diverges is then controlled by $\gamma$. This is a sense in which we might view $\gamma$ as encoding a degree of singularity of $Q_X$ with respect to $P_X$. 
Here is a concrete example (see also Figure \ref{fig:gamma}): 

Let $Q_X$ be uniform on $([-1, 1], \rho \doteq |\cdot|)$ (or have bounded Lebesgue density), and let $P_X$ have Lebesgue density $f_P(x)\propto |x|^\gamma$ on $[-1, 1]$. Then $dQ_X/dP_X = 1/(2f_P)$ and diverges at $x=0$. It is immediate that \eqref{eq:ass1equation} holds for any ball centered at $x=0$. It is not hard to check however that \eqref{eq:ass1equation} holds at all $x \in [-1, 1]$ since $P_X$ can only assign higher mass away from $0$.

Following the above example, we can see that $\gamma = \infty$ happens when $dQ_X/dP_X$ diverges at a rate faster than polynomial (e.g. let $f_P(x) \propto \exp(-1/|x|)$).  Such fast divergence in $dQ_X/dP_X$ happens for instance if $P_X$ and $Q_X$ are sufficiently separated Gaussians, in which case transfer can be hard. In fact, the example of two Gaussians was given earlier in \citep{cortes2010learning} as an example of hard transfer for importance-sampling approaches; our present results indicate that, in a minimax sense, such situations are hard irrespective of the learning approach. 
\end{example}

\begin{example}[$Q_X$ has lower-dimension] \label{ex:diffDim}
Suppose that $Q_X$ has support $\tDom$ of dimension $d_Q$, while $P_X$ has support $\sDom$ of dimension $d_P\geq d_Q$ (Figure \ref{fig:gamma}). In a generic metric space, this would be formalized with respect to the mass assigned to balls as $Q_X(B(x, r)) \propto r^{d_Q}$ while $P_X(B(x, r))\propto r^{d_P}$ for $x \in \tDom$ (see e.g. Definition \ref{def:doublingMeas}), following similar intuition for Euclidean spaces. It is then direct that we would have 
$\gamma = d_P - d_Q$. This is again a sense in which $\gamma$ encodes the \emph{strength} of singularity of $Q_X$ with respect to $P_X$. We'll then see that the smaller $d_Q\ll d_P$, the more useful target labels are.   
\end{example}

We remark that, in practice, $\gamma$ might capture any mix of the above examples, and as such, can be viewed as measuring the \emph{degree} to which $Q_X$ is close to singular with respect to $P_X$ (if $Q_X \ll P_X$ as in Example \ref{ex:unboundedDensity}), or otherwise the \emph{strength} of such singularity (Example \ref{ex:diffDim}). 

\begin{myremark}[Other divergences can be pessimistic] 
\label{remark:divergences}
Common notions of dissimilarity used in transfer take the form 
$\text{div}(Q_X, P_X) = \sup_{A \in \mathcal{A}} |Q_X(A) - P_X(A)|$, where $A \in \mathcal{A}$ are subsets of $\mathcal{X}$ encoding classification decisions (indicators over classifiers $h$ in a fixed set $\mathcal{H}$, or their symmetric differences; see e.g., $d_{\cal A}$ and $d_{\cal Y} $ divergences of \cite{ben2010theory, cortesadaptation}. For common families $\mathcal{A}$ we would have $\text{div}(Q_X, P_X) \geq 1/2$ (leading to vacuous transfer rates), while we'll see that nontrivial transfer remains possible ($0< \gamma < \infty$). This will be the case for instance when $\tDom$ is of lower-dimension than $\sDom$ as in Example \ref{ex:diffDim} above: suppose for instance that $P_X$ is uniform on a cube $[0, 1]^{d_P}$, and $Q_X$ is uniform on a hyperplane through the cube; if $\mathcal{A}$ is all half-spaces (encoding linear separators or their symmetric differences) it's then clear that $\text{div}(Q_X, P_X) \geq 1/2$ while $\gamma = 1$. In fact, even when $P$ and $Q$ have the same support (hence dimension, as in Example \ref{ex:boundedDensity} or \ref{ex:unboundedDensity}), we can construct similar situations where $\text{div}(Q_X, P_X)$ is large, simply by assigning different masses to appropriately chosen $A \in \mathcal{A}$, while allowing small $\gamma$. 

In fact, \emph{metrics} (e.g., $d_{\cal A}$, total variation, Wasserstein, etc ...) cannot adequately capture transfer: we emphasize that transfer is inherently an asymmetric problem, namely, $P$ might have much information on $Q$ but not the other way around. This is clear for instance in Example \ref{ex:diffDim}, where $\gamma (P\to Q) < \infty$, but $\gamma(Q\to P) = \infty$. Similarly in Example \ref{ex:unboundedDensity}, $\gamma(P\to Q) >0$ but it is easy to see that $\gamma(Q\to P) = 0$. 

Information-theoretic divergences (Renyi or Kullback Leibler (KL)) seem related to $\gamma$ if not only for the fact that 
$\gamma$ serves to characterize the behavior of $\log Q_X(B(x, r))/P_X(B(x, r))$ as $r\to 0$. In particular, for the distributions in Examples \ref{ex:boundedDensity}, \ref{ex:unboundedDensity} above, it is easy to check that 
KL-divergence remains small with small $\gamma$, and diverges for the examples with $\gamma = \infty$. However, the exact relations between such divergences and $\gamma$ (when $dQ_X/ dP_X$ exists) remain unclear and worth further study. Nonetheless, the notion of $\gamma$ captures more general situations, as it remains well-defined even when $Q_X$ is singular with respect to $P_X$. 

\end{myremark} 



\subsection{Classification Regimes}
\label{sec:regimes} 
We consider transfer under two nonparametric classification regimes introduced in \citep{audibert2007fast}. Both regimes similarly parametrize the behavior of $\eta(x) = \mathbb{E}[Y|x]$ near the boundary $1/2$, but differ in their regularity assumptions on $Q_X$, i.e., in whether $Q_X$ properly covers its support $\tDom$ or not.  
These regimes capture the hardness of classification with respect to $Q_X$, while the transfer-exponent $\gamma$ of earlier, captures the hardness of transfer from $P$ to $Q$. 

\subsubsection*{Smoothness of $\eta$ and Low Noise Conditions}

\begin{mydefinition}[Smoothness] \label{def:smoothness}
The {\em regression function} $\regFct$ is $(\holderCoeff, \holderExp)$--H\"{o}lder for $\holderExp \in (0,1]$, $\holderCoeff > 0$, if $\, \forall x,x' \in \spa, \quad |\regFct(x) - \regFct(x')| \leq \holderCoeff \cdot \dist(x,x')^\holderExp$. 
\end{mydefinition}

Next,  we characterize how likely it is for $\eta$ to be close to $1/2$ under $Q_X$. 

\begin{mydefinition} [Tsybakov's noise condition for $\tJoinProb$] \label{def:noise} 
$Q$ has noise parameters $\tTsyExp, \tTsyCoeff>0$, if $\forall t \geq 0$, 
$\quad \tProb \left( 0 < \left| \regFct(\featVar) - \frac{1}{2} \right| \leq t \right) \leq \tTsyCoeff t^{\tTsyExp}$. \emph{The larger $\beta$, the easier the classification task.} Note that the above always hold for any $Q$ with at least $\tTsyExp = 0$ and $C_0 = 1$.
\end{mydefinition} 
\subsubsection*{Regimes, and Dimension of $X \sim Q_X$}
We now present the two classification regimes. The first regime \dm, ensures that $Q_X$ has near \emph{uniform} mass, and corresponds to the \emph{strong-density} condition of \cite{audibert2007fast}, and holds for instance for \emph{doubling measures} where $\forall 0< r < \Delta_X$, $Q_X(B(x, r/2)\geq C\cdot Q_X(B(x,r) $ for some $C$.

\begin{mydefinition} [Bounded Mass] \label{def:doublingMeas}
We say that $\tProb$ is $(\tDmCoeff,\tDmDim)$-doubling, 
for $\tDmCoeff \in (0, 1]$ and $d\geq 1$, if $\forall r \in [0,\diamDom], \forall x \in \tDom, \quad \tProb(B(x,r)) \geq 
\tDmCoeff \left({r}/{\diamDom} \right)^{d}$.
\end{mydefinition}
The first classification regime is thus formalized as follows:
\begin{namedHyp}{\dm} The regression function $\eta$ is $(\holderCoeff, \holderExp)$--H\"{o}lder, and $Q$ has noise parameters $\tTsyExp, \tTsyCoeff$. Furthermore, $\tProb$ is $(\tDmCoeff,\tDmDim)$-doubling.
\end{namedHyp}

Classification is easiest in this regime, and so turns out to yield faster transfer rates. The quantity $d$ plays the role of the \emph{dimension} of the input $X\sim Q_X$ (think for instance of $Q_X \equiv \mathcal{U}([0, 1]^d), \ell_\infty)$). 

The second regime \bcn, allows arbitrary $Q_X$, and therefore results in harder classification, and also slower transfer, as we will see. For this regime, the following regularity conditions (and quantity $d$) serve to capture the \emph{dimension} of the support $\tDom$. Recall, that the $r$-covering number of a pre-compact set $\tDom$, denoted $\mathcal{N}(\tDom,\dist, r)$, is the smallest number of $\dist$-balls {of radius $r$} needed to cover $\tDom$.

\begin{mydefinition} [Bounded covering number] \label{def:boundedCov}
 $\tDom$ is said to have $(\tCovCoeff,\tCovDim)$-bounded covering number, for $\tCovDim \geq 1$, $\tCovCoeff \geq 1$, if $\ \forall r \in (0,\diamDom], \, \, \mathcal{N}(\tDom,\dist,r) \leq \tCovCoeff \left( {\diamDom}/{r} \right)^{d}$.
\end{mydefinition}
The second regime is thus formalized as follows. 
\begin{namedHyp}{\bcn} The regression function $\eta$ is $(\holderCoeff, \holderExp)$--H\"{o}lder, and $Q$ has noise parameters $\tTsyExp, \tTsyCoeff$. Furthermore, $\tDom$ has $(\tCovCoeff,\tCovDim)$-bounded covering number.
\end{namedHyp}

The above two parameters, together with the transfer-exponent $\transMarginExp$, characterize the classes of distribution tuples $(\sJoinProb, \tJoinProb)$ considered in this work. 

\begin{mydefinition} [Transfer classes] \label{def:distClass}
Fix parameters $(\transMarginCoeff, \transMarginExp, \holderCoeff, \holderExp, \tTsyCoeff, \tTsyExp, \tDmCoeff, \tDmDim)$ as in Definitions \ref{def:transferCoefficient}, \ref{def:smoothness}, \ref{def:noise}, \ref{def:doublingMeas} or \ref{def:boundedCov}. We call $\dmFam$ (resp. $\bcnFam$) the class of all distribution tuples $(\sJoinProb, \tJoinProb)$ with transfer parameters $(\transMarginCoeff, \transMarginExp)$ and where $Q$ satisfies \dm~(resp. \bcn) for the fixed parameters.
\end{mydefinition}

\section{Results Overview}
\label{sec:overview}
We start with lower-bounds (Section \ref{sec:lowerbounds}), and matching \emph{oracle} upper-bounds (Section \ref{sec:upperbounds}). 
Our adaptivity results follow in Section \ref{sec:adaptivity}. 

\subsection{Minimax Lower-Bounds} 
\label{sec:lowerbounds}
As shown below, the rates of transfer get worse with large $\gamma \in \mathbb{R}_+ \cup \{0, \infty\}$. For  simplicity, we assume here that $\tDmDim$ is an integer, $\spa = [0,1]^{\tDmDim}$ and $\dist(x,y) = \|x - y \|_{\infty}$. Similar lower-bounds can be established on more general metric spaces, through appropriate \emph{packings} of such space, while $\spa = [0,1]^{\tDmDim}$ affords us a simpler construction on a grid. In the following, $P_X$ is of the same dimension $d$ as $Q_X$, while Proposition \ref{prop:lbDmLowerDim} of Appendix \ref{app:example4} illustrates a lower-bound construction for the case of Example \ref{ex:diffDim} where dimensions differ. 

\begin{theorem} [Lower-bounds] \label{thm:minimaxLowerBound}
Let $(\spa,\dist) = ([0,1]^{\tDmDim},\| . \|_{\infty})$, for $\tDmDim \in \mathbb{N}^{*}$. 
Consider any classifier $\hat h$ learned on $\sample$, with knowledge of $P_X, Q_X$. 
The following holds for any admissible values of class parameters, up to specified restrictions. 

\begin{itemize}
\item \dm: Let $\rates = 2  + \tCovDim / \holderExp$, and $c = c (\dmFam)$ and suppose $\alpha\beta <d$. We have:
$$\sup_{(\sJoinProb, \tJoinProb) \in \dmFam} \mathbb{E}_{\sample}[\exErr(\hat h)] \geq \lowConst \left( \sN ^{\rates / (\rates + \transMarginExp / \holderExp)} + \tN \right)^{ -(\tTsyExp + 1) / \rates}.$$
For $\alpha \beta = d$, for any such $\beta$, there exists $C_\beta>0$ such that the above holds.

\item \bcn: Let $\rates = 2 + \tTsyExp + \tCovDim / \holderExp$, and 
$c = c (\bcnFam)$. We have: 
$$\sup_{(\sJoinProb, \tJoinProb) \in \family} \mathbb{E}_{\sample}[\exErr(\hat h)] \geq \lowConst \left( \sN ^{\rates / (\rates + \transMarginExp / \holderExp)} + \tN \right)^{ -(\tTsyExp + 1) / \rates}.$$
\end{itemize}


\end{theorem}

Note that, for $n_P = 0$, we recover known classification lower-bounds of \cite{audibert2007fast}.  
As in that work, our lower-bounds exclude the regime $\alpha\beta> d$ for $\dmFam$ since nontrivial such settings are impossible (see Proposition 3.4 and discussion in \cite{audibert2007fast}). 

The main technicality in our transfer lower-bound is in dealing with two sources of randomness $(P, Q)$, along with keeping $(P, Q)$ related through $\gamma$. Unlike in usual lower-bounds, the learner has access to non-identical samples, in addition to \emph{knowing} both marginals $P_X, Q_X$. This brings up an interesting point: \emph{additional unlabeled data do not improve the minimax rates of transfer}. 

\subsection{Minimax Upper-Bounds}
\label{sec:upperbounds}
Our oracle upper-bounds are established through a generic $k$-NN classifier over the combined sample, as defined below. 

\begin{mydefinition}[$k$-NN] \label{def:kNNClass} Pick $1\leq k \leq n_P\vee n_Q$. Fix $x \in \spa$, and let $\{\nnFeat{i}{}\}_{i =1}^k$ 
denote the $k$ nearest neighbors of $x$ in $\featVectBase$ (break ties anyhow), with corresponding labels 
 $\{\nnLab{i}{}\}_{i = 1}^k$. 
Define the regression estimate $\empReg(x) \equiv \frac{1}{\nn} \sum_{i=1}^{\nn} \nnLab{i}{}$. 
The $k$-NN classifier at $x$ is then given by $\hPQ(x) \equiv \mathbbm{1}\{\empReg(x) \geq 1/2\}$. 
\end{mydefinition}


\begin{myremark}[New rates for $k$-NN under $\bcn$]\label{rem:newKNNbounds}  
Theorem \ref{thm:expErrRates} below is of independent interest for vanilla $k$-NN (by setting $n_P = 0$): namely, 
while $k$-NN methods have received much renewed attention \cite{samworth2012optimal, chaudhuri2014rates, shalev2014understanding}, most results concern the $\dm$ setting, 
i.e., assume near-uniform marginals. Notable recent exceptions are \cite{gadat2014classification, cannings2017local}, which both even allow unbounded support ${\cal X}_Q$. 
On one hand, under $\bcn$, \cite{gadat2014classification} show that the minimax rates 
of $n_Q^{-(\beta + 1)/(2+ \beta + d/\alpha)}$ are reachable by NN methods where $k$ is chosen locally as $k(x)$.  
Our results instead states that such optimal rates are reachable by vanilla $k$-NN with a global choice of $k$. 
The recent results of \cite{cannings2017local} also hold for global choices of $k$, but assuming $\beta = 1$, along with further smoothness assumptions on $\eta$ deviating from the $\bcn$ setting considered here. An interesting fact revealed here is that, a global optimal regression choice of $k$ of the form $n_Q^{1/(2+d/\alpha)}$) is suboptimal for classification, while the optimal choice of $k$ is smaller, of the form $k_2 = n_Q^{1/(2+\beta + d/\alpha)}$. 
Finally, for more context, we note that the original rates of $\bcn$ in \citep{audibert2007fast} are achieved by a non-polynomial time procedure based on intractable covers of a function space.

\end{myremark}

\begin{theorem} [Upper-bounds] \label{thm:expErrRates}
Let $\hPQ$ as given in Definition \ref{def:kNNClass}. The following holds for an oracle choice of $k$, set according to whether \dm\, or \bcn\, holds. 
\begin{itemize}
\item For $(\sJoinProb, \tJoinProb) \in \dmFam$, let $\rates = 2  + \tCovDim / \holderExp$, and $\upConst = \upConst (\dmFam)$ we have: 
$$ \mathbb{E}_{\sample}[\exErr(\hPQ)] \leq \upConst \left( \sN ^{\rates / (\rates + \transMarginExp / \holderExp)} + \tN \right)^{ -(\tTsyExp + 1) / \rates}.$$

\item For $(\sJoinProb, \tJoinProb) \in \bcnFam$, let $\rates = 2 + \tTsyExp + \tCovDim / \holderExp$, and 
$\upConst = \upConst (\bcnFam)$. We have: 
\begin{equation*}
\mathbb{E}_{\sample}[\exErr(\hPQ)] \leq
\begin{cases} 
\upConst \left( \sN ^{\rates / (\rates + \transMarginExp / \holderExp)} + \tN \right)^{ -(\tTsyExp + 1) / \rates} \text{ if } 
{\color{black} \alpha < 1}, \text{ otherwise} \\
\upConst\cdot \log(2(n_P + n_Q))\cdot \left( \sN ^{\rates / (\rates + \transMarginExp / \holderExp)} + \tN \right)^{ -(\tTsyExp + 1) / \rates} 
\end{cases} 
\end{equation*}
\end{itemize} 

The optimal {oracle} choice of $k$ is $\Theta \left( \sN ^{\rates / (\rates + \transMarginExp/ \holderExp)} + \tN \right)^{2 /\rates}$, where $\rates$ is as defined above for each of $\dmFam$, or $\bcnFam$.
\end{theorem}
The bounds match those of Theorem \ref{thm:minimaxLowerBound} (for 
$(\spa,\dist) \subset (\mathbb{R}^{\tDmDim},\ell_p), 1 \leq p \leq \infty$), 
apart for the corner case $\bcnFam$ with $\color{black} \alpha = 1$ where an additional log term gets introduced. Thus, the relative benefits of source and target samples is captured, through the transfer exponent $\gamma$, in the rate $(\sN ^{\rates / (\rates + \transMarginExp/ \holderExp)} + \tN)$, $d_0 = d_0(\family)$. In particular, source samples are most beneficial when
$\sN^{\rates / (\rates + \transMarginExp/ \holderExp)} \gg \tN$ (the rates are then of order $\sN^{- (\beta + 1)/(\rates + \transMarginExp/ \holderExp)}$), otherwise target samples are most beneficial (the rates then transition to $\tN^{- (\beta + 1)/\rates}$). 
Notice that the threshold $\sN ^{\rates / (\rates + \transMarginExp/ \holderExp)}$, viewed as an \emph{effective sample size from $P$}, is largest at $\gamma = 0$, and decreases to $1$ as $\gamma \to \infty$, in which case even a small amount $n_Q$ of target labels can considerably improve classification with respect to $Q$. Further intuition can be given for this effective sample size, e.g., in the case of  Example \ref{ex:diffDim}, as discussed in Appendix \ref{app:example4}. 

Setting $n_Q = 0$, we see that transfer remains possible in a rich continuum of regimes between $\gamma = 0$ and $\gamma = \infty$ with rates of the form $n_P^{-(\beta + 1)/(d_0 + \gamma/\alpha)}$, including \emph{fast} rates $o(n_P^{-1/2})$ for large $\beta$ (low noise). 


\begin{myremark}[Extended settings] We consider deviations from the above settings of $\dm$ and $\bcn$ in Appendix 
\ref{app:extensions}. First, Theorem \ref{theo:betainfinity} of the appendix addresses the case of $\beta = \infty$ under $\dm$ where, similar to the vanilla classification setting of \cite{audibert2007fast}, we can obtain exponential decreasing rates with constants expressed in terms of $n_P$, $n_Q$, and $\gamma$. 

Next, when the supports $\mathcal{X}_Q, \mathcal{X}_P$ do not overlap, or when 
$P_{Y|X}$ deviates from $Q_{Y|X}$, we obtain similar rates with additive terms accounting for such deviation. 
\end{myremark}

\subsection{Adaptive Upper-Bounds}
\label{sec:adaptiveRates}
\begin{algorithm} [t]
\caption{Adaptive NN classification estimate}
\DontPrintSemicolon
\KwIn{A labeled sample $\sampleBase'$ of size $\totN$, a query point $x$ and an integer $\nnZero \geq 1$.}
\vspace{0.1cm} 
For any value of $\nn$, let $\empReg_{\nn}(x) \doteq \frac{1}{\nn} \sum_{i=1}^{\nn}Y'_{(i)}$ the $\nn$-NN regression estimate using $\sampleBase'$\; 
Let $\nn = \nnZero$, $\lowerBd_k = \empReg_{\nn}(x) - \sqrt{\frac{\vcDim }{\nn}}\log \totN$, $\upperBd_k = \empReg_{\nn}(x) + \sqrt{\frac{\vcDim }{\nn}}\log \totN$ \; 
\While{ $\lowerBd_k \leq 1/2$ and $\upperBd_k \geq 1/2$ and $\nn \leq  \totN/2$}{
$\nn \leftarrow 2\nn$ \;
$\lowerBd_k \leftarrow (\empReg_{\nn}(x) - \sqrt{\frac{\vcDim }{\nn}}\log \totN) \vee \lowerBd_{k/2}$ \;
$\upperBd_k \leftarrow (\empReg_{\nn}(x) + \sqrt{\frac{\vcDim }{\nn}}\log \totN) \wedge \upperBd_{k/2}$ \;
\uIf{$\upperBd_k < \lowerBd_k$}{
$\empReg_{\nn} = (\upperBd_k+ \lowerBd_k)/2$ \\
\tcp{   Note: $\empReg_{\nn}$ is potentially different from $\empReg_{\nn}(x)$}
\textbf{break}}
}
\Return Classification estimate $\hat h(x) \leftarrow \mathbbm{1} \{  \empReg_{\nn} \geq 1/2 \}$ \label{alg:adaptiveKNN}\\
$\,$
\end{algorithm}

While the rates of Theorem \ref{thm:expErrRates} are tight, they require a choice of $k$ that depends on unknown distributional parameters. In this section we argue that, under some additional regularity on the metric $(\spa, \dist)$, $k$ can be chosen adaptively at each query $x$ as $k(x)$ to nearly attain the above rates. Such an adaptive choice is given in Algorithm \ref{alg:adaptiveKNN}, and is a refinement of so-called \emph{Lepski's method} \citep{lepski1997optimal}, 
but goes back to the \emph{intersecting confidence intervals} (ICI) approach of Goldenshluger and Nemirovski \cite{goldenshluger1997spatially}. A main distinction here is that, while typical analyses of such methods concern unknown smoothness $\alpha$, we here have to also adapt to unknown $d,  \gamma, \beta$; this comes with no substantial change to the basic algorithmic approach, but requires a more careful analysis. 

\begin{assumption}[Bounded VC ]\label{ass:boundedVC}
The family $\ballColl$ of all balls in $(\spa, \dist)$ has \emph{known} finite (Vapnik-Chervonenkis) VC-dimension $\vcDim$. 
\end{assumption}

The above regularity assumption holds for instance for subsets $\spa$ of a Euclidean space, with  
$\dist$ corresponding to a norm on the space. The assumption allows to bound pointwise regression rates, uniformly over $x \in \spa_Q \subset \spa$.

\begin{theorem}[Adaptive Rates of Algorithm \ref{alg:adaptiveKNN}] \label{thm:genericadaptivity}
Let Assumption \ref{ass:boundedVC} hold, and let $\family$ denote $\dmFam$ or $\bcnFam$. For $\family = \bcnFam$ assume further that $\alpha < d$. Suppose Algorithm \ref{alg:adaptiveKNN} takes as input $\sampleBase' \doteq \sampleBase$, 
with $\nnZero \doteq \lceil \vcDim \log(2(\sN + \tN)) \rceil$.
Let $\hHat$ denote the output of Algorithm \ref{alg:adaptiveKNN}. 
We have, for a constant $\upConst = \upConst(\family)$: 
\begin{equation*}
\sup_{(\sJoinProb, \tJoinProb) \in \family} \mathbb{E}_{\sample}[\exErr(\hHat)] \leq \upConst \left( \frac{\nnZero\cdot \log (2(\sN + \tN))}{\sN ^{\rates / (\rates + \transMarginExp / \holderExp)} + \tN }\right)^{ (\tTsyExp + 1) / \rates}, 
\end{equation*}
where $\rates = 2  + \tCovDim / \holderExp$ when $\family = \dmFam$, and $\rates = 2 + \tTsyExp + \tCovDim / \holderExp$ when $\family = \bcnFam$.
When $\family = \bcnFam$ with $\alpha = d$, replace $\upConst$ above with 
$\upConst(\bcnFam)\cdot \log(2(n_P + n_Q))$. 
\end{theorem}

The above rates match those of Theorem \ref{thm:expErrRates} up to log terms. 

\subsection{Active Sampling of Target Labels}
\label{sec:adaptivity}

While all discussion so far assumed an i.i.d. labeled target sample from $Q$, an increasingly popular setting \citep{saha2011active, chen2011co, pmlr-v28-chattopadhyay13} consists of selective labeling of a sample of unlabeled datapoints from $Q_X$. In this section, we show that, despite the new dependencies introduced in such settings, the transfer exponent $\gamma$ still yields bounds on statistical accuracy, while also controlling labeling requirements. We will consider the following setup:

\begin{quote}\normalsize
{\bf Setup:} The learner has access to labeled source data 
$(\sFeatVect, \sLabVect)$ of size $n_P$, and unlabeled target data $\tFeatVect$ of size $n_Q$. The goal is to request as few target labels as possible (at most $n_Q$ by design), and return a classifier $\hat h$, trained on the final labeled sample. 
\end{quote} 

A natural idea, recently formalized by \cite{berlind2015active}, is to request labels only at those datapoints $x\in \tFeatVect$ that 
have \emph{little} coverage under $P$, i.e., have relatively few neighbors from $\sFeatVect$. 
The algorithmic approach of \cite{berlind2015active} builds on the following useful concept. In all that follows we use the 
shorthand notation $[n]\doteq \{1, \ldots, n\}$, and the abbreviation \emph{NN} for \emph{nearest neighbor}. 


\begin{mydefinition}  [$\nn$-$2\nn$ Cover] \label{def:k2kCover}
Let $1\leq \nn \leq (\sN \vee \tN)/2$, and let $\featVectBase_{\coverSet}$ denote samples in $\featVectBase \doteq \sFeatVect \cup \tFeatVect$ indexed by  $\coverSet \subset [\sN + \tN]$. We say that $\featVectBase_{\coverSet}$ is a {\bf $\nn$-$2\nn$ cover} of $\featVectBase$ if, for any $\featVar_{i} \in \featVectBase$, either $\featVar_{i} \in \featVectBase_{\coverSet}$, or its $2\nn$ NN's in 
 $\featVectBase$ (including $\featVar_{i}$ itself) include at least $\nn$ samples from $\featVectBase_{\coverSet}$. If the choice of the $2 \nn$-NN is not unique, at least {\em one} of the possible choices must contain $\nn$ samples from $\featVectBase_{\coverSet}$.
\end{mydefinition}

While \citep{berlind2015active} leaves open the question of the choice of $k$, the idea is to request labels only for those points in $\mathbf{X}_R \cap \mathbf{X}_Q$. 
They present various ways to build such a cover, the obvious way being to start with the labeled samples, i.e., 
$\mathbf{X}_R = \mathbf{X}_P$, and add in points from $\mathbf{X}_Q \setminus \mathbf{X}_R$ that do not satisfy the conditions.  Following this, classification then consists of a $k$-NN estimate over a labeled sample 
$(\mathbf{X}_R, \mathbf{Y}_R)$. 

\paragraph{Contribution} We modify the procedure of \cite{berlind2015active} and construct a \emph{cover} $\mathbf{X}_R$ which is simultaneously a $k$-$2k$ cover for all $k$ in log-scale $\cal K$ of the form $[\log n: n/2], n = n_P \vee n_Q$ 
(Algorithm \ref{alg:queryLabels}). We then use Algorithm \ref{alg:adaptiveKNN} to make a choice of $k = k(x) \in \cal K$, and show that, despite added dependencies in $(\mathbf{X}_R, \mathbf{Y}_R)$, this approach achieves a near optimal excess risk as in Theorem \ref{thm:genericadaptivity} above. Furthermore, we show that the amount of label requests can be upper-bounded in terms of the exponent $\gamma$, and the behavior of nearest neighbor distances under $Q$. In particular, if the dimension parameter $d$ is \emph{tight}, in the sense of Assumption \ref{ass:boundedQmass}, \emph{no label} is queried whenever $n_Q$ is too small (w.r.t. $n_P$ and $\gamma$) to yield much new information over the source data. Interestingly, this threshold on $n_Q$ is detected without knowledge of $\gamma$.


\begin{algorithm}[h] 
\caption{Simultaneous $k$-$2k$ covers over a set of dyadic values of $k$}
\DontPrintSemicolon
\KwIn{Source $(\sFeatVect, \sLabVect)$ of size $n_P$, target $\tFeatVect$ of size $n_Q$, and confidence parameter $0< \delta < 1$}
\vspace{0.1cm}
Start with indices $\coverSet \leftarrow [n_P]$, and set 
$\nnZero = \lceil \vcDim \log(2(\sN + \tN)) + \log(6/\delta)  \rceil$  \;
\For{ $i = 0$ to $\lfloor \log_{2}((\sN \vee \tN)/ 2\nnZero) \rfloor$}{
	Let $\nn \leftarrow 2^{i} \nnZero$ \;
	\tcc{Ensure that $R$ is a $\nn$-$2\nn$ cover of $\mathbf{X}$}
	$\coverSet \leftarrow \coverSet  \cup \{i \in (\sN, \sN + \tN]: 
	 \featVar_{i} \text{ has less than } \nn\text{ NNs from } \featVectBase_{R}\text{ amongst its } 2\nn\text{ NNs from } \featVectBase \}$
}
\Return $\featVect_{\coverSet}$ \label{alg:queryLabels}\\
$\,$
\end{algorithm}


\begin{assumption}[Bounded $Q$-mass]\label{ass:boundedQmass}
Let $d$ be the \emph{dimension} parameter in either $\dm$ or $\bcn$. $Q_X$ further satisfies the following, for some $\tDmCoeff'>0$:  
$$\forall r \in [0,\diamDom], \forall x \in \tDom, \tProb(B(x,r)) \leq 
\tDmCoeff' \left( r / \diamDom \right)^{d}.
$$
\end{assumption}

The main results of this section are given in the following theorem. 

\begin{theorem} [Guarantees for $k$-2$k$ covers] \label{thm:labelingComplexity}
Let $0< \delta < 1$, be the input to Algorithm \ref{alg:queryLabels}, and let 
$\mathbf{X}_R$ (a uniform $k$-$2k$ cover for $k \in \mathcal{K}$) be its output. 

\begin{itemize} 
\item {\it (Classification Rates)} Suppose Assumption \ref{ass:boundedVC} holds. If Algorithm \ref{alg:adaptiveKNN} is given as input 
$\sampleBase' \doteq \sampleBase_{\coverSet}$, then its output $\hat h$ satisfies the adaptive excess error bounds of Theorem \ref{thm:genericadaptivity}. 

\item {\it (Labeling Complexity)} Define $r_Q(x; \alpha) = \inf \{r: Q_X(B(x, r)) \geq \alpha\}$, $\alpha\in[0, 1]$, for any $x\in {\cal X}_Q$ and $\alpha \in [0, 1]$. Then, with probability at least $1-2\delta$, Algorithm \ref{alg:queryLabels} will not query the label of any datapoint in 
$$\mathbf{X}_Q \cap {\cal X}_{Q}^\gamma \doteq \left\{x \in {\cal X}_Q: r_Q\!\!\left(x; \frac{k_0}{n_Q}\right ) \geq \Delta_{\cal X} \cdot\left(9C_\gamma \frac{n_Q}{n_P}\right)^{1/\gamma} \right \}.$$

In particular, if $Q_X$ satisfies Assumption \ref{ass:boundedQmass}, then there is \emph{no} label query whenever 
$\sN^{\tDmDim / (\tDmDim + \transMarginExp)} \geq \upConst \tN$, for a constant $C = C(\gamma, d, \Delta_{\cal X})$.
\end{itemize} 
\end{theorem}

In the above, the \emph{no query} set  ${\cal X}_{Q}^\gamma$ gets larger as $n_Q$ gets smaller, since the defining conditions get looser: the r.h.s. of the inequality gets smaller, while $k_0/n_Q$ gets larger and therefore so does $r_Q(x, k_0/n_Q)$. Intuitively, the source $\sFeatVect$ has better coverage of smaller targets $\tFeatVect$, so target labels are less important. 

\newcommand{\ratesOne}{d_{1}}

\section{Lower Bound Analysis} \label{sec:lowerboundAnalysis} 

Theorem \ref{thm:minimaxLowerBound} is a consequence of Propositions \ref{prop:lbDmFinite}, \ref{prop:lbBcnFinite} and \ref{prop:lbInfinite}. The more involved construction being that of Proposition \ref{prop:lbDmFinite}, we provide it fully here, 
while the other 2 propositions are covered in the appendix, and follow similar arguments but relatively simpler constructions. 

As stated earlier, the main technicality in showing Theorem \ref{thm:minimaxLowerBound} is in dealing with classifiers learned on non-identical samples.
Various basic tools are used in the literature, which often build on Fano's inequality, Assouad, or LeCam's approach \citep{yu1997assouad}. 
In particular, Theorem 2.5 of \citet{tsybakov2009introduction} will best suit our needs. In what follows, let $\KLDiv{\cdot}{\cdot}$ denote Kullback-Leibler (KL) divergence. 

\begin{myproposition} [Thm 2.5 of \citet{tsybakov2009introduction}] \label{prop:tsy25} Let $\{ \sampleDist_{\h} \}_{\h \in \modelClass}$ be a family of distributions indexed over a subset $\modelClass$ of a semi-metric $( \mathcal{F}, \semiMetric)$. Suppose $\exists \, \h_0, \ldots, \h_{\M} \in \modelClass$, for $\M \geq 2$, such that:
\begin{flalign*} 
\qquad {\rm (i)} \quad  &\semiDist{\h_{i}}{\h_{j}} \geq 2 \s > 0, \quad \forall 0 \leq i < j \leq \M,  & \\
\qquad {\rm (ii)} \quad  & \sampleDist_{\h_i} \ll \sampleDist_{\h_0} \quad \forall i \in  [\M], \text{ and the average  KL-divergence to } \sampleDist_{\h_0} \text{ satisfies } & \\
& \qquad 
\frac{1}{\M} \sum_{i = 1}^{\M} \KLDiv{\sampleDist_{\h_i}}{ \sampleDist_{\h_0}} \leq \alp \log \M, \text{ where } 0 < \alp < 1/8.
\end{flalign*}
Let $Z\sim\sampleDist_{\h}$, and let $\hat \h : Z \mapsto \mathcal{F}$ denote any \emph{improper} learner of $h\in \modelClass$. We have: 
\begin{equation*}
\sup_{\h \in \modelClass} \sampleDist_{\h} \left( \semiDist{\hat \h(Z)}{\h} \geq \s \right) \geq \frac{\sqrt{\M}}{1 + \sqrt{M}} \left( 1 - 2 \alp - \sqrt{\frac{2 \alp}{\log(\M)}} \right) \geq \frac{3 - 2 \sqrt{2}}{8}.
\end{equation*}
\end{myproposition}

For our purpose, the $h$ indices would stand for Bayes classifiers over possible regression functions $\eta$ satisfying $\dm$ or $\bcn$. Let $(P, Q)$ denote a transfer tuple with corresponding Bayes classifier $h$; for fixed $n_P, n_Q$, we let 
$\Pi_h = P^{n_P} \times Q^{n_Q}$, thus coupling $P$ and $Q$ into a single distribution. Now, while the KL-divergences over the family involve both $P$ and $Q$, we are free to define $\bar \rho$ over $Q_X$ alone, and thus relate it to the target excess error $\exErr$. Now what's left is to ensure that the various conditions of $\dm$ or $\bcn$ are satisfied. 

For conditions involving only $Q$ (smoothness, noise, and dimension), we follow closely the original lower-bound construction of \cite{audibert2007fast}, apart for some technical details in our choice of smooth basis functions ($\eta$ is chosen as a linear combination of simple basis functions crossing $1/2$). Now, to ensure that any given transfer-exponent $\gamma$ holds, we divide up the mass of $P_X$ appropriately over space, following the type of intuitions laid out in Examples 
\ref{ex:disjointSupports}, \ref{ex:boundedDensity}, \ref{ex:unboundedDensity}, \ref{ex:diffDim} of Section \ref{sec:transferexponent}. The rest involves adjusting the construction properly so that $h$'s are sufficiently far in $\bar \rho = \bar \rho(Q_X)$, while 
$\Pi_h$'s (involving both $P, Q$) remain sufficiently close in KL-divergence. 

We note that the marginals $P_X, Q_X$ remain fixed for our choice family $\{\Pi_h\}$, and thus might be known to the learner (which is allowed to know the family, but not the data's distribution).


Finally, the choice of the $\M+1$ elements $\h_{i}$ in Proposition \ref{prop:tsy25} should allow $\M$ as large as possible while maintaining the \emph{packing} {\rm(i)} and \emph{covering} {\rm(ii)} conditions of Proposition \ref{prop:tsy25}. The following lemma often comes in handy in achieving {\rm(ii)}, by reduction from a larger family of $2^m$ distributions. 

\begin{lemma} [Varshamov-Gilbert bound (see e.g. \cite{tsybakov2009introduction})] \label{lem:VGBound}
Let $\m \geq 8$. Then there exists a subset $\{ \sig_0, \ldots, \sig_{\M}\}$ of $\{-1 ,1 \}^{\m}$ such that $\sig_0 = (1,\ldots,1)$,
\begin{equation*}
\hammingDist{\sig_{i}}{\sig_{j}} \geq \frac{\m}{8}, \quad \forall\,  0 \leq i < j \leq \M, \quad \text{and} \quad \M \geq 2^{\m / 8},
\end{equation*}
where $\hammingDist{\sig}{\sig'} \doteq \text{card}(\{ i \in [\m] :  \sig(i) \neq \sig'(i) \})$ is the Hamming distance.
\end{lemma}


We present the lower bounds for the family $\dmFam$, for $\transMarginExp < \infty$, while lower-bounds for other 
regimes are treated in a similar fashion in the appendix. 

\subsection{Lower Bound for $\family = \dmFam$ when $\transMarginExp < \infty$}

\begin{figure}
\hspace{-1.9cm}
\begin{subfigure}{.53\textwidth}
\centering 
\includegraphics[width=0.52\linewidth]{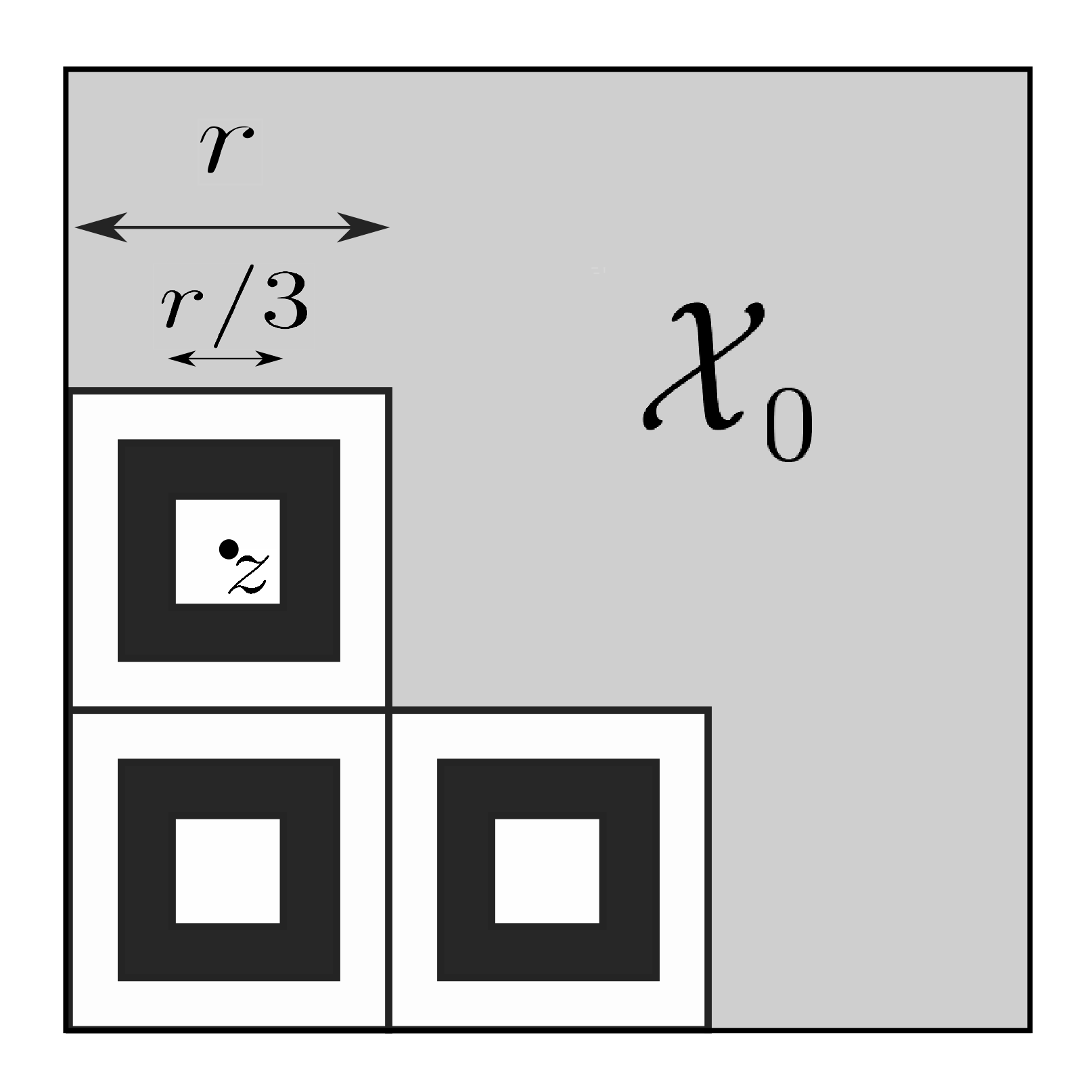}
\caption{}
\label{fig:support}
\end{subfigure} 
\hspace{-3cm}
\begin{subfigure}{.53\textwidth}
\centering 
\includegraphics[width=0.52\linewidth]{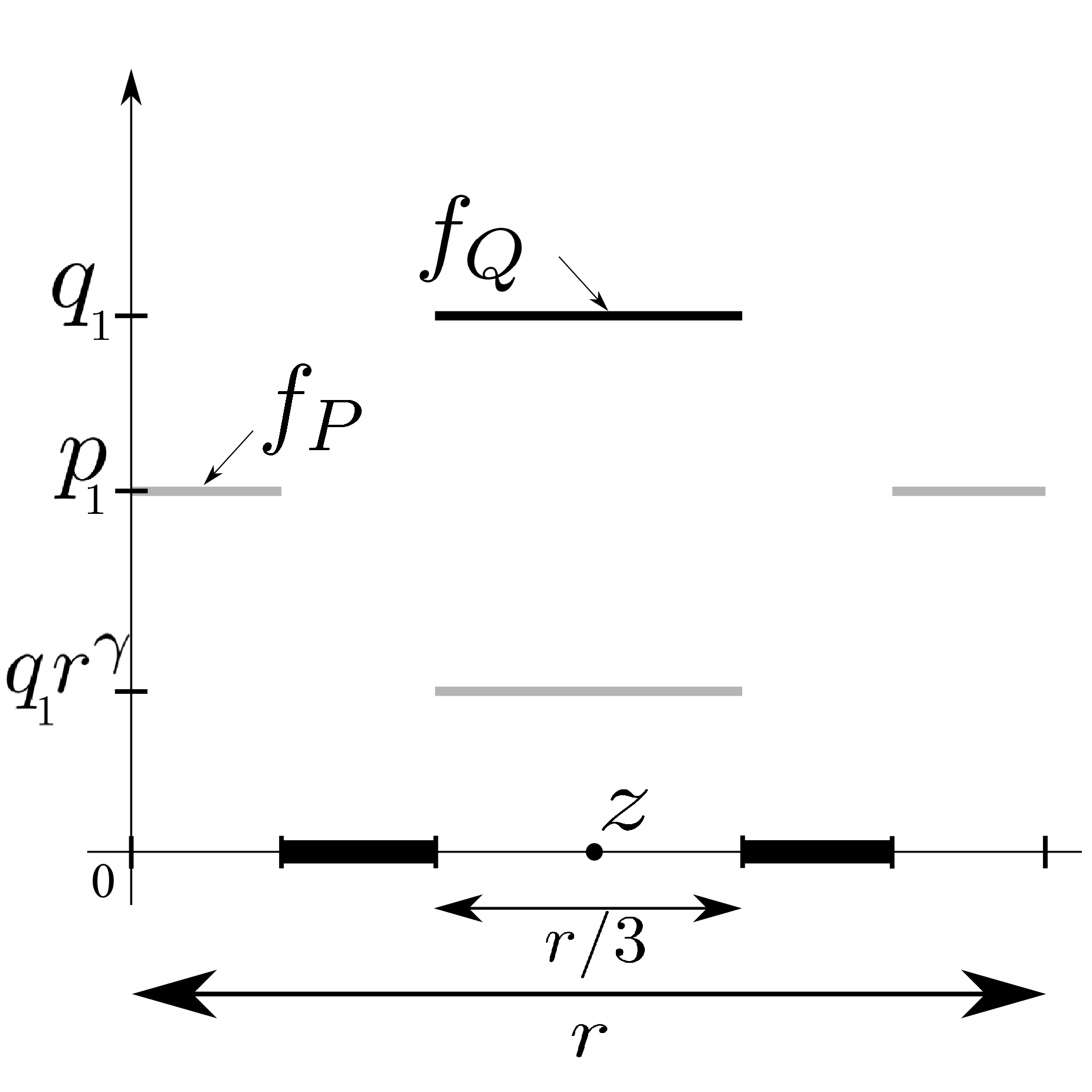}
\caption{}
\label{fig:margProfile}
\end{subfigure}
\hspace{-3cm}
\begin{subfigure}{.53\textwidth}
\centering 
\includegraphics[width=0.52\linewidth]{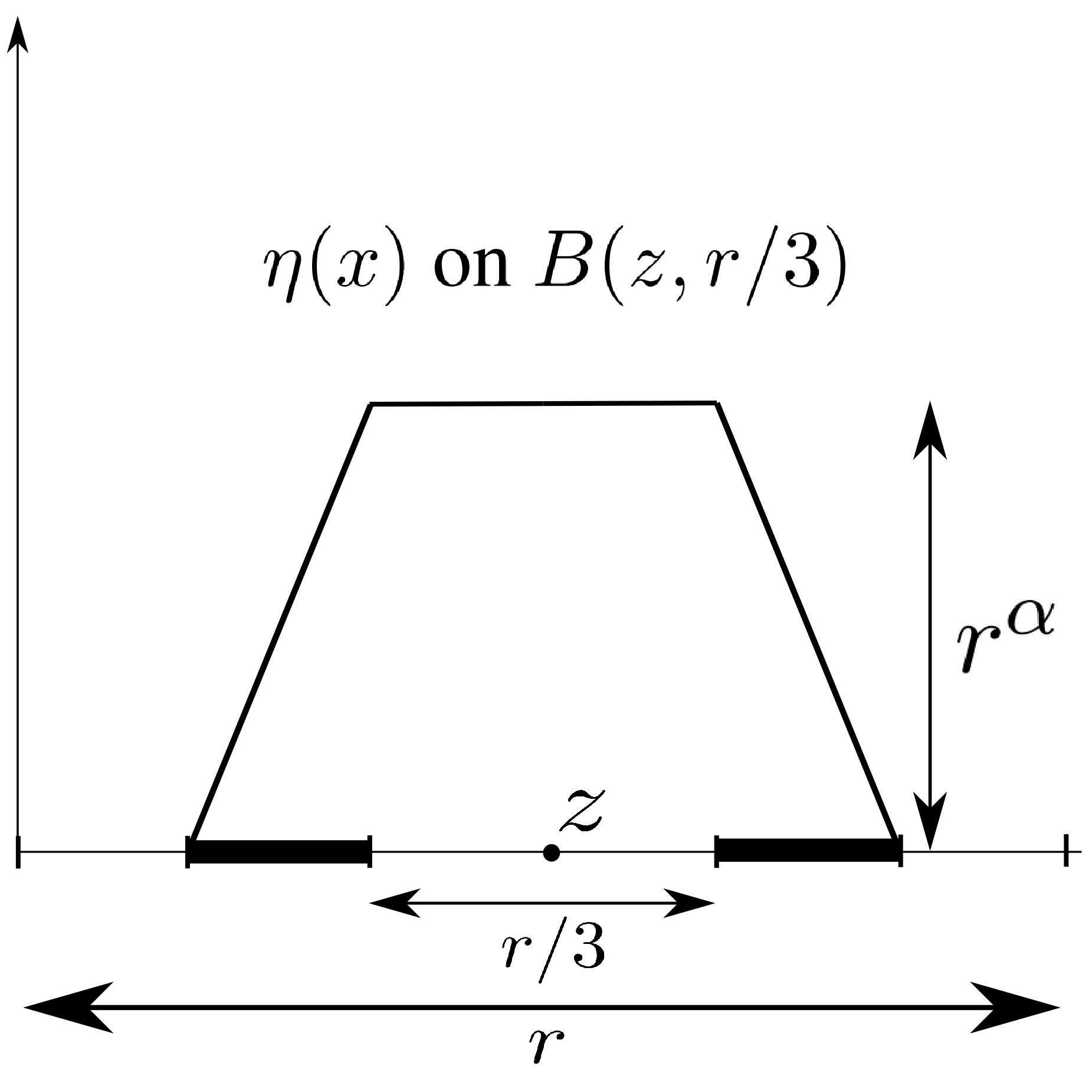}
\caption{}
\label{fig:regProfile}
\end{subfigure}
\caption{{\ref{fig:support}} illustrates the supports of $\sProb$ and $\tProb$ on the hypercubes arising from the subdivision of the space {\color{black} $\spa'$} into $\spa_0$ (which serves to account for missing mass) and {\color{black} $\spa'\setminus \spa_0$} where we make $\eta$ vary (across hypercubes) so as to make classification difficult, subject to the various distributional conditions.  
{\ref{fig:margProfile}} shows the profiles of densities $f_P, f_Q$ of $\sProb$ and $\tProb$ on a hypercube in {\color{black} $\spa'\setminus \spa_0$}, having some center $z$; we note that the construction here is a simple one that allows a transfer-exponent $\gamma$ at resolution $r$, while simplifying the analysis; however other constructions such as ones described in Figure \ref{fig:gamma} also work, but add technicality with no additional insight. {\ref{fig:regProfile}} displays the profile of the regression function (in the Lipschitz case) on the above-mentioned hypercubes. Notice that $\eta$ maintains a margin $r^\alpha$ on $B(z, r/6)$, i.e., on the support of $Q_X$ inside the hypercube $B(z, r/2)$.}
\label{fig:lbGraphs}
\end{figure}

\begin{myproposition} \label{prop:lbDmFinite}
Let $(\spa,\dist) = ([0,1]^{\tDmDim},\| . \|_{\infty})$, for some $\tDmDim \in \mathbb{N}^{*}$, and assume that $\holderExp \tTsyExp< \tDmDim$ and $\transMarginExp < \infty$, for any admissible value of the parameters 
$C_\beta, C_\gamma, C_\alpha, C_d$. There exists a constant $\lowConst = \lowConst(\dmFam)$ such that, for any classifier $\hat h$ learned on $\sample$ and with knowledge of $P_X, Q_X$, we have: 
\begin{equation*}
\sup_{(\sJoinProb, \tJoinProb) \in \dmFam} \mathbb{E}_{\sample}[\exErr(\hat h)] \geq \lowConst \left( \sN ^{\rates / (\rates + \transMarginExp / \holderExp)} + \tN \right)^{ -(\tTsyExp + 1) / \rates}, \text{ where } \rates = 2  + \tCovDim / \holderExp.
\end{equation*}
For $\alpha \beta = d$, for any such $\beta$, there exists $C_\beta>0$ such that the same bound holds.
\end{myproposition}

\begin{proof}
First, {\color{black} fix $\domBnd \in (0,1]$ and} define the following variables
\begin{equation*} 
\radius = \rConst {\color{black} \domBnd} \left( \sN^{\rates / (\rates + \transMarginExp / \holderExp)} + \tN \right)^{-1/(\holderExp \rates)} , \, \m = \left \lfloor \mConst \left(  {\color{black} \frac{\radius}{\domBnd}}\right)^{\holderExp \tTsyExp - \tDmDim} \right \rfloor , \, \w = \wConst {\color{black} \left( \frac{\radius}{\domBnd}\right)^{ \tDmDim}}, 
\end{equation*} 
in terms of constants $\rConst = 1/9$, $\color{black} \mConst = 8 \times 9^{\holderExp \tTsyExp - \tDmDim}$ {\color{black} and, when $\alpha\beta <d$,  $\wConst =\domBnd^{\alpha\beta} \min(\mConst^{-1}\tTsyCoeff(\holderCoeff'/2)^{\tTsyExp}, 2^{-4} \log(2) \rConst^{-\holderExp \rates}\holderCoeff'^{-2}, 1/2)$, where $\holderCoeff' \doteq \holderCoeff 6^{-\holderExp} \wedge 1/2$. When $\alpha\beta = d$ we set instead $\wConst = \min(2^{-4} \log(2) \rConst^{-\holderExp \rates}\holderCoeff'^{-2}, 1/2)$.}


Note that we have chosen $\rConst$ and $\mConst$ so that $8 \leq \m < \left \lfloor {\color{black} \domBnd}\radius^{-1} \right \rfloor^{\tDmDim}$. 
This also implies that we have $\m \w < 1$. The constant $\domBnd$ will serve as a \emph{knob} to achieve any desired $C_\beta, C_\gamma$ and $C_d$ whenever {\color{black} $\alpha\beta <d$; in the case $\alpha\beta = d$,} {\color{black} it will be used to achieve any desired $C_\gamma$ and $C_d$.}

\paragraph{Marginal $\tProb$} 

Consider a regular subdivision of ${\color{black}\spa'  = [0,\domBnd]^{\tDmDim} \subseteq \spa}$ into $\lfloor {\color{black} \domBnd} \radius^{-1} \rfloor^{\tDmDim}$ smaller hypercubes of side length $\radius$ (see Figure \ref{fig:support}). 

Call $\grid$ the set of centers of {\color{black} the $\lfloor  \domBnd \radius^{-1} \rfloor^{\tDmDim}$ hypercubes of radius $\radius$}. Now divide $\grid$ into disjoint subsets $\grid_0$ and $\grid_1$ such that $|\grid_1| = \m$ (Figure \ref{fig:support} shows those hypercubes centered in $\grid_1$). Set $\tProb$ to have a uniform density $\lbda_{1}$ with respect to Lebesgue on each set $\cBall{\p}{\radius/6}$ for $\p \in \grid_1$, whereby $\color{black}Q_X(\cBall{\p}{\radius/6}) = \w$ (see Figure \ref{fig:margProfile}). Finally, put the remaining mass $1 - \m\w$ of $\tProb$ uniformly with density $\lbda_{0}$ over $\spa_{0} \doteq \cup_{\p \in \grid_{0}} \cBall{\p}{\radius/2}$. The rest of the space has zero mass under $\tProb$. We can then lower-bound $q_0, q_1$ as follows:

\begin{equation} \label{eq:firstCondOnCm}
\lbda_{0} = \frac{1 - \m \w}{\text{vol} \left( \bigcup_{\p \in \grid_0} \cBall{\p}{\radius/2} \right)} \geq {\color{black} \frac{1 - \wConst }{\domBnd^{\tDmDim}} }, \,
\lbda_{1} = \frac{\w}{\text{vol} \left( \cBall{\p}{\radius/6} \right)} \geq {\color{black}\wConst \domBnd^{- \tDmDim}}.
\end{equation}

{\color{black} We can let $D$ sufficiently small, i.e., $q_0$ and $q_1$ sufficiently large to achieve any desired $C_d$, independently of $n_P, n_Q$.}


\paragraph{Marginal $\sProb$} 

Now let's turn to the construction of $\sProb$. The idea is to let $\sProb$ be uniformly distributed on each of the sets $\cBall{\p}{\radius/6}$ for $\p \in \grid_{1}$, and so that its density is getting smaller w.r.t.~$\tProb$'s density as $\radius$ goes to zero (when $\transMarginExp > 0$). More precisely, let $\PDensity_{1} = \lbda_{1} {\color{black}(\radius \domBnd^{-1})}^{\transMarginExp}$ be the density of $\sProb$ on $\cBall{\p}{\radius/6}$ for any $\p$ in $\grid_{1}$. Because of the factor ${\color{black}(\radius \domBnd^{-1})}^{\transMarginExp} \leq 1$, we have that $\sProb(\cBall{\p}{\radius/6}) = \tProb(\cBall{\p}{\radius/6}) {\color{black}(\radius \domBnd^{-1})}^{\transMarginExp} \leq \tProb(\cBall{\p}{\radius/6})$. We therefore put the remaining mass of $\sProb$ (if any) uniformly on each set $\cBall{\p}{\radius/2} \backslash \cBall{\p}{\radius/3}$ such that $\sProb(\cBall{\p}{\radius/2}) = \tProb(\cBall{\p}{\radius/2}), \,\, \forall \p \in \grid_{1}$ (see Figure \ref{fig:margProfile}). We let $\sProb$ have a uniform density $\PDensity_{0}$, equal to that of $\tProb$ (that is $\PDensity_{0} = \lbda_{0}$), on the remaining hypercubes $\cBall{\p}{\radius/2}$ for $\p \in \grid_{0}$. Hence we have also $\sProb(\cBall{\p}{\radius/2}) = \tProb(\cBall{\p}{\radius/2}), \,\, \forall \p \in \grid_{0}$.

Recall that the support of $\tProb$ is the union of the sets $\cBall{\p}{\radius/6}$ for all $\p \in \grid_{1}$ and $\cBall{\p}{\radius/2}$ for $\p \in \grid_{0}$, so we need only check \eqref{eq:ass1equation} for points $x$ in these sets. Fix $\p \in \grid_{1}$, we have $\forall x \in \cBall{\p}{\radius/6}, \forall \tempRadius \in [0, \radius /3]$ that $\sProb(\cBall{x}{\tempRadius})$ is at least 
\begin{equation}\label{eq:verif1trans}
 \PDensity_{1} \text{vol}(\cBall{x}{\tempRadius} \cap \cBall{\p}{\radius/6}) = {\color{black}\domBnd^{-\transMarginExp}} \radius^{\transMarginExp} \tProb(\cBall{x}{\tempRadius}) \geq {\color{black}\domBnd^{-\transMarginExp}}\tempRadius^{\transMarginExp} \tProb(\cBall{x}{\tempRadius}).
\end{equation}
For $\p \in \grid_{0}$, and $\forall x \in \cBall{\p}{\radius/2}, \forall \tempRadius \in [0, \radius /3]$, the inequality is even more direct:
\begin{align}\label{eq:verif2trans}
\sProb(\cBall{x}{\tempRadius}) &\geq \sProb(\cBall{x}{\tempRadius} \cap (\cup_{\p \in \grid_{0}} \cBall{\p}{\radius/2})) \\
&= \tProb(\cBall{x}{\tempRadius}) \geq {\color{black}\domBnd^{-\transMarginExp}} \tempRadius^{\transMarginExp} \tProb(\cBall{x}{\tempRadius}). \nonumber
\end{align}

Therefore we can see that for small values of $\tempRadius \leq \radius / 3$, equation \eqref{eq:ass1equation} from Definition \ref{def:transferCoefficient} holds. Furthermore, since we set $\sProb(\ball_{\p}) = \tProb(\ball_{\p}), \forall \p \in \grid$, equation \eqref{eq:ass1equation} holds also for larger $\tempRadius$ {\color{black} for any pre-specified $\transMarginCoeff$ by taking $\domBnd$ small enough if necessary, as we did before for $\tDmCoeff$. Note that if $\transMarginExp = 0$ then $\tProb = \sProb$ and therefore \eqref{eq:ass1equation} holds directly with $\transMarginCoeff = 1$.}

\paragraph{Conditional Distributions} 

Let $\holderFct: \mathbb{R}_{+} \rightarrow \mathbb{R}_{+}$ such that:
\begin{equation*}
\holderFct(x) = \left\{ 
				\begin{array}{ll}
						1 \quad & \text{if} \,\, x \leq 1/6, \\
						1-6(x-1/6) \quad & \text{if} \,\, x \in (1/6,1/3], \\
						0 \quad & \text{elsewhere.}
				\end{array}
							\right.
\end{equation*}

It is easy to see that $\holderFct$ is $6$--Lipschitz. {\color{black} Recall} $\holderCoeff' \doteq \min(\holderCoeff 6^{-\holderExp}, 1/2)$ (the fact that we take $\holderCoeff' \leq 1/2$ will be useful later in our proof). This implies that $\holderCoeff' \holderFct^{\holderExp}(\| . \|_{\infty})$ is $(\holderCoeff, \holderExp)$--H\"{o}lder, as by concavity we have $\forall 0 \leq x \leq y, \,\, y^{\holderExp} - x^{\holderExp} \leq (y - x)^{\holderExp}$. Therefore, the following functions are $(\holderCoeff, \holderExp)$--H\"{o}lder:
\begin{equation*}
\forall \p \in \grid_1, \quad \regFct_{\p}(x) \doteq \holderCoeff' \radius^{\holderExp} \holderFct^{\holderExp}( \| x - \p \|_{\infty} / \radius).
\end{equation*}

The profile of these functions on each hypercube $\cBall{\p}{\radius/2}$ for $\p \in \grid_{1}$ is represented in Figure \ref{fig:regProfile}. Now consider the vectors $\sig \in \left\{ -1, 1 \right\}^{\m}$ that assign values $-1$ or $1$ to each of the $\m$ centers $\p$ from the set $\grid_1$. And let the following $2^{\m}$ $(\holderCoeff,\holderExp)$--H\"{o}lder regression functions, indexed by $\sig$:
\begin{equation*}
\regFct_{\sig}(x) = \left\{ 
				\begin{array}{ll}
					(1 + \sig(\p) \regFct_{\p}(x)) / 2 \quad & \text{if } x \in \cBall{\p}{\radius/2}, \quad \p \in \grid_1, \\
					1/2 \quad & \text{elsewhere.}
				\end{array}
									\right.
\end{equation*}
where $\sig(\p) \in \{-1,1\}$ is the value that $\sig$ assigns to $\p$. Note that each of these functions will take constant values $(1 \pm \holderCoeff' \radius^{\holderExp})/2$ over the balls $\cBall{\p}{\radius/6}$ of centers $\p \in \grid_1$ and be equal to $1/2$ everywhere else. 
We therefore define the following $2^{\m}$ distribution tuples $(\sJoinProb^{\sig}, \tJoinProb^{\sig})$, indexed by $\sig$: 
$\forall \sig  \in \left\{ -1, 1 \right\}^{\m}, \quad$
\begin{equation*}
\sJoinProb_{\featVar}^{\sig} \doteq \sProb, \,\, \tJoinProb_{\featVar}^{\sig} \doteq \tProb, \,\, \sJoinProb^{\sig}(\labVar = 1 | \featVar) = \tJoinProb^{\sig}(\labVar = 1 | \featVar) \doteq \regFct_{\sig}(\featVar) .
\end{equation*}

We then define the corresponding \emph{sample} distributions 
$\sampleDist_{\sig} \doteq {\sJoinProb^{\sig}}^{\otimes \sN} \otimes {\tJoinProb^{\sig}}^{\otimes \tN}.$

\subsubsection*{Tsybakov Noise Assumption}

Now we verify that the Tsybakov low-noise assumption (Definition \ref{def:noise}) is satisfied. We have 
\begin{align*}
&\text{ For } t < \holderCoeff' \radius^{\holderExp}/2, \quad \tProb(0<|\regFct(X) - 1/2|\leq t) = 0, \\
&\text{ and for } t \geq \holderCoeff' \radius^{\holderExp}/2, \quad 
\tProb(0<|\regFct(X) - 1/2|\leq t) = \m \w.
\end{align*}

{\color{black} When $\alpha\beta <d$, as we set $\wConst \leq \domBnd^{\alpha\beta} \mConst^{-1}\tTsyCoeff(\holderCoeff'/2)^{\tTsyExp}$, we have that}
\begin{align} \label{eq:tsyMarginCond}
\m \w \leq \mConst \wConst \left(\frac{\radius}{\domBnd}\right)^{\holderExp \tTsyExp} \leq \tTsyCoeff \left( \frac{\holderCoeff' \radius^{\holderExp}}{2} \right)^{\tTsyExp}.
\end{align}

{\color{black} Furthermore, when $\alpha\beta = d$, inequality \eqref{eq:tsyMarginCond} is valid only for some constant $\tTsyCoeff > 0$ independent of $\sN$ and $\tN$.}

\paragraph{Condition (i) of Proposition \ref{prop:tsy25}}

First we have to define our semi-metric $\semiDist{\cdot}{\cdot}$. Note that, given a target measure $\tJoinProb^{\sig}, \sig \in \{ -1,1\}^{\m}$, for any classifier $\h$, the excess error $\mathcal{E}_{\tJoinProb^{\sig}}(\h)$ 
equals 
\begin{equation} \label{eq:distExcessEquiv}
 \holderCoeff' \radius^{\holderExp} \tProb \left( \{ \h(X) \neq \hStar_{\sig}(X) \} \cap \bigcup_{\p \in \grid_1} \cBall{\p}{\radius/6}\right),
\end{equation}
where $\hStar_{\sig}$ is the Bayes classifier corresponding to $\regFct_{\sig}$. Hence, following the notations of Proposition \ref{prop:tsy25}, let $\mathcal{F}$ be the space of all classifiers, that is of all measurable functions from $\spa$ to $\{0,1\}$. We can define the following semi-metric on $\mathcal{F}$:
\begin{equation*}
\forall \h, \h' \in \mathcal{F}, \quad \semiDist{\h}{\h'} \doteq \holderCoeff' \radius^{\holderExp} \tProb \left( \{ \h(X) \neq \h'(X) \} \cap \bigcup_{\p \in \grid_1} \cBall{\p}{\radius/6}\right).
\end{equation*}

Note that we have:
$
\forall \sig, \sig' \in \{ -1 , 1 \}^{\m}, \quad\semiDist{\hStar_{\sig}}{\hStar_{\sig'}} = \holderCoeff' \radius^{\holderExp} \w \hammingDist{\sig}{\sig'},
$
where $\hammingDist{\sig}{\sig'} \doteq \text{card}(\{ \p \in \grid_{1} :  \sig(\p) \neq \sig'(\p) \})$ is the Hamming distance. 


\paragraph{Family of distributions} Towards applying Proposition \ref{prop:tsy25}, let $\{ \sigma_i\}_{i = 0}^M$, $M \geq 2^{m/8}$, denote the packing of the cube elicited by Lemma \ref{lem:VGBound}. 
For $0\leq i \leq M$, write $(\sJoinProb^{i}, \tJoinProb^{i}) \doteq (\sJoinProb^{\sig_i}, \tJoinProb^{\sig_i})$, and let $h^*_{\sigma_i}$ denote the corresponding Bayes classifier, uniquely defined (so we can equivalently index the family of distributions over $i$, $\sigma_i$, or over $h^*_{\sigma_i}$). Next, define the corresponding (full sample) distribution as 
\begin{equation*}
\sampleDist_{i} \doteq {\sJoinProb^{i}}^{\otimes \sN} \otimes {\tJoinProb^{i}}^{\otimes \tN}.
\end{equation*}

\paragraph{Condition (i) of Proposition \ref{prop:tsy25}}
For this family of distributions, condition (i) of Proposition \ref{prop:tsy25} is satisfied as follows for some $\dmLbConst>0$ independent of $\sN$ and $\tN$: $\forall 0 \leq i < j \leq \M$,
\begin{equation}\label{eq:defineS}
\semiDist{h^*_{\sigma_i}}{h^*_{\sigma_j}} \geq \holderCoeff' \frac{\w \m \radius^{\holderExp}}{8 } \doteq 2 \s \geq 2 \dmLbConst \left( \sN ^{\rates / (\rates + \transMarginExp / \holderExp)} + \tN \right)^{ - (\tTsyExp + 1) /\rates}.
\end{equation}

\paragraph{Condition (ii) of Proposition \ref{prop:tsy25}}

By independence, we have for $i \in \{1,\ldots, \M \}$:
\begin{equation*}
\KLDiv{\sampleDist_{i}}{\sampleDist_{0}} = \sN \KLDiv{\sJoinProb^{i}}{\sJoinProb^{0}} + \tN \KLDiv{\tJoinProb^{i}}{\tJoinProb^{0}}.
\end{equation*}
Note that since $\holderCoeff' \leq 1/2$, every regression functions $\regFct_{\sig_{i}}$ is in $[1/4, 3/4]$. As a consequence $\forall i, \,\, \sJoinProb^{i} \ll \sJoinProb^{0}$ and $\tJoinProb^{i} \ll \tJoinProb^{0}$ as all these distributions have the same marginals $\sProb$ and $\tProb$ respectively. Hence, we get:
\begin{align*}
\KLDiv{\sJoinProb^{i}}{\sJoinProb^{0}}   
&=   \int \left(\log \left( \frac{\regFct_{i}(x)}{\regFct_{0}(x)} \right) \regFct_{i}(x) + \log \left( \frac{1 - \regFct_{i}(x)}{1 - \regFct_{0}(x)} \right) (1 - \regFct_{i}(x))\right) \diff \sProb^{i}(x) \\
&=   \sum_{\p: \, \sig_{i}(\p) \neq \sig_{0}(\p)} \sProb(\cBall{\p}{\radius/6})  \left[\log \left( \frac{1 + \holderCoeff' \radius^{\holderExp}}{1 - \holderCoeff' \radius^{\holderExp}} \right) \frac{1 + \holderCoeff'\radius^{\holderExp}}{2} \right. \\
&\qquad \qquad \qquad  \quad \quad + \left. \log \left( \frac{1 - \holderCoeff' \radius^{\holderExp}}{1 + \holderCoeff' \radius^{\holderExp}} \right) \frac{1 - \holderCoeff'\radius^{\holderExp}}{2} \Large \right ] \\
&=   \hammingDist{\sig_{i}}{\sig_{0}} \w \log \left( \frac{1 + \holderCoeff' \radius^{\holderExp }}{1 - \holderCoeff' \radius^{\holderExp}} \right)\holderCoeff' \radius^{\holderExp + \transMarginExp} {\color{black}\domBnd^{-\transMarginExp}} \\
&\leq  \m\w \holderCoeff'^2\radius^{2\holderExp + \transMarginExp}/(1 - \holderCoeff' \radius^{\holderExp})
{\color{black}\domBnd^{-\transMarginExp}} \leq 2 \m \w  \holderCoeff'^2\radius^{2\holderExp + \transMarginExp}{\color{black}\domBnd^{-\transMarginExp}},
\end{align*} 
as $\holderCoeff' \leq 1/2$ and $\radius \leq 1$. On the other hand, following the same steps we get:
\begin{equation*}
\KLDiv{\tJoinProb^{i}}{\tJoinProb^{0}} \leq 2 \m \w  \holderCoeff'^2 \radius^{2\holderExp}.
\end{equation*}

The two bounds thus differ by a factor of ${\color{black}(\radius/\domBnd)}^{\transMarginExp}$. Now, ${\color{black}(\radius/\domBnd)}^{\transMarginExp}$ equals 
\begin{align*}
\rConst^{\transMarginExp} (\sN ^{\rates / (\rates + \transMarginExp / \holderExp)} + \tN )^{-\transMarginExp/(\holderExp \rates)} \leq  (\sN ^{\rates / (\rates + \transMarginExp / \holderExp)})^{-\transMarginExp/(\holderExp \rates)} = \sN^{-\transMarginExp/( \holderExp \rates + \transMarginExp)}.
\end{align*}

Therefore we get:
\begin{align} \label{eq:iiCond}
\KLDiv{\sampleDist_{i}}{\sampleDist_{0}} &\leq 2 \m \w  \holderCoeff'^2 \radius^{2\holderExp} (\tN +   \sN^{1-\transMarginExp/(\holderExp \rates + \transMarginExp)}) \nonumber \\
&\leq 2 \m \wConst \holderCoeff'^2 {\color{black}(\radius / \domBnd)}^{\holderExp \rates} (\tN + \sN^{\rates / (\rates + \transMarginExp / \holderExp)}) \nonumber \\
&\leq 2  \rConst^{\holderExp \rates} \wConst \holderCoeff'^{2} \m  
\leq 2^{4}  \log(2)^{-1}   \rConst^{\holderExp \rates} \wConst \holderCoeff'^{2} \log(\M).
\end{align}

{\color{black} With our choice of constant $\wConst$, the constant in front of $\log(\M)$ is below $1/8$, and hence the last condition of Proposition \ref{prop:tsy25} is verified.}

\paragraph{Choosing the constant $\domBnd$} {\color{black} Based on the inequalities \eqref{eq:firstCondOnCm}, \eqref{eq:verif1trans} and \eqref{eq:verif2trans} we can find $\domBnd$ small enough such that Definitions \ref{def:transferCoefficient} and \ref{def:doublingMeas} will be verified for any pre-specified constants $\tDmCoeff$ and $\transMarginCoeff$. As mentioned earlier such a choice of $\domBnd$ will be independent of $\sN$ and $\tN$, and hence will depend only on the parameters of $\dmFam$.}

\paragraph{Concluding}

We have thus verified that the conditions of Proposition \ref{prop:tsy25} are all verified for the family $\{\Pi_i\}_{i = 0}^M$, included in $\dmFam$. We can now conclude from Proposition \ref{prop:tsy25} that, for any classifier $\hHat$ built upon $\sample$, we have:
\begin{equation*}
\sup_{(\sJoinProb, \tJoinProb) \in \dmFam} \mathbb{P}_{\sample}\left( \exErr(\hHat) \geq \s \right) \geq \sup_{\sig \in \{-1,1\}^{\m}} \sampleDist_{\sig} \left( \mathcal{E}_{\tJoinProb^{\sig}}(\h) \geq \s \right) \geq \frac{3 - 2 \sqrt{2}}{8},
\end{equation*}
where $\s = \dmLbConst \left( \sN ^{\rates / (\rates + \transMarginExp / \holderExp)} + \tN \right)^{ - (\tTsyExp + 1) /\rates}$. By Markov's inequality, we therefore get the lower bound in expectation of the proposition's statement.
\end{proof}

%

\section{Upper-bound Analysis}
\label{sec:upper-bounds}

We build on previous insights from work on $k$-NN methods whenever possible. Two new technicalities are (a), accounting for the noise condition (the parameter $\beta$) in the $\bcn$ setting (without assuming local choices of $k$ or knowledge of $Q_X$ as in \cite{gadat2014classification}), and (b),  merging this with the fact that $\hPQ$ is defined on two non-identical samples (and accounting for $\gamma$). 

First, a general step in analyses of $k$-NN (and plug-in classifiers in general) is the following inequality which relates classification error for $\hPQ = \mathbbm{1}\{\hat \eta_k \geq 1/2\}$ to the regression error $|\hat \eta_k - \eta|$: 
\begin{equation} \label{eq:firstBound}
\exErr(\hPQ) \leq 2 \mathbb{E}_{\tJoinProb} \left|\regFct(\featVar)-\frac{1}{2}\right|  \cdot \mathbbm{1} \left \{\left| \regFct(\featVar)-\frac{1}{2} \right| \leq \left| \empReg_k(\featVar) - \regFct(\featVar) \right| \right \}.
\end{equation}
This is direct from the definition of excess error in equation \eqref{eq:excesserror}: notice that, for any fixed $x$, the event $\hPQ(x) \neq h^*(x)$ implies that $|\hat \eta_k (x) - \eta(x)| \geq |\eta(x) - 1/2|$. 

The usual approach in accounting for $\beta$ relies on the following simple insight: suppose a uniform bound 
$\sup_x |\hat \eta_k (x) - \eta(x)| \leq t$ held (at least in high-probability) for some $t = t(k, n_P, n_Q)$, then \eqref{eq:firstBound} implies  
$\exErr(\hPQ) \leq C_\beta t^{(\beta + 1)}$, using the fact that 
$\Expectation Z\cdot \mathbbm{1}\{Z \leq t\} \leq t\cdot \mathbb{P}(Z \leq t)$, and letting $Z \doteq |\eta(X) - 1/2|$. 

Under $\dm$ such uniform bound on regression error are possible, even in our transfer setting, since the problem is similarly hard everywhere on $\mathcal{X}_Q$. Unfortunately, this is not the case under $\bcn$ where regression difficulty can change over space as both $P_X, Q_X$ vary.   
Our approach therefore is to decompose the regression error into various terms, some of which can be bounded uniformly over 
$x \in \mathcal{X}_Q$. Namely, suppose $|\hat \eta_k (x) - \eta(x)| \leq \sum_{i\in [c]} G_i(x)$, then 
\begin{align} 
\mathbbm{1}\{Z \leq |\hat \eta_k (x) - \eta(x)| \} \leq \sum_{i \in [c]} \mathbbm{1}\{Z \leq c\cdot G_i(x)\}. 
\label{eq:incatortrick}
\end{align}

In other words, if we can bound some such term $G_i$ uniformly over $x$ by some $t_i$, we can proceed as above to bound $\Expectation Z \cdot \mathbbm{1}\{Z \leq c \cdot G_i(X)\}$ by  $C_\beta (c \cdot t_i)^{(\beta +1)}$, and thus account for $\beta$ in our final bound on the classification error $\exErr(\hPQ)$. We start our decomposition in a standard way as follows.

Fix any $x$ and let $\{X_{(i)}\}_1^k$ denote its $k$ nearest neighbors in $\mathbf{X} \doteq \mathbf{X}_P \cup \mathbf{X}_Q$. By a triangle inequality and the fact that $\eta$ is $(C_\alpha, \alpha)$ H\"older, we have: 
\begin{align} 
\left| \empReg_k(x) - \regFct(x) \right| 
 \leq \frac{1}{\nn}\left| \sum_{i = 1}^{\nn} \nnLab{i}{} - \regFct(\nnFeat{i}{}) \right| + \frac{\holderCoeff}{\nn} \sum_{i = 1}^{\nn} \dist( \nnFeat{i}{} , x )^{\holderExp}. \label{eq:errdecompositon1}
\end{align}

Now, although NN distances $\dist( \nnFeat{i}{} , x )$ over $\mathbf{X}$ can be bounded by 
the distance to the $k$-th NN of $x$ in \emph{either} samples $\mathbf{X}_P$ or $\mathbf{X}_Q$, this fails to capture the 
interaction between the two samples, as captured by $\gamma$. 
As it turns out, 
such interaction is captured by directly bounding $1$-NN (rather than $k$-NN) distances over $\mathbf{X}$ .

We therefore proceed by first reducing the problem of bounding $k$-NN distances to that of bounding $1$-NN distances,
where we adapt a technique of \citet[Section 6.3]{gyorfi2006distribution} to our transfer setting with two samples: 

\begin{mydefinition}[Implicit $1$-NNs] \label{def:implicit1NN}
Divide $\sample$ into $\nn$ disjoint batches each containing $\left \lfloor \frac{\sN}{\nn} \right \rfloor$ samples from $\sSample$ and $\left \lfloor \frac{\tN}{\nn} \right \rfloor$ samples from $\tSample$. 
Fix $x \in \spa$ and define $\{\tilde{X}_{i}\}_{i =1}^k$ as its $1$-NNs in each of the $\nn$ batches. Let the assignment to each batch consist of picking, without replacement, $\left \lfloor \frac{\sN}{\nn} \right \rfloor$ indices from $[n_P]$ and $\left \lfloor \frac{\tN}{\nn} \right \rfloor$ indices from $[n_Q]$, so that the $\tilde{X}_i$'s are i.i.d. given $x$.
\end{mydefinition}

It can then be shown that, for any fixed $x\in \mathcal{X}$ we have (see Lemma \ref{lem:biasBoundImplicit1NN} of Section \ref{sec:upper-bounds}) 
$\sum_{i = 1}^{\nn} \dist(\nnFeat{i}{} , x) ^{\holderExp} \leq  \sum_{i = 1}^{\nn} \dist( \tilde{X}_{i} , x )^{\holderExp}.$
Combining this last inequality with \eqref{eq:errdecompositon1}, it follows that $|\empReg_k(x)  - \regFct(x) |$ is at most 
\begin{align}
 \underbrace{\frac{1}{\nn}\left| \sum_{i = 1}^{\nn} \nnLab{i}{} - \regFct(\nnFeat{i}{}) \right|}_{G_1(x)}
+ \underbrace{\frac{\holderCoeff}{\nn} \sum_{i = 1}^{\nn} \left(\dist( \tilde{X}_i , x )^{\holderExp} - \Expectation_{\tilde{X}_1} \dist( \tilde{X}_1 , x )^{\holderExp}\right)}_{G_2(x)} + 
\underbrace{\holderCoeff\Expectation_{\tilde{X}_1} \dist( \tilde{X}_1 , x )^{\holderExp}}_{G_3(x)}. \label{eq:decompStep2}
\end{align} 
The decomposition in \eqref{eq:decompStep2} serves to further isolate terms that can be bounded uniformly over $x$, namely $G_1$ and $G_2$. We arrive at the following proposition. 

\begin{myproposition}[Error Decomposition] \label{prop:biasVarianceDecomp}
Let $1\leq \nn \leq \sN \vee \tN$ and let $\hPQ$ be the $\nn$-NN classifier on $\sample$. Consider any $x\in \mathcal{X}$ 
with $k$ nearest neighbors $\{X_{(i)}\}_{1}^k$, and implicit $1$-NN's $\{\tilde X_i\}_1^k$. Let $G_i(x)$, $i\in [3]$ denote the terms in \eqref{eq:decompStep2}, and define 
$\bigElement_{i}(x) \doteq  2 \left|\regFct(x)-{1}/{2}\right| \cdot\mathbbm{1}\left \{ \left| \regFct(x)-{1}/{2} \right| \leq 3G_{i}(x) \right \}$. We have:
\begin{equation}\label{ineq:propBiasVar}
 \quad \mathbb{E}[\exErr(\hPQ)] \leq   \,\, \mathbb{E}[ \bigElement_1(X) ] + \mathbb{E} [ \bigElement_2(X) ] +\mathbb{E} [ \bigElement_3(X) ], 
\end{equation}
where the expectations are taken over $\sample$ and $X$. 
\end{myproposition}
\begin{proof} 
Apply \eqref{eq:incatortrick} (with $Z \doteq |\eta -1/2|$) to the decomposition of \eqref{eq:decompStep2}, and 
conclude using \eqref{eq:firstBound}.  
\end{proof} 



\subsection{Proof of Theorem \ref{thm:expErrRates}}
The main arguments are given here inline, and require bias and variance bounds we establish in subsequent sections. 


Under both \dm~and \bcn, the terms $\mathbb{E}[ \bigElement_1(X) ] + \mathbb{E} [ \bigElement_2(X) ]$ in \eqref{ineq:propBiasVar} are of order 
$(1/\sqrt{k})^{(\beta +1)}$
as shown via concentration and an step-wise integration argument in Lemma \ref{lem:boundingVar}. The term $\mathbb{E} [ \bigElement_3(X) ]$ is bounded using Lemma \ref{lem:boundPhi3DM} (of Appendix \ref{app:upperBound}) for \dm, and Lemma \ref{lem:boundPhi3BCN} for \bcn. 
This last term accounts for $\gamma$. 

In both cases \dm, \bcn, $\mathbb{E}_{\sample}[\exErr(\hPQ)]$ is then bounded by 
\begin{align*}
\upConst_{1} \left( \frac{1}{\sqrt{\nn}} \right)^{\tTsyExp + 1} + \upConst_{2} \left( \left \lfloor \frac{\sN}{\nn} \right \rfloor^{(\rates - 2)/((\rates - 2) + \transMarginExp / \holderExp)} + \left \lfloor \frac{\tN}{\nn} \right \rfloor\right)^{-(\tTsyExp + 1)/(\rates - 2)}, 
\end{align*}
where, under \dm, $\rates = 2 + \tDmDim / \holderExp$, and under \bcn, 
$\rates = 2 + \tTsyExp + \tCovDim / \holderExp$. 

The upperbounds of Theorem \ref{thm:expErrRates} are then deduced by plugging in the value of $k = \Theta ( \sN ^{\rates / (\rates + \transMarginExp/ \holderExp)} + \tN )^{2 /\rates}$ (where $\rates$ is defined as in Theorem \ref{thm:expErrRates}).
The fact that the given setting of $k$ indeed yields the rates of Theorem \ref{thm:expErrRates} involves a bit of algebra
handled in Lemma \ref{lem:boundImpliesRates} of Appendix \ref{app:upperBound}. The rates for \bcn\, being of independent interest for vanilla $k$-NN, we provide all essential arguments in this section; similar (but more standard) arguments for \dm\, are instead given in Appendix \ref{app:upperBound}.

\subsection{Supporting Lemmas} 
\label{app:techLem}

The next two lemmas are proved in Appendix \ref{app:upperBound}.

\begin{lemma} [A useful inequality]\label{lem:basicIneq} 
Let $\alpha_1, \alpha_2, \beta_1, \beta_2 > 0$ and $a,b \geq 0$ such that $a+b>0$ and $\alpha_1 \beta_1 \leq 1$. Assume $\alpha_2 - \alpha_1 = \frac{1}{\beta2} - \frac{1}{\beta1}$. Then, defining $c = \max(a^{\beta_1}, b^{\beta_2})^{-1}$ we have:
\begin{equation*}
\left(a c^{\alpha_1} + b c^{\alpha_2}\right)^{-1} \leq 2\left(a^{1-\alpha_1 \beta_1} + b^{1 - \alpha_2 \beta_2}\right)^{-1}.
\end{equation*}



\end{lemma}

\begin{lemma}[Relating $k$-NN's to implicit 1-NNs] \label{lem:biasBoundImplicit1NN}
Fix $x \in \spa$. Let $\{\nnFeat{i}{}\}_{i =1}^{\nn}$ its $\nn$ nearest-neighbors as in Definition \ref{def:kNNClass} and $\{ \tilde{X}_{i} \}_{i =1}^{\nn}$ its $\nn$ implicit 1-NNs from Definition \ref{def:implicit1NN}. We have the following inequality:
\begin{equation*}
\sum_{i = 1}^{\nn} \dist(\nnFeat{i}{} , x) ^{\holderExp} \leq  \sum_{i = 1}^{\nn} \dist( \tilde{X}_{i} , x )^{\holderExp}.
\end{equation*}
\end{lemma}

\subsection{Bounding $\mathbb{E}[\bigElement_{1}(\featVar)]$ and $\mathbb{E}[\bigElement_{2}(\featVar)]$}

The following is a generalization of an integral approximation argument of \cite[Lemma 3.1]{audibert2007fast} adapted to our setting. In particular, in their result, the counterpart for the function $G_{k}$ is the \emph{regression error of a generic estimator}; 
here we extend their techniques to any $G_{k}$ depending on $\nn$. 

\begin{lemma} [A generic integration argument] \label{lem:chaining}
Consider a distribution $\tJoinProb$ with noise parameters $\tTsyExp, \tTsyCoeff > 0$ (see Definition \ref{def:doublingMeas}). Let $\{ G_{\nn}(\sample;\featVar)\}_{\nn = 1}^{\sN \vee \tN}$ a set of measurable functions of $\sample$ and $\featVar$ indexed by $\nn$, where $\featVar \sim \tProb$ independent of $\sample$. Suppose that there exist $\upConst, \lowConst > 0$, such that:
\begin{equation*}
\forall x \in \spa, \forall \nn \geq 1, \forall \epsilon > 0, \quad \mathbb{P}_{\sample}(G_{\nn}(\sample, x) \geq \epsilon) \leq \upConst \exp(-\lowConst \nn \epsilon^{2}).
\end{equation*}
Then the below expectation (taken w.r.t.~both $\sample$ and $\featVar$) is bounded as follows:
\begin{equation} \label{eq:chaining}
\mathbb{E}\left[\left| \regFct(X) - \frac{1}{2} \right| \cdot \mathbbm{1} \left\{  \left| \regFct(X) - \frac{1}{2} \right| \leq G_{\nn}(\sample,\featVar)\right\} \right] \leq 3 \upConst \cdot\tTsyCoeff \left( \frac{\tTsyExp + 1}{\lowConst \nn} \right)^{(\tTsyExp +1)/2}.
\end{equation}
\end{lemma}
The proof of the above lemma is given in Appendix \ref{app:upperBound}. 

\begin{lemma} [Bounding $\mathbb{E}(\bigElement_{1}(\featVar))$ and $\mathbb{E}(\bigElement_{2}(\featVar))$] \label{lem:boundingVar}
Consider $\bigElement_{1}$ and $\bigElement_{2}$ as defined in Proposition \ref{prop:biasVarianceDecomp}. Under both \dm~and \bcn~distributional regimes, there exists a constant $\upConst > 0$ such that:
\begin{equation*}
\mathbb{E}[\bigElement_{1}(\featVar)] +\mathbb{E}[\bigElement_{2}(\featVar)] \leq \upConst \left( \frac{1}{\sqrt{\nn}} \right)^{\tTsyExp + 1}.
\end{equation*}
\end{lemma}
\begin{proof}
We start with $\bigElement_{1}$.
Let $A_k(x)= A(\sample, x) \doteq \frac{3}{k} \left| \sum_{i = 1}^{\nn} \nnLab{i}{} - \regFct(\nnFeat{i}{}) \right|$. 
By Hoeffding's inequality, we have that $\forall \nn \geq 1, \forall x \in \spa, \forall \epsilon>0$:
\begin{equation}
\mathbb{P}_{\sample}(A_k(x)  \geq \epsilon ) = \mathbb{E}_{\featVect}\left[ \mathbb{P}_{\labVect|\featVect}\left( \frac{3}{k} \left| \sum_{i = 1}^{\nn} \nnLab{i}{} - \regFct(\nnFeat{i}{}) \right| \geq \epsilon\right) \right] \leq 
2 e^{-\frac{2}{9}\nn \epsilon^{2}}. \label{eq:varhighprob}
\end{equation}  

Now let $B_k(x) = B_k(\sample, x)$ denote the quantity \\
$\frac{3 \holderCoeff}{\nn} \sum_{i = 1}^{\nn} \left( \dist (\tilde{\featVar}_{i} , x )^{\holderExp} - \mathbb{E}_{\tilde{\featVar}_{1}} \left[ \dist( \tilde{\featVar}_{1} ,x )^{\holderExp} \right] \right)$.
 Again, using Hoeffding's inequality we have that $\forall \nn \geq 1, \forall x \in \spa, \forall \epsilon>0$:
\begin{equation}
\mathbb{P}_{\sample}(B_k(x) \geq \epsilon) \leq 2 \exp \left(-\frac{2 \nn \epsilon^{2}}{9 \holderCoeff^{2} \diamDom^{2 \holderExp}} \right). \label{eq:biashighprob}
\end{equation} 

Hence, we conclude by applying Lemma \ref{lem:chaining} twice to bound $\bigElement_1$ and $\bigElement_2$. 
\end{proof}

\subsection{Bounding $\mathbb{E}[\bigElement_{3}(\featVar)]$ under \bcn}
Recall that a first issue under \bcn \, is that $k$-NN distances are not uniformly bounded over $x$. 
However, they can be bounded by decomposition over finite covers of $\mathcal{X}_Q$; such intuition appears in 
previous work on $k$-NN, e.g., \citet{gyorfi2006distribution, kulkarni1995rates}. Here, our added difficulty is in that we consider $k$-NN distances over a combined sample $\mathbf{X}$ from two distributions. Our second issue, is how to (crucially) account for the noise parameter $\beta$. 
We start with a result concerning the \emph{tail} of such distances, whose proof is given in Appendix \ref{app:upperBound}. 

\begin{lemma}[Bounding 1-NN bias] \label{lem:kulkarniAdapt}
Let $(P, Q) \in \mathcal{T}_\bcn$, $\color{black} \epsilon \in (0,\diamDom^{\holderExp} \wedge 1]$ and
\begin{equation*}
A(\epsilon,x) \doteq \int_{\epsilon}^{\diamDom^{\holderExp}} \left(1 - \sProb(\cBall{x}{t^{1/\holderExp}})\right)^{\left \lfloor \frac{\sN}{\nn} \right \rfloor} \left(1 - \tProb(\cBall{x}{t^{1/\holderExp}}) \right)^{\left \lfloor \frac{\tN}{\nn} \right \rfloor} \diff t.
\end{equation*}
Then, there exist two constants $\upConst_{1}, \upConst_{2} > 0$ such that, when $\transMarginExp < \infty$:
\begin{equation*}
\mathbb{E}_{\tJoinProb}[A(\epsilon, \featVar)] \leq \left \{
\begin{matrix}
\upConst_{1} \left( \left \lfloor \frac{\sN}{\nn} \right \rfloor \epsilon^{(\transMarginExp + \tCovDim)/\holderExp - 1} + {\color{black}(\left \lfloor \frac{\tN}{\nn} \right \rfloor \vee 1)} \epsilon^{ \tCovDim/\holderExp - 1} \right)^{-1}, \quad & \text{ for } \holderExp < \tCovDim, \\
\upConst_{1} (\log(1/\epsilon) + \upConst_{2}) \left( \left \lfloor \frac{\sN}{\nn} \right \rfloor \epsilon^{\transMarginExp /\holderExp} + {\color{black} \left \lfloor \frac{\tN}{\nn} \right \rfloor \vee 1 }\right )^{-1}, \quad & \text{ for } \holderExp = \tCovDim,
\end{matrix}
\right.
\end{equation*}
and when $\transMarginExp = \infty$, the bound matches the limit of the above as $\gamma \to \infty$.
\end{lemma}

\begin{lemma} [Bounding $\mathbb{E}(\bigElement_{3}(\featVar))$ under \bcn] \label{lem:boundPhi3BCN}
Consider $\bigElement_{3}$ as defined in Proposition \ref{prop:biasVarianceDecomp}. We work under \bcn~regime. Assume $\transMarginExp < \infty$, there exist two constants $\upConst_{1}, \upConst_{2} > 0$ such that, for $\tCovDim > \holderExp$:
\begin{equation*}
\mathbb{E}[\bigElement_{3}(\featVar)] \leq \upConst_{1} \left( \left \lfloor \frac{\sN}{\nn} \right \rfloor^{\frac{(\rates - 2)}{(\rates - 2) + \transMarginExp / \holderExp}} + \left \lfloor \frac{\tN}{\nn} \right \rfloor\right)^{-(\tTsyExp + 1)/ (\rates - 2)},\, \rates \doteq \tCovDim / \holderExp + \tTsyExp + 2.
\end{equation*}
For $\tCovDim = \holderExp$, replace $\upConst_{1}$ above, by 
$\upConst_{1} \left( \log \left(\left \lfloor \frac{\sN}{\nn} \right \rfloor + \left \lfloor \frac{\tN}{\nn} \right \rfloor \right) + \upConst_{2}\right) $.

The case $\transMarginExp = \infty$ matches the limits (as $\gamma \to \infty$) of the above bounds . 
\end{lemma}

\begin{proof}
We start with a decomposition 
into small and larger nearest neighbor distances, captured by a \emph{tail} parameter $\epsilon>0$. We have that 
$\mathbb{E}_{\tilde{\featVar}_{1}} [\dist(\tilde{\featVar}_{1},x)^{\holderExp}]$ equals 
\begin{align*}
\int_{0}^{\diamDom^{\holderExp}} & \mathbb{P}_{\tilde{\featVar}_{1}} \left( \dist(\tilde{\featVar}_{1},x)^{\holderExp} > t \right) \diff t
= \int_{0}^{\diamDom^{\holderExp}} \mathbb{P}_{\tilde{\featVar}_{1}} \left( \dist(\tilde{\featVar}_{1},x) > t^{1/\holderExp} \right) \diff t\\
& = \int_{0}^{\diamDom^{\holderExp}} \left(1 - \sProb(\cBall{x}{t^{1/\holderExp}})\right)^{\left \lfloor \frac{\sN}{\nn} \right \rfloor} \left(1 - \tProb(\cBall{x}{t^{1/\holderExp}})\right)^{\left \lfloor \frac{\tN}{\nn} \right \rfloor} \diff t \\
& \leq \epsilon + \int_{\epsilon}^{\diamDom^{\holderExp}} \left(1 - \sProb(\cBall{x}{t^{1/\holderExp}})\right)^{\left \lfloor \frac{\sN}{\nn} \right \rfloor} \left(1 - \tProb(\cBall{x}{t^{1/\holderExp}})\right)^{\left \lfloor \frac{\tN}{\nn} \right \rfloor} \diff t \\
& \doteq \epsilon + A(\epsilon, x).
\end{align*}
We can now use the above to bound
$\mathbb{E}[\bigElement_{3}(\featVar)]$ as follows. Let $g(x) \doteq |\eta(x) - 1/2|$. We have that 
$\mathbb{E}[\bigElement_{3}(\featVar)]$ is at most 
\begin{align} \label{eq:boundPhi3Decomp}
 2 \mathbb{E}_{\tJoinProb} &\left[  g(X) \mathbbm{1}\left \{ g(X) \leq 3 \holderCoeff (\epsilon + A(\epsilon,\featVar))\right \} \right] \nonumber\\
& \leq 2 \mathbb{E}_{\tJoinProb} \left[ g(X) \mathbbm{1}\left \{ g(X) \leq 6 \holderCoeff \epsilon \right \} \right] 
 + 2 \mathbb{E}_{\tJoinProb} \left[ g(X) \mathbbm{1}\left \{ g(X) \leq 6 \holderCoeff A(\epsilon,\featVar)\right \} \right] \nonumber \\
& \leq 2 \tTsyCoeff (6\holderCoeff \epsilon)^{\tTsyExp + 1} + 12 \holderCoeff \mathbb{E}_{\tJoinProb}[A(\epsilon, \featVar)],
\end{align}

where we used equation \eqref{eq:ass1equation} from Definition \ref{def:transferCoefficient} in the last inequality. We now use Lemma \ref{lem:kulkarniAdapt} to bound $\mathbb{E}_{\tJoinProb}[A(\epsilon, \featVar)]$ for $\epsilon \leq 1$. Assume $\transMarginExp < \infty$. The case $\transMarginExp = \infty$ follows the same lines (and is in fact more direct). Recall that $\holderExp \leq \tCovDim$ and take
\begin{equation*}\color{black}
c_k \doteq \left( \max \left( \left \lfloor \frac{\sN}{\nn}\right \rfloor^{\holderExp/(\tCovDim + \transMarginExp + \tTsyExp \holderExp)}, \left \lfloor \frac{\tN}{\nn}\right \rfloor^{\holderExp/(\tCovDim + \tTsyExp \holderExp)} \right) \right)^{-1} \leq 1, 
\end{equation*}
and {\color{black}let $\epsilon = (\diamDom^{\holderExp} \wedge 1) \cdot c_k $, then }$\epsilon^{\beta +1}$ is of the desired order of the lemma's statement. Now let $a = \left \lfloor \frac{\sN}{\nn}\right \rfloor$ and $b = \left \lfloor \frac{\tN}{\nn}\right \rfloor$. Set $\alpha_{1} = \frac{\transMarginExp + \tCovDim - \holderExp}{ \holderExp}$, $\alpha_{2} = \frac{ \tCovDim - \holderExp}{ \holderExp}$, $\beta_{1} = \frac{\holderExp}{\tCovDim + \transMarginExp + \tTsyExp \holderExp}$, $\beta_{2} = \frac{\holderExp}{\tCovDim + \tTsyExp \holderExp}$. Notice that $\alpha_{1} \beta_{1} = \frac{\transMarginExp + \tCovDim - \holderExp}{\transMarginExp + \tCovDim + \tTsyExp \holderExp} \leq 1$ and $\alpha_{2} - \alpha_{1} = \frac{\transMarginExp}{\holderExp} = \frac{1}{\beta_{2}} - \frac{1}{\beta_{1}} $. {\color{black} Define $C = ( \diamDom^{\holderExp} \wedge 1)^{-\alpha_1}$}. Therefore we can apply Lemma \ref{lem:basicIneq} to bound $\mathbb{E}_{\tJoinProb}[A(\epsilon, \featVar)]$ as follows: {\color{black}
\begin{align*}  \label{eq:boundPhi3PlugIneq}
 &\left( \left \lfloor \frac{\sN}{\nn} \right \rfloor \epsilon^{\frac{\transMarginExp + \tCovDim}{\holderExp} - 1} 
+ \left \lfloor \frac{\tN}{\nn} \right \rfloor \epsilon^{\frac{ \tCovDim}{\holderExp} - 1}\right)^{-1} \leq C (a c_k^{\alpha_{1}} + b c_k^{\alpha_{2}})^{-1} \nonumber \\
\leq  \quad & 2 C \left( \left \lfloor \frac{\sN}{\nn} \right \rfloor^{\frac{\tTsyExp + 1}{(\rates - 2) + \transMarginExp / \holderExp}} + \left \lfloor \frac{\tN}{\nn} \right \rfloor^{\frac{\tTsyExp + 1}{(\rates - 2) }} \right)^{-1}
= 2 C \left( a^{1-\alpha_{1} \beta_{1}} + b^{1 - \alpha_{2} \beta_{2}}\right)^{-1},
\end{align*}}
which again is of the desired order (relevant inequalities are given in the appendix). 


The case $\tCovDim = \holderExp$ is treated the same way.
\end{proof}

\subsection{Analysis Outline for Adaptive Rates}
This section lays out the main intuition behind the adaptive rates of Theorem \ref{thm:genericadaptivity}, while the full proof is given in Appendix \ref{app:adaRates}. 

First, the classifier returned by Algorithm \ref{alg:adaptiveKNN} is defined as $\hat h \doteq \mathbbm{1}\{\hat \eta\geq 1/2\}$ for a $k$-NN regression estimate $\hat \eta$, where $k = k(x)$ is chosen adaptively \emph{at every $x$}. 
Namely, $\empReg(x)$ is chosen from a confidence interval on $\regFct(x)$  iteratively refined over $\nn$-NN regression estimates $\empReg_{\nn}(x)$ of $\regFct(x)$ for increasing values of $\nn$ in the range $\mathcal{K}$. These intervals are of the form $\hat \eta_k \pm 1/\sqrt{k}$, accounting for variance in the estimates, and are shown to \emph{overlap} -- i.e., they all contain $\eta(x)$ -- as long as variance dominates bias. The stopping condition is such that, whenever these intervals no longer overlap, 
the current value of $k$ is shown to approximately balances bias and variance, and in particular yields a regression bound $\Phi(x; n_P, n_Q, k^*)$ on $|\hat \eta(x) - \eta(x)|$, of similar order -- up to log terms -- as would be obtained with an optimal global choice $k^* = k^*(\gamma, \alpha, \beta, d)$ (a priori unknown). 

Up to this point, the main arguments are standard (see e.g. Chapter 9.9 of \cite{wasserman2006all}), but require specializing various details to our setting with non-identical data distributions. It now remains to show that the above regression rates translate into the right classification rates, especially given the earlier difficulties -- outlined in Section \ref{sec:upper-bounds} -- in accounting for $\beta$ under $\bcn$ where regression rates are not uniform in $x\in {\cal X}_Q$. 
Here again, inequality \eqref{eq:firstBound} comes in handy in showing that the classification error of $\hat h$ is of similar order as that of $\hat h_{k^*}$ since, pointwise we have 
$$\mathbbm{1}\{|\eta(x) - 1/2| < |\hat \eta(x) - \eta(x)|\} \leq \mathbbm{1}\{|\eta(x) - 1/2| < \Phi(x; n_P, n_Q, k^*)\},$$
where $\Phi(x; n_P, n_Q, k^*)$ upper-bounds both $|\hat \eta(x) - \eta(x)|$ and $|\eta_{k^*}(x) - \eta(x)|$. 
The rest of the argument is then identical to that laid out in Section \ref{sec:upper-bounds}.

\subsection{Analysis Outline for Adaptive Labeling}
\paragraph{Adaptive Classification Rates}
To show that Algorithm \ref{alg:adaptiveKNN} with $\sample_R$ as input, instead of $\sample$, achieves the same adaptive rates, we use the same argument as above, by first showing that $\sample_R$ maintains important properties of $\sample$. In particular, 
as shown by \cite{berlind2015active}, NN distances are approximately preserved; in our case we show that such distances are preserved uniformly over choices of $k$.



\paragraph{Labeling Complexity} The analysis relies on the main intuition below. 
Fix $x=X_i\in \mathbf{X}_Q$. Initially $\mathbf{X}_R = \mathbf{X}_P$. Thus, we won't request a label at $x$ if at least $k$ samples \emph{from} $\mathbf{X}_P$ fall in a neighborhood of $x$. In particular, if $n_P$ is sufficiently large with respect to $n_Q$, we can ensure that the smallest ball $B_Q(x)$ containing $k$ samples from $\mathbf{X}_Q$ must also contain $k$ samples from $\mathbf{X}_P$ (this follows from lower-bounding $P_X$-mass by $Q_X$-mass using the definition of $\gamma$). Now, if the smallest ball $B_{P, Q}(x)$ containing $2k$ samples from $\mathbf{X}$ contains $B_Q(x)$, we are done, i.e., $x$'s label won't be queried; 
otherwise $B_{P, Q}(x)$ has less than $k$ samples from $\mathbf{X}_Q$ and so must have at least $k$ samples from $\mathbf{X}_P$, in which case again there is no label query at $x$. The theorem formalizes these conditions on $B_Q(x)$, $n_P$, $n_Q$, $\gamma$.  

The detailed proofs are given in Appendix \ref{app:adaRates} and \ref{app:adaResults}.

\section*{Final Remarks}
The transfer-exponent $\gamma$ successfully captures the relative benefits of source and target data, as shown through matching upper and lower-bounds. Our results hold for nonparametric classification. However, other interesting transfer problems such as in parametric models of regression have received much attention in the literature \citep{blitzer2011domain, kuzborskij2013stability, hoffman2017multiple}; such problems certainly require separate consideration. 


\bibliography{refs}

\appendix
\section{Lower-bounds for $\bcnFam$, and the case $\transMarginExp =\infty$}
\label{app:lowerbounds}
\subsection{Lower Bound for $\family = \bcnFam$ when $\transMarginExp < \infty$}

\begin{myproposition} \label{prop:lbBcnFinite}
Let $(\spa,\dist) = ([0,1]^{\tDmDim},\| . \|_{\infty})$, for some $\tDmDim \in \mathbb{N}^{*}$, and assume that $\transMarginExp < \infty$. There exists a constant $\lowConst = \lowConst(\bcnFam)$ such that, for any classifier $\hat h$ learned on $\sample$ and with knowledge of $P_X, Q_X$, we have, for $\rates = 2 + \tTsyExp + \tCovDim / \holderExp$:
\begin{equation*}
\sup_{(\sJoinProb, \tJoinProb) \in \bcnFam} \mathbb{E}_{\sample}[\exErr(\hat h)] \geq \lowConst \left( \sN ^{\rates / (\rates + \transMarginExp / \holderExp)} + \tN \right)^{ -(\tTsyExp + 1) / \rates}. 
\end{equation*}
\end{myproposition}

\begin{proof}
The proof of the lower bound for the \bcn~regime follows almost all the same lines as the above lower bound proof of Proposition \ref{prop:lbDmFinite} for the \dm~regime. The only difference is that we don't have to satisfy the doubling measure assumption for $\tProb$ and hence we don't need the densities to be bounded away from zero independently of $\sN$ and $\tN$ as in equation \eqref{eq:firstCondOnCm}. In this case, we can set $\radius$, $\m$ and $\w$ as follows:
\begin{equation*}
\radius = \rConst {\color{black} \domBnd} \left( \sN^{\rates / (\rates + \transMarginExp / \holderExp)} + \tN \right)^{-1/(\holderExp \rates)} , \m = \left \lfloor \mConst {\color{black}\left( \frac{\radius}{\domBnd} \right)}^{- \tDmDim} \right \rfloor , \w = \wConst {\color{black}\left( \frac{\radius}{\domBnd} \right)}^{\tDmDim + \holderExp \tTsyExp },
\end{equation*}
where {\color{black}$\wConst =\domBnd^{\alpha\beta} \min(\mConst^{-1}\tTsyCoeff(\holderCoeff'/2)^{\tTsyExp}, 2^{-4} \log(2) \rConst^{-\holderExp \rates}\holderCoeff'^{-2}, 1/2)$}, $\rConst = 1/9$ and $\mConst = (8/9)^{\tCovDim}$, implying that $8 \leq \m < \left \lfloor \radius^{-1} \right \rfloor^{\tDmDim}$ and $\m \w < 1$. After, all the steps are identical to the lower bound proof for \dm, with the difference that $\rates = 2 + \tTsyExp + \tCovDim / \holderExp$ here. In particular, equations \eqref{eq:tsyMarginCond} and \eqref{eq:iiCond} are unchanged, {\color{black} and $\domBnd$ can be chosen low enough so that we can achieve any desired $\tDmCoeff$ or $\transMarginCoeff$}. Finally, from equation \eqref{eq:defineS}, we have again, $\forall 0 \leq i < j \leq \M$:
\begin{equation*}
\semiDist{\hStar_i}{\hStar_j} \geq \holderCoeff' \frac{\w \m \radius^{\holderExp}}{8 } \doteq 2 \s \geq 2 \dmLbConst \left( \sN ^{\rates / (\rates + \transMarginExp / \holderExp)} + \tN \right)^{ - (\tTsyExp + 1) /\rates},
\end{equation*}
where $\dmLbConst>0$. Thus, by applying Proposition \ref{prop:tsy25}, we get the desired lower bound. 
\end{proof}

\subsection{Lower Bounds when $\transMarginExp = \infty$}

\begin{myproposition} \label{prop:lbInfinite}
Let $(\spa,\dist) = ([0,1]^{\tDmDim},\| . \|_{\infty})$, for some $\tDmDim \in \mathbb{N}^{*}$ and consider $\transMarginExp = \infty$. Let $\family$  denote either $\dmFam$ or $\bcnFam$. {For $\family = \dmFam$ assume further that $\holderExp \tTsyExp\leq \tDmDim$} and when $\alpha \beta = \tDmDim$ the lower bound holds only when $\tTsyCoeff$ is higher than some threshold which depends on the other parameters of $\dmFam$ and is derived in the proof. There exists a constant $\lowConst = \lowConst(\family)$ such that, for any classifier $\hat h$ learned on $\sample$ and with knowledge of $P_X, Q_X$, we have: 
\begin{equation*}
\sup_{(\sJoinProb, \tJoinProb) \in \family} \mathbb{E}_{\sample}[\exErr(\hat h)] \geq \lowConst \left( 1 + \tN \right)^{ -(\tTsyExp + 1) / \rates}, 
\end{equation*}
where $\rates = 2  + \tCovDim / \holderExp$ when $\family = \dmFam$, and $\rates = 2 + \tTsyExp + \tCovDim / \holderExp$ when $\family = \bcnFam$.
\end{myproposition}
\begin{proof}
As the proof of the lower bound for $\transMarginExp = \infty$ is again quite similar to the previous ones, we treat both regimes \bcn~and \dm~simultaneously, by taking $\rates = 2 + \tDmDim/\holderExp$ when $\family = \dmFam$ and $\rates = 2 + \tTsyExp + \tCovDim / \holderExp$ when $\family = \bcnFam$. Actually, the main difference is the choice of the source marginal $\sProb$. Notice that because $\transMarginExp = \infty$ we have no restriction on the choice of such a probability measure. In particular, we could set the density of $\sProb$ being equal to zero on $\tDom$, and the proof would be even more direct. However, we do the proof of the lower bound with $\PDensity_{0}, \PDensity_{1} > 0$ to show that, indeed, the lower bound even holds for the situations where we have both $\tDom \subset \sDom$ and $\transMarginExp = \infty$. For $\rConst = 1/9$, we set:
\begin{equation*}
\radius = \rConst D \left( 1 + \tN \right)^{-1/(\holderExp \rates)},
\end{equation*}
and $\w$ and $\m$ are defined as in the previous proofs. The construction of the marginal $\tProb$ remains also the same. Recall that $\lbda_{1}$ is the density of $\tProb$ w.r.t.~Lebesgue measure on each set $\cBall{\p}{\radius/6}$ for $\p \in \grid_{1}$. We define $\sProb$ as having density $\PDensity_{1}$ on these sets as follows:
\begin{equation*}
 \PDensity_{1} = \frac{\lbda_{1}}{\sN}.
\end{equation*}

Note that when $\sN = 0$ we can actually choose any arbitrary distribution for $\sProb$ as it will no longer appear in inequality \eqref{eq:iiCond}.

Furthermore, as before, we let $\sProb$ be uniform on $\cBall{\p}{\radius/2}  \backslash  \cBall{\p}{\radius/3}$ for each $\p \in \grid_{1}$ so that $\sProb(\cBall{\p}{\radius/2})= \tProb(\cBall{\p}{\radius/2})$, and similarly we let $\sProb$ to have the same density $\PDensity_{0} = \lbda_{0}$ as $\tProb$ on the hypercubes $\cBall{\p}{\radius/2}$ for all $\p \in \grid_{0}$. The main arguments remain unchanged, apart from the bound on Kullback-Leibler divergence which changes as follows, when $\sN > 0$: $\forall i \in \{1, \ldots, \M \}$
\begin{align*}
 \KLDiv{\sJoinProb^{i}}{\sJoinProb^{0}}
&= \sN^{-1} \hammingDist{\sig_{i}}{\sig_{0}} \w \log \left( \frac{1 + \holderCoeff' \radius^{\holderExp }}{1 - \holderCoeff' \radius^{\holderExp}} \right)\holderCoeff' \radius^{\holderExp}\\ 
&\leq \sN^{-1} \m\w \holderCoeff'^2\radius^{2\holderExp}/(1 - \holderCoeff' \radius^{\holderExp})
\leq 2 \sN^{-1} \m \w  \holderCoeff'^2\radius^{2\holderExp}.
\end{align*}

Hence, equation \eqref{eq:iiCond} becomes: $\forall i \in \{1, \ldots, \M \}$
\begin{align*}
\quad \KLDiv{\sampleDist_{i}}{\sampleDist_{0}} &\leq 2 \m \w  \holderCoeff'^2 \radius^{2\holderExp} (\tN + 1) \\
&\leq 2  \rConst^{\holderExp \rates} \wConst \holderCoeff'^{2} \m  
\leq 2^{4}  \log(2)^{-1}   \rConst^{\holderExp \rates} \wConst \holderCoeff'^{2} \log(\M).
\end{align*}

Note also that equation \eqref{eq:defineS} is now as follows:
\begin{equation*}
\forall 0 \leq i < j \leq \M, \quad \semiDist{\hStar_i}{\hStar_j} \geq \holderCoeff' \frac{\w \m \radius^{\holderExp}}{8 } \doteq 2 \s \geq 2 \dmLbConst \left( 1 + \tN \right)^{ - (\tTsyExp + 1) /\rates}.
\end{equation*}

Finally, we can apply Proposition \ref{prop:tsy25} to get the lower bound of Theorem \ref{thm:minimaxLowerBound} for the case where $\transMarginExp = \infty$.
\end{proof}


\subsection{Extension: Lower Bound for Differing Support Dimensions}
\label{app:example4}

Consider the case of Example \ref{ex:diffDim}, i.e., where $P$ and $Q$ are of different dimensions $d_P, d_Q$, under $\dm$, with $\gamma = d_P - d_Q$. In this case, the upper bound of Theorem \ref{thm:expErrRates} becomes 
    $$O\left ( n_P^{{(2\alpha + d_Q)}/{(2\alpha + d_P)}} + n_Q   \right)^{-1/(2\alpha + d_Q)}.$$
The \emph{effective sample size contributed by $P$}, namely the term $n_P^{{(2\alpha + d_Q)}/{(2\alpha + d_P)}}$ can be explained through the effective mass of sample points $\sim P_X$ near the support ${\cal X}_Q$ of $Q_X$ (as pointed out by one of the reviewers): roughly, if the mass of a ball $B(x, r)$ under $P$ behaves as $r^{d_P}$, then (up to curvature), the mass under $P_X$ of the envelope 
${\cal X}_{Q+r} \doteq \cup_{x\in {\cal X}_Q } B(x, r)$ behaves like $ r^{d_P - d_Q}$ (assuming ${\cal X}_Q$ is of diameter bounded by $1$). Now for an optimal choice of $$k \approx \left (n_P^{(2\alpha + d_Q) /(2\alpha + d_P)}+ n_Q\right)^{2\alpha /(2\alpha + d_Q)} \geq n_P^{2\alpha/(2\alpha + d)},$$ 
the $k$ nearest neighbors of any $x \in {\cal X}_Q$ are at distance at most $r \doteq (k/n_P)^{1/d_P}\approx n_P^{-1/(2\alpha + d_P)}$ in expectation. Then, for this $r$, the number of datapoints contributed from $P$ to 
${\cal X}_{Q+r}$ would be of order at most $n_P\cdot r^{d_P - d_Q} = n_P^{(2\alpha + d_Q)/(2\alpha + d_P)}$. 

This intuition is further validated by the alternative lower bound construction of Proposition \ref{prop:lbDmLowerDim} below which covers this situation where $P$ and $Q$ are of different dimensions. 

\begin{myproposition} \label{prop:lbDmLowerDim}
Let $(\spa,\dist) = ([0,1]^{\tDmDim + \transMarginExp},\| . \|_{\infty})$, for some $\tDmDim, \transMarginExp \in \mathbb{N}^{*}$, and let $\alpha \beta \leq d$. Consider $\dmFam$ as in Definition \ref{def:distClass}, except that when $\alpha \beta = \tDmDim$ the lower bound holds only when $\tTsyCoeff$ is higher than some threshold which depends on the other parameters of $\dmFam$ and is derived in the proof. Call $\mathcal{T}^{\tDmDim}_{\dm}$ the family of distribution tuples $(\sJoinProb,\tJoinProb)$ from $\dmFam$ such that $\tProb$ has support of dimension $\sDmDim$ and $\sProb$ has support of dimension $\sDmDim + \transMarginExp$. There exists a constant $\lowConst = \lowConst(\mathcal{T}^{\tDmDim}_{\dm})$ such that, for any classifier $\hat h$ learned on $\sample$ and with knowledge of $P_X, Q_X$, we have, for $\rates = 2 + \tCovDim / \holderExp$:
\begin{equation*}
\sup_{(\sJoinProb, \tJoinProb) \in \mathcal{T}^{\tDmDim}_{\dm}} \mathbb{E}_{\sample}[\exErr(\hat h)] \geq \lowConst \left( \sN ^{\rates / (\rates + \transMarginExp / \holderExp)} + \tN \right)^{ -(\tTsyExp + 1) / \rates}. 
\end{equation*}
\end{myproposition}

\begin{proof}
The proof of Proposition \ref{prop:lbDmLowerDim} follows similar steps than for Proposition \ref{prop:lbDmFinite} apart from a few modifications outlined below.

$\tProb$ is now built on a subset of $\spa$ of dimension $\sDmDim$: $\spa_{\tDmDim} \doteq [0,D]^\tDmDim \times \{ 1/2 \}^\transMarginExp$, while the definitions of $\radius$, $\m$ and $\w$ remain identical. The construction of $\tProb$ is the same as in Proposition \ref{prop:lbDmFinite} where $\spa'$ is replaced by $\spa_{\tDmDim}$ and every balls $\cBall{\p}{\radius/2}$ or $\cBall{\p}{\radius/6}$ are replaced by their $\tDmDim$-dimensional restrictions to $\spa_{\tDmDim}$, denoted $B_{\tDmDim}(\p,\radius/2)$ and $B_{\tDmDim}(\p,\radius/6)$. That is, we subdivide $\spa_{\tDmDim}$ in $\lfloor \domBnd \radius^{-1} \rfloor^{\tDmDim}$ hypercubes of dimension $d$ and, as in Proposition \ref{prop:lbDmFinite}, we split the set $\grid$ of the centers of these hypercubes into two disjoint subsets $\grid_0$ and $\grid_1$ such that $|\grid_1| = \m$. We then set $\tProb$ to have a uniform density $\lbda_{1}$ on each set $B_{\tDmDim}(\p,\radius/6)$ for $\p \in \grid_1$ such that $Q_X(B_{\tDmDim}(\p,\radius/6)) = \w$, and the remaining mass $1 - \m\w$ of $\tProb$ is distributed uniformly over $\spa_{0} \doteq \cup_{\p \in \grid_{0}} \cBall{\p}{\radius/2}$ with density $\lbda_{0}$. Note, the densities $\lbda_0$ and $\lbda_1$ are therefore with respect to the Lebesgue measure on $\spa_{\tDmDim}$, and also satisfy inequalities \eqref{eq:firstCondOnCm}, meaning that $\domBnd$ can be chosen arbitrarily small to achieve any desired $\tDmCoeff$.

The marginal $\sProb$, on the other hand, is built on a support of full dimension $d + \gamma$ in the space $\spa$. For every $\p \in \grid_{0}$ we set $\sProb$ to be uniformly distributed on $\cBall{\p}{\radius/2}$ with density $\PDensity_{0}$ such that $\sProb(\cBall{\p}{\radius/2}) = \tProb(B_{\tDmDim}(\p,\radius/2))$, and for every $\p \in \grid_{1}$ we set $\sProb$ to be uniformly distributed on $\cBall{\p}{\radius/6}$ with density $\PDensity_{1} = \lbda_{1} (\domBnd / 3)^{-\transMarginExp}$. Here the densities are understood to be w.r.t.~the Lebesgue measure on $\spa$. Hence we have
\begin{align*}
\forall \p \in \grid_{1}, \quad \sProb(\cBall{\p}{\radius/6}) &= \lbda_{1} (\domBnd/3)^{-\transMarginExp} \left(\frac{\radius}{3}\right)^{\tDmDim + \transMarginExp} = \tProb(B_{\tDmDim}(\p,\radius/6)) \domBnd^{-\transMarginExp} \radius^{\transMarginExp} \\
& \leq \tProb(B_{\tDmDim}(\p,\radius/6)) = \w.
\end{align*}
Then, we also put the remaining mass on $\cBall{\p}{\radius/2} \backslash \cBall{\p}{\radius/3}$ in each of these hypercubes such that $\sProb(\cBall{\p}{\radius/2}) = \tProb(B_{\tDmDim}(\p,\radius/2)) = \w, \,\, \forall \p \in \grid_{1}$. Recall that in order to verify that \eqref{eq:ass1equation} holds we only need to check the inequality only at points $x$ in the support of $\tProb$. We deduce from this an inequality similar to \eqref{eq:verif1trans}, that is for any $\p \in \grid_{1}$, we have that $\forall x \in B_{\tDmDim}(\p,\radius/6), \forall \tempRadius \in [0, \radius /3]$,
\begin{align*}
 \sProb(\cBall{x}{\tempRadius}) &\geq \PDensity_{1} \text{vol}(\cBall{x}{\tempRadius} \cap \cBall{\p}{\radius/6}) \geq (\domBnd / 3)^{-\transMarginExp}  \tempRadius^{\transMarginExp} \tProb(B_{\tDmDim}(x,\tempRadius)).
\end{align*}

Finally, for $\p \in \grid_{0}$, and $\forall x \in B_{\tDmDim}(\p,\radius/2), \forall \tempRadius \in [0, \radius /3]$, we also see that an inequality similar to \eqref{eq:verif2trans} is satisfied:
\begin{align*}
\sProb(\cBall{x}{\tempRadius}) &\geq \sProb(\cBall{x}{\tempRadius} \cap (\cup_{\p \in \grid_{0}} \cBall{\p}{\radius/2})) \\
&\geq p_0 \text{vol}(\cBall{x}{\tempRadius} \cap (\cup_{\p \in \grid_{0}} \cBall{\p}{\radius/2})) \\
&\geq \domBnd^{-\transMarginExp} (2\tempRadius)^{\transMarginExp} \tProb(B_{\tDmDim}(x,\tempRadius)), 
\end{align*}
since we have
$$
\forall \p \in \grid_{0}, \quad \PDensity_{0} = \tProb(B_{\tDmDim}(\p,\radius/2)) / \radius^{\tDmDim + \transMarginExp} = \lbda_0 / \radius^\transMarginExp \geq \lbda_0 \domBnd^{-\transMarginExp}.
$$

 Such inequalities show that, as in the proof of Proposition \ref{prop:lbDmFinite}, we can achieve any $\transMarginCoeff$ by taking a small enough $\domBnd$, if necessary. Therefore, condition \eqref{eq:ass1equation} from Definition \ref{def:transferCoefficient} is also satisfied for the above construction of $(\sProb, \tProb)$.

Now, in the definitions of the excess error in \eqref{eq:distExcessEquiv} and the ensuing semi-distance $\semiDist{\cdot}{\cdot}$, simply replace 
$\cBall{\p}{\radius/2}$ by $B_{\tDmDim}(\p,\radius/2)$. 

All remaining steps of the proof are then identical to those in Proposition \ref{prop:lbDmFinite}: in particular the conditional distributions are defined the same way and both inequalities \eqref{eq:defineS} and \eqref{eq:iiCond} are satisfied. To be more precise, we define the $(\holderCoeff, \holderExp)$--H\"{o}lder functions exactly as in the proof of Proposition \ref{prop:lbDmFinite}, that is:
\begin{equation*}
\forall \p \in \grid_1, \quad \regFct_{\p}(x) \doteq \holderCoeff' \radius^{\holderExp} \holderFct^{\holderExp}( \| x - \p \|_{\infty} / \radius),
\end{equation*}
where here $\| \cdot \|_{\infty} $ refers to the infinite norm on the (full dimensional) space $\spa$. Then we define the regression functions $\regFct_{\sig}$ and the distributions $(\sJoinProb^{\sig}, \tJoinProb^{\sig})$, $\sampleDist_{\sig}$ and $\sampleDist_{i}$ also the same way, so that the family of distributions satisfies condition (i) of Proposition \ref{prop:tsy25}. Furthermore, condition (ii) of Proposition \ref{prop:tsy25} holds
in the same way here as in \eqref{eq:iiCond}, since it relies only on the mass that either $\sProb$ or $\tProb$ assigns to regions where the regression functions $\regFct_{\sig}$ are not equal to $1/2$. 

The result then follows. 
\end{proof}

\section{Upper Bound Analysis} 
\label{app:upperBound} 

\subsection{Proof of Technical Lemmas}

The inequality of Lemma \ref{lem:basicIneq} are combined with other useful inequalities in the following lemma. 

\begin{lemma} [Basic inequalities]\label{lem:basicIneq2} We have the following inequalities:

\begin{enumerate}
\item \label{ineq:basicIneq1} Take $\alpha \geq 1$ and $a,b \geq 0$, then:
\begin{equation*}
a^\alpha + b^\alpha \leq (a+b)^\alpha \leq 2^{\alpha - 1}(a^\alpha + b^\alpha).
\end{equation*}

\item \label{ineq:basicIneq2} Take $\alpha, \alpha' > 0$ and $a,b \geq 0$ such that $a+b>0$. Then if $\alpha \geq 1$:
\begin{equation*}
\left( \frac{1}{a^{\alpha/\alpha'}+ b}  \right)^{\frac{1}{\alpha}} \geq \frac{1}{a^{1/\alpha'} + b^{1/\alpha}} \geq \frac{1}{2} \left( \frac{1}{a^{\alpha/\alpha'} + b} \right)^{\frac{1}{\alpha}},
\end{equation*}
and when $\alpha < 1$, we have:
\begin{equation*}
2^{1 - 1 / \alpha} \frac{1}{a^{1/\alpha'} + b^{1/\alpha}} \leq \left( \frac{1}{a^{\alpha/\alpha'}+ b}  \right)^{\frac{1}{\alpha}} \leq \frac{1}{a^{1/\alpha'} + b^{1/\alpha}}.
\end{equation*}

\item \label{ineq:basicIneq3} Take $\alpha_1, \alpha_2, \beta_1, \beta_2 > 0$ and $a,b \geq 0$ such that $a+b>0$ and $\alpha_1 \beta_1 \leq 1$. Assume $\alpha_2 - \alpha_1 = \frac{1}{\beta2} - \frac{1}{\beta1}$. Then, for $c = \max(a^{\beta_1}, b^{\beta_2})^{-1}$ we have:
\begin{equation*}
\frac{1}{a c^{\alpha_1} + b c^{\alpha_2}} \leq \frac{2}{a^{1-\alpha_1 \beta_1} + b^{1 - \alpha_2 \beta_2}}.
\end{equation*}
\end{enumerate}

\end{lemma}
\begin{proof}
Inequalities (a) are well-known and inequalities (b) are direct consequences of the later. So we need just to prove inequality (c). Note that the cases $a = 0$ or $b = 0$ are trivial, so we can restrict ourselves to the situation where both $a>0$ and $b>0$. Plugging in the expression of $c$ we get:
\begin{equation*}
\frac{1}{a c^{\alpha_1} + b c^{\alpha_2}} = \frac{1}{\min(a^{1 -\alpha_1 \beta_1}, a  b^{-\alpha_1 \beta_2} ) + \min(b a^{- \alpha_2 \beta_1}, b^{1 - \alpha_2 \beta_2} )}.
\end{equation*}

Note that:
\begin{align} \label{ineq:basicIneqProof}
 & \quad a^{1- \alpha_1 \beta_1} \leq a b^{-\alpha_1 \beta_2}
\Leftrightarrow  \quad a^{- \alpha_1 \beta_1} \leq b^{-\alpha_1 \beta_2}
\Leftrightarrow 	\quad a^{- \beta_1} \leq b^{- \beta_2}\\
\Leftrightarrow 	& \quad a^{- \alpha_2 \beta_1} \leq b^{-\alpha_2 \beta_2}
\Leftrightarrow 	\quad b a^{- \alpha_2 \beta_1} \leq b^{1 -\alpha_2 \beta_2}. \nonumber
\end{align}

This means that $a^{1- \alpha_1 \beta_1}$ is minimum in the left component of the denominator if and only if $ b a^{- \alpha_2 \beta_1}$ is minimum in the right component. 
First, assume that it is $a^{1- \alpha_1 \beta_1}$ the minimum in the left component. Recall that $\alpha_1 \beta_1 \leq 1$ and $\alpha_2 - \alpha_1 = \frac{1}{\beta2} - \frac{1}{\beta1} \Leftrightarrow \frac{\beta_2}{\beta_1} - \alpha_1 \beta_2 = 1 - \alpha_2 \beta_2$. In this case, from equation (\ref{ineq:basicIneqProof}) we have:
\begin{equation*}
a^{-\beta_1} \leq b^{-\beta_2} \Leftrightarrow a \geq b^{\beta_2 / \beta_1} \Rightarrow a^{1 - \alpha_1 \beta_1} \geq b^{\frac{\beta_2}{\beta_1} - \alpha_1 \beta_2}  = b^{1 - \alpha_2 \beta_2}.
\end{equation*}

Hence, we can notice that $a^{1 - \alpha_1 \beta_1} \geq \frac{1}{2} a^{1 - \alpha_1 \beta_1} + \frac{1}{2} b^{1 - \alpha_2 \beta_2}$. This lead us to the result:
\begin{equation*}
\frac{1}{a c^{\alpha_1} + b c^{\alpha_2}} \leq \frac{1}{a^{1 - \alpha_1 \beta_1}} \leq \frac{2}{a^{1-\alpha_1 \beta_1} + b^{1 - \alpha_2 \beta_2}}.
\end{equation*}

Now assume that $b^{1-\alpha_{2}\beta_{2}}$ is strictly the minimum in the right component (recall that, by (\ref{ineq:basicIneqProof}), this is equivalent to $a^{-\beta_{1}} > b^{-\beta_{2}}$), we have:
\begin{equation*}
a^{-\beta_1} > b^{-\beta_2} \Leftrightarrow a < b^{\beta_2 / \beta_1} \Rightarrow a^{1 - \alpha_1 \beta_1} \leq b^{\frac{\beta_2}{\beta_1} - \alpha_1 \beta_2}  = b^{1 - \alpha_2 \beta_2}.
\end{equation*}

Therefore, again we have $b^{1 - \alpha_2 \beta_2} \geq \frac{1}{2} a^{1 - \alpha_1 \beta_1} + \frac{1}{2} b^{1 - \alpha_2 \beta_2}$ from which we can conclude:
\begin{equation*}
\frac{1}{a c^{\alpha_1} + b c^{\alpha_2}} \leq \frac{1}{b^{1 - \alpha_2 \beta_2}} \leq \frac{2}{a^{1-\alpha_1 \beta_1} + b^{1 - \alpha_2 \beta_2}}.
\end{equation*}
\end{proof}

\begin{proof}[Proof of Lemma \ref{lem:biasBoundImplicit1NN}]
The relation in fact holds generally for any subset of size $\nn$ of the data $\sample$.
That is, for any $\{ \featVar'_{i} \}_{i = 1}^{\nn} \subset \{ \featVar_{i} \}_{i = 1}^{\sN + \tN}$ we have:
\begin{equation} \label{ineq:biasImplicitIneq}
\sum_{i = 1}^{\nn} \dist(\nnFeat{i}{} , x) ^{\holderExp} \leq  \sum_{i = 1}^{\nn} \dist( X'_{i} , x )^{\holderExp}.
\end{equation}

Indeed, assume WLOG that $\dist(\featVar'_{1},x) \leq \ldots \leq \dist(\featVar'_{\nn},x)$. Then, $\featVar'_{i}$ is in fact the $i$th nearest neighbor of $x$ from $\{ \featVar'_{i} \}_{i = 1}^{\nn}$, while $\nnFeat{i}{}$ is its $i$th nearest neighbor from $\{ \featVar_{i} \}_{i = 1}^{\sN + \tN}$. As $\{ \featVar'_{i} \}_{i = 1}^{\nn} \subset \{ \featVar_{i} \}_{i = 1}^{\sN + \tN}$, this clearly implies that $\forall i \in \{1, \ldots, \nn \}, \,\, \dist(\nnFeat{i}{},x) \leq \dist(\featVar'_{i},x)$. Inequality \eqref{ineq:biasImplicitIneq} follows.
\end{proof}

\subsection{Proof of Theorem \ref{thm:expErrRates}}

\begin{lemma} [Plugging in the value of $\nn$] \label{lem:boundImpliesRates}
Let the exponent $\rates>2$ be as defined in Theorem \ref{thm:expErrRates}, that is, $\rates = 2  + \tCovDim / \holderExp$ when $\family = \dmFam$, and $\rates = 2 + \tTsyExp + \tCovDim / \holderExp$ when $\family = \bcnFam$. Recall that $1 \leq \nn \leq \sN \vee \tN$. Suppose that for some constant $\upConst_{1} > 0$, $k$ is upper-bounded as 
\begin{equation*}
\nn \leq \upConst_{1} \left(\sN^{\rates / (\rates + \transMarginExp / \holderExp) } + \tN \right)^{2/\rates}.
\end{equation*}
Then, for some constant $\upConst_{2} > 0$ we have that:
\begin{equation} \label{eq:resultBIR}
\left( \left \lfloor \frac{\sN}{\nn} \right \rfloor^{(\rates - 2) / ((\rates - 2) + \transMarginExp / \holderExp)} + \left \lfloor \frac{\tN}{\nn} \right \rfloor \right)^{-1/ (\rates - 2)} \leq \upConst_{2}(\sN^{\rates / (\rates + \transMarginExp / \holderExp)} + \tN)^{-1/\rates}.
\end{equation}
\end{lemma}
\begin{proof}
From result \ref{ineq:basicIneq2} of Lemma \ref{lem:basicIneq}, note that proving the bound (\ref{eq:resultBIR}) is in fact equivalent to proving that there exists a constant $\upConst_{2}>0$ such that:
\begin{equation*}
\left( \left (\frac{\sN}{\nn} \right )^{1/((\rates - 2) + \transMarginExp / \holderExp)} + \left (\frac{\tN}{\nn} \right )^{1/(\rates - 2)} \right)^{-1} \leq \upConst_{2} \left( \sN^{1/(\rates + \transMarginExp / \holderExp)} + \tN^{1/ \rates}\right)^{-1}.
\end{equation*} 

To further apply Lemma \ref{lem:basicIneq}, remark that we can rewrite the upper bound on $\nn$ as:
\begin{equation*}
\nn \leq 2^{2/\rates} \upConst_{1} \max \left( \sN^{2/(\rates + \transMarginExp / \holderExp)}, \tN^{2 / \rates} \right).
\end{equation*}

We can use finally result \ref{ineq:basicIneq3} from Lemma \ref{lem:basicIneq} by setting: $\alpha_{1} = 1/((\rates - 2) + \transMarginExp / \holderExp)$ and $\alpha_{2} = 1/ (\rates - 2)$ ; $\beta_{1} = 2(\alpha_{1})^{-1}/(\rates + \transMarginExp / \holderExp)$ and $\beta_{2} = 2(\alpha_{2})^{-1}/\rates$ ; $a = \sN^{\alpha_{1}}$ and $b = \tN^{\alpha_{2}}$. Now let's verify the conditions of result \ref{ineq:basicIneq3}:
\begin{align*}
& 1-\alpha_{1} \beta_{1} = \frac{(\rates - 2) + \transMarginExp / \holderExp}{\rates + \transMarginExp / \holderExp} = \frac{\alpha_{1}^{-1}}{\rates + \transMarginExp / \holderExp} \geq 0, 1-\alpha_{2} \beta_{2} = \frac{(\rates - 2) }{\rates} = \frac{\alpha_{2}^{-1}}{\rates}\\
& \frac{1}{\beta_{2}} - \frac{1}{\beta_{1}} = \alpha_{2} - \alpha_{1} + \frac{(\rates - 2)}{2} \alpha_{2} - \frac{(\rates - 2) + \transMarginExp / \holderExp}{2}\alpha_{1} = \alpha_{2} - \alpha_{1}.
\end{align*}

Therefore we conclude by using inequality (\ref{ineq:basicIneq3}) from Lemma \ref{lem:basicIneq}. For some constant $\upConst_{2}>0$, we have:
\begin{align*}
\left( \left(\frac{\sN}{\nn} \right)^{1/((\rates - 2) + \transMarginExp / \holderExp)} + \left(\frac{\tN}{\nn} \right)^{1/(\rates - 2)} \right)^{-1} & \leq \upConst_{2} \left( a^{1 - \alpha_{1} \beta_{1}} + b^{1 - \alpha_{2} \beta{2}}\right)^{-1}\\
& \leq \upConst_{2} \left( \sN^{1/(\rates + \transMarginExp / \holderExp)} + \tN^{1/\rates}\right)^{-1}.
\end{align*}
\end{proof}

\subsection{Bounding $\mathbb{E}[\bigElement_{1}(\featVar)]$ and $\mathbb{E}[\bigElement_{2}(\featVar)]$}

\begin{proof}[Proof of Lemma \ref{lem:chaining}]
By using Fubini theorem along with the bound assumed in the lemma, we have:
\begin{align*}
& \mathbb{E}\left[\left| \regFct(X) - \frac{1}{2} \right| \mathbbm{1} \left\{  \left| \regFct(X) - \frac{1}{2} \right| \leq G_{\nn}(\sample,\featVar)\right\} \right]\\
&= \,\, \mathbb{E}_{\tJoinProb} \left[\left| \regFct(x) - \frac{1}{2}\right| \mathbb{P}_{\sample} \left(G_{\nn}(\sample, \featVar) \geq \left| \regFct(X) - \frac{1}{2} \right| \right) \right]\\
& \leq \,\,  \upConst \mathbb{E}_{\tJoinProb}\left[\left| \regFct(X) - \frac{1}{2} \right| e^{ -\lowConst \nn \left| \regFct(\featVar) - 1/2 \right|^{2}} \right].
\end{align*}

Let $\delta = \sqrt{\frac{\tTsyExp + 1}{\lowConst \nn}}$ and $\delta_{i} = i.\delta$ for $i \geq 0$. Call $A_{i} = \{ x: \left| \regFct(x) - \frac{1}{2} \right| \in ( \delta_{i}, \delta_{i+1}] \}$. We can decompose the above expectation over the disjoint sets $A_i$ as:
\begin{align} \label{eq:summation}
\sum_{i \geq 0} \mathbb{E}_{\tJoinProb}\left[\left| \regFct(X) - \frac{1}{2} \right| e^{-\lowConst \nn \left| \regFct(\featVar) - 1/2 \right|^{2}} \mathbbm{1} \{\featVar \in A_i \}\right].
\end{align}

Now each term in the above sum is upper-bounded by
\begin{align} \label{eq:boundTheTerm}
&\delta_{i+1} e^{-\lowConst \nn \delta_{i}^{2}} \tProb \left(\delta_{i} < \left| \regFct(x) - 1/2 \right| \leq \delta_{i+1} \right) \nonumber \\
&\leq \,\, \delta (i+1) e^{-\lowConst \nn \delta^{2} i^{2}} \tProb \left(0 < \left| \regFct(x) - 1/2 \right| \leq \delta_{i+1} \right) \nonumber \\
&\leq \,\, \tTsyCoeff \delta^{\tTsyExp +1} (i + 1)^{\tTsyExp + 1} e^{-\lowConst \nn \delta^{2} i^{2}} \nonumber \\
&\leq \,\, \tTsyCoeff \left( \frac{\tTsyExp + 1}{\lowConst \nn} \right)^{(\tTsyExp + 1)/2} (i + 1)^{\tTsyExp + 1}  e^{-(\tTsyExp + 1) i^{2}},
\end{align}
where we used Definition \ref{def:noise} in the second inequality, and replaced $\delta$ by its value in the last one. Now, we have:
\begin{align} \label{eq:boundSum}
& \sum_{i \geq 0} (i+1)^{\tTsyExp + 1} e^{-(\tTsyExp + 1) i^{2}} = \sum_{i \geq 0} e^{-(\tTsyExp + 1) i^{2} + (\tTsyExp + 1)\log(i+1)} \leq \sum_{i \geq 0} e^{-(\tTsyExp + 1) i^{2} + (\tTsyExp + 1)i}  \nonumber \\
\leq \,\, & \sum_{i \geq 0} e^{-(\tTsyExp + 1) i (i - 1)} = 1 + \sum_{i \geq 1} e^{-(\tTsyExp + 1) i (i - 1)} \leq 1 + \frac{1}{1-\exp(-(\tTsyExp + 1))} \leq 3.
\end{align}

Therefore, from equation (\ref{eq:summation}) and using the inequalities from (\ref{eq:boundTheTerm}) and (\ref{eq:boundSum}), we finally get inequality (\ref{eq:chaining}) of the lemma.

\end{proof}

\subsection{Bounding $\mathbb{E}[\bigElement_{3}(\featVar)]$ under \dm}

\begin{lemma} [Bounding $\mathbb{E}(\bigElement_{3}(\featVar))$ under \dm] \label{lem:boundPhi3DM}
Consider $\bigElement_{3}$ as defined in Proposition \ref{prop:biasVarianceDecomp}. Under \dm, there exists a constant $\upConst > 0$ such that
\begin{equation*}
\mathbb{E}[\bigElement_{3}(\featVar)] \leq \upConst \left( \left \lfloor \frac{\sN}{\nn} \right \rfloor^{\frac{\tDmDim}{ \tDmDim + \transMarginExp}} + \left \lfloor \frac{\tN}{\nn}\right \rfloor \right)^{-\holderExp(\tTsyExp + 1)/ \tDmDim}.
\end{equation*}
\end{lemma}
\begin{proof}
Recall that 
$$\bigElement_{3}(x) \doteq 2 \left|\regFct(x)-\frac{1}{2}\right| \mathbbm{1}\left \{ \left| \regFct(x)-\frac{1}{2} \right| \leq 3 \holderCoeff \mathbb{E}_{\tilde{\featVar}_{1}} \left[ \dist( \tilde{\featVar}_{1} , x )^{\holderExp} \right] \right \} .$$
Let $x \in \tDom$, we have that $\mathbb{E}_{\tilde{\featVar}_{1}}\left[ \dist(\tilde{\featVar}_{1},x)^{\holderExp} \right]$ equals 
\begin{align*}
    & \int_{0}^{\diamDom^{\holderExp}} \mathbb{P}_{\tilde{\featVar}}\left( \dist(\tilde{\featVar}_{1},x)^{\holderExp} > t\right) \mathrm{d}t 
= \int_{0}^{\diamDom^{\holderExp}} \mathbb{P}_{\tilde{\featVar}}\left( \dist(\tilde{\featVar}_{1},x) > t^{1/\holderExp}\right) \mathrm{d}t \\
		& =  \int_{0}^{\diamDom^{\holderExp}} \left(1-\sProb(\cBall{x}{t^{1/\holderExp}}) \right)^{\left \lfloor \frac{\sN}{k} \right \rfloor} \left(1-\tProb(\cBall{x}{t^{1/\holderExp}}) \right)^{\left \lfloor \frac{\tN}{k} \right \rfloor} \mathrm{d}t.
\end{align*}

Now let's recall that $\tProb$ is doubling (see \dm~and Definition \ref{def:doublingMeas}), that is:
\begin{equation*}
\forall x \in \tDom, \forall r \in (0,\diamDom], \quad \tProb(\cBall{x}{r}) \geq \tDmCoeff \left( \frac{r}{\diamDom} \right)^{\tDmDim}.
\end{equation*}

Combining with (\ref{eq:ass1equation}) from Definition \ref{def:transferCoefficient}, we have $\forall x \in \tDom, \forall r \in (0,\diamDom]$
\begin{equation*}
 \sProb(\cBall{x}{r}) \geq  \tProb(\cBall{x}{r}) \transMarginCoeff \left( \frac{r}{\diamDom} \right)^\transMarginExp \geq  \tDmCoeff \transMarginCoeff \left( \frac{r}{\diamDom} \right)^{\tDmDim +\transMarginExp}.
\end{equation*}

Thus, note that for any $\transMarginExp$ (and in particular for  $\transMarginExp = \infty$), using only the fact that $\tProb$ is doubling, we can bound the expectation as:
\begin{equation} \label{eq:dmGammaInfinity}
\forall x\in \tDom, \quad \mathbb{E}_{\tilde{\featVar}_{1}} \left[ \dist(\tilde{\featVar}_{1},x)^{\holderExp} \right]   \leq \int_{0}^{\diamDom^{\holderExp}} \left(1-\tDmCoeff  \left( t\diamDom^{-\holderExp} \right)^{\tDmDim / \holderExp} \right)^{\left \lfloor \frac{\tN}{k} \right \rfloor} \diff t.
\end{equation}

For the moment assume $\transMarginExp < \infty$. Recall that $\tDmCoeff, \transMarginCoeff\leq 1$. We get for any $x \in \tDom$ that $\mathbb{E}_{\tilde{\featVar}_{1}}  \left[ \dist(\tilde{\featVar}_{1},x)^{\holderExp} \right]$ is at most 
\begin{align*}
& \int_{0}^{\diamDom^{\holderExp}} \left(1-\tDmCoeff \transMarginCoeff \left( t\diamDom^{-\holderExp} \right)^{(\tDmDim +\transMarginExp)/\holderExp} \right)^{\left \lfloor \frac{\sN}{k} \right \rfloor} \left(1-\tDmCoeff  \left( t\diamDom^{-\holderExp} \right)^{\tDmDim / \holderExp} \right)^{\left \lfloor \frac{\tN}{k} \right \rfloor} \mathrm{d}t \\
& \leq \int_{0}^{\diamDom^{\holderExp}} \exp \left( -\tDmCoeff \transMarginCoeff \left( t\diamDom^{-\holderExp} \right)^{(\tDmDim +\transMarginExp)/\holderExp} \left \lfloor \frac{\sN}{k} \right \rfloor -\tDmCoeff  \left( t\diamDom^{-\holderExp} \right)^{\tDmDim / \holderExp} \left \lfloor \frac{\tN}{k} \right \rfloor \right) \diff t \\
& \leq \int_{0}^{\diamDom^{\holderExp}} \exp \left( -\left(t (\tDmCoeff \transMarginCoeff )^{\holderExp/\tDmDim} \diamDom^{-\holderExp} \right)^{(\tDmDim +\transMarginExp)/\holderExp} \left \lfloor \frac{\sN}{k} \right \rfloor \right. \\
&\qquad \qquad \qquad \quad - \left.  \left( t (\tDmCoeff \transMarginCoeff)^{\holderExp / \tDmDim} \diamDom^{-\holderExp} \right)^{\tDmDim / \holderExp} \left \lfloor \frac{\tN}{k} \right \rfloor \right) \diff t \\
& \leq \frac{\diamDom^{\holderExp}}{(\tDmCoeff \transMarginCoeff)^{\holderExp / \tDmDim}} \int_{0}^{\infty} \exp \left( -  \left \lfloor \frac{\sN}{k} \right \rfloor s^{(\tDmDim +\transMarginExp)/\holderExp} - \left \lfloor \frac{\tN}{k} \right \rfloor s^{\tDmDim / \holderExp}  \right) \diff s,
\end{align*}
using the change of variable $s = t (\tDmCoeff \transMarginCoeff)^{\holderExp / \tDmDim} \diamDom^{-\holderExp}$. 

To bound this last integral we break it up over a suitable discretization of its range. Set $r_k = \left(\max(\left \lfloor \frac{\sN}{k} \right \rfloor^{\holderExp/(\tDmDim +\transMarginExp)}, \left \lfloor \frac{\tN}{k} \right \rfloor^{\holderExp / \tDmDim}) \right)^{-1}$. We make use of inequality (\ref{ineq:basicIneq3}) from Lemma \ref{lem:basicIneq}. Following the notations of the lemma, set $\alpha_{1} = \frac{\tDmDim + \transMarginExp}{\holderExp} = \frac{1}{\beta_{1}}$ and $\alpha_{2} = \frac{\tDmDim }{\holderExp} = \frac{1}{\beta_{2}}$. Notice that this implies $\alpha_{2} - \alpha_{1} = \frac{1}{\beta_{2}} - \frac{1}{\beta_{1}}$ and $\alpha_{1} \beta_{1} = 1 = \alpha_{2} \beta_{2}$. Finally let $a = \left \lfloor \frac{\sN}{\nn} \right \rfloor$ and $b = \left \lfloor \frac{\tN}{\nn} \right \rfloor$. Recall that $\nn \leq \sN \vee \tN$ and hence $a+b>0$. Applying inequality (\ref{ineq:basicIneq3}) from Lemma \ref{lem:basicIneq}:
\begin{equation} \label{eq:dmUseOfBasicIneq}
\left \lfloor \frac{\sN}{\nn} \right \rfloor r_k^{(\tDmDim + \transMarginExp)/\holderExp} + \left \lfloor \frac{\tN}{\nn} \right \rfloor r_k^{\tDmDim / \holderExp} = a r_k^{\alpha_{1}} + b r_k^{\alpha_{2}} \geq \frac{1}{2} (a^{1 - \alpha_{1} \beta_{1}} + b^{1 - \alpha_{2} \beta_{2}}) \geq \frac{1}{2}.
\end{equation}

{Now let $c_{i} = i\cdot r_k$ for integer $i \geq 0$. We then have:
\begin{align*}
\int_{0}^{\infty} &\exp \left( -  \left \lfloor \frac{\sN}{k} \right \rfloor t^{(\tDmDim +\transMarginExp)/\holderExp} - \left \lfloor \frac{\tN}{k} \right \rfloor t^{\tDmDim / \holderExp}  \right) \diff t \\
& = \sum_{i\geq 0} \int_{c_{i}}^{c_{i+1}} \exp \left ( - \left \lfloor \frac{\sN}{k} \right \rfloor t^{(\tDmDim +\transMarginExp)/\holderExp} - \left \lfloor \frac{\tN}{k} \right \rfloor t^{\tDmDim / \holderExp} \right ) \diff t \\
&\leq r_k \sum_{i \geq 0} \exp\left( - \left \lfloor \frac{\sN}{k} \right \rfloor c_{i}^{(\tDmDim +\transMarginExp)/\holderExp} - \left \lfloor \frac{\tN}{k} \right \rfloor c_{i}^{\tDmDim / \holderExp} \right)
 \leq r_k \sum_{i \geq 0} \exp \left(-\frac{i^{\tDmDim/\holderExp}}{2} \right),
\end{align*}
where we used equation (\ref{eq:dmUseOfBasicIneq}) in the last inequality. Notice that, as $\tDmDim > 0$, there exists a constant $\upConst > 0$ such that 
$\sum_{i \geq 0} \exp \left( - \frac{i^{\tDmDim / \holderExp}}{2} \right) \leq \upConst$.}

Now, for any $x \in \tDom$ we can bound the expectation as follows:
\begin{align}
\mathbb{E}_{\tilde{\featVar}_{1}} \left[ \dist(\tilde{\featVar}_{1},x)^{\holderExp} \right] & \leq \frac{ \upConst \diamDom^{\holderExp}}{( \tDmCoeff \transMarginCoeff)^{\holderExp / \tDmDim}} r_k \nonumber\\
& \doteq \frac{\upConst \diamDom^{\holderExp}}{(\tDmCoeff \transMarginCoeff)^{\holderExp / \tDmDim}} \left(\max \left(
\left \lfloor \frac{\sN}{k} \right \rfloor^{\tDmDim/(\tDmDim +\transMarginExp )}, \left \lfloor \frac{\tN}{k} \right \rfloor \right) \right)^{-\holderExp/\tDmDim} \label{eq:expecBias}\\
& \leq \frac{\upConst \diamDom^{\holderExp} 2^{\holderExp / \tDmDim}}{(\tDmCoeff \transMarginCoeff)^{\holderExp / \tDmDim}}  \left( \left \lfloor \frac{\sN}{\nn} \right \rfloor^{\frac{\tDmDim}{ \tDmDim + \transMarginExp}} + \left \lfloor \frac{\tN}{\nn}\right \rfloor \right)^{-\holderExp / \tDmDim} \doteq G_{\nn}^{'}(\sN,\tN,x),\nonumber
\end{align}
where we used result (\ref{ineq:basicIneq2}) from Lemma \ref{lem:basicIneq} in the last inequality. Therefore, by using  Definition \ref{def:noise}, we get:
\begin{align*}
\mathbb{E}_{\tJoinProb}[\bigElement_{3}(\featVar)] & \leq 2 \mathbb{E}_{\tJoinProb} \left[ \left|\regFct(x)-\frac{1}{2}\right| \mathbbm{1}\left \{ \left| \regFct(x)-\frac{1}{2} \right| \leq 3 \holderCoeff G_{\nn}^{'}(\sN,\tN,x) \right \} \right] \\
& \leq 2 \tTsyCoeff \left( \frac{ 3 \upConst \holderCoeff \diamDom^{\holderExp} 2^{\holderExp / \tDmDim}}{(\tDmCoeff \transMarginCoeff)^{\holderExp / \tDmDim}} \right)^{\tTsyExp + 1} \left( \left \lfloor \frac{\sN}{\nn} \right \rfloor^{\frac{\tDmDim}{ \tDmDim + \transMarginExp}} + \left \lfloor \frac{\tN}{\nn}\right \rfloor \right)^{-\holderExp(\tTsyExp + 1)/ \tDmDim}.
\end{align*}

Finally, the case $\transMarginExp = \infty$ is proved by starting from equation (\ref{eq:dmGammaInfinity}) and by following the same steps (and even simpler ones) as above.
\end{proof}


\subsection{Bounding $\mathbb{E}[\bigElement_{3}(\featVar)]$ under \bcn}

\begin{proof}[Proof of Lemma \ref{lem:kulkarniAdapt}]
By Fubini, we get that $\mathbb{E}_{\tJoinProb}[A(\epsilon, \featVar)]$ equals 
\begin{align*}
\int_{\epsilon}^{\diamDom^{\holderExp}} \int_{\tDom} \left(1 - \sProb(\cBall{x}{t^{1/\holderExp}})\right)^{\left \lfloor \frac{\sN}{\nn} \right \rfloor} \left(1 - \tProb(\cBall{x}{t^{1/\holderExp}})\right)^{\left \lfloor \frac{\tN}{\nn} \right \rfloor} \diff \tProb(x) \diff t.
\end{align*}

Let's now consider the inner integral. Take $t \in [\epsilon, \diamDom^{\holderExp}]$ and consider a cover of $\tDom$ with balls $(B_{i})_{i \in I}$ of {\color{black} diameter} $t^{1/\holderExp}$ indexed by some set $I$ of size $\mathcal{N}(\tDom,\dist, \frac{1}{2}t^{1/\holderExp})$. The inner integral is then at most 
\begin{align} \label{eq:lemKulkarniAssumption3}
& \sum_{i \in I} \int_{B_i} \left(1 - \sProb(\cBall{x}{t^{1/\holderExp}})\right)^{\left \lfloor \frac{\sN}{\nn} \right \rfloor} \left(1 - \tProb(\cBall{x}{t^{1/\holderExp}})\right)^{\left \lfloor \frac{\tN}{\nn} \right \rfloor} \diff \tProb(x) \nonumber \\
\leq \quad & \sum_{i \in I} \left(1 - \sProb(B_{i})\right)^{\left \lfloor \frac{\sN}{\nn} \right \rfloor} \left(1 - \tProb(B_{i})\right)^{\left \lfloor \frac{\tN}{\nn} \right \rfloor} \tProb(B_{i}) \nonumber \\
\leq \quad & \sum_{i \in I} \exp \left( - \sProb(B_{i}) \left \lfloor \frac{\sN}{\nn} \right \rfloor - \tProb(B_{i}) \left \lfloor \frac{\tN}{\nn} \right \rfloor\right) \tProb(B_{i}) \nonumber \\
\leq \quad & \sum_{i \in I} \exp \left( - \tProb(B_{i})  \transMarginCoeff \left( t^{1/\holderExp}/(2 \diamDom) \right)^{\transMarginExp} \left \lfloor \frac{\sN}{\nn} \right \rfloor - \tProb(B_{i}) \left \lfloor \frac{\tN}{\nn} \right \rfloor \right) \tProb(B_{i}) \\ 
\leq \quad & \color{black} \tCovCoeff 2^\tCovDim\diamDom^\tCovDim t^{- \tCovDim / \holderExp} \left(  \transMarginCoeff \left( t^{1/\holderExp}/(2 \diamDom) \right)^{\transMarginExp} \left \lfloor \frac{\sN}{\nn} \right \rfloor + \left \lfloor \frac{\tN}{\nn} \right \rfloor \right)^{-1}, \nonumber
\end{align} 
where we used equation \eqref{eq:ass1equation} from Definition \ref{def:transferCoefficient} in inequality (\ref{eq:lemKulkarniAssumption3}), {\color{black} assuming $\transMarginExp < \infty$}. 
Note also that when $\transMarginExp = \infty$ {\color{black} and $\left \lfloor \frac{\tN}{\nn} \right \rfloor > 0$} inequality (\ref{eq:lemKulkarniAssumption3}) can be replaced by:
\begin{equation} \label{eq:lemKulkarniInfinity}
\sum_{i \in I} \exp \left( - \left \lfloor \frac{\tN}{\nn} \right \rfloor \tProb(B_{i}) \right) \tProb(B_{i}) \leq {\color{black} \tCovCoeff 2^\tCovDim\diamDom^\tCovDim t^{- \tCovDim / \holderExp}  \left \lfloor \frac{\tN}{\nn} \right \rfloor ^{-1}}.
\end{equation}

However for now, let's consider $\transMarginExp < \infty$, and let $\tCovCoeff' \doteq \tCovCoeff 2^\tCovDim\diamDom^\tCovDim$. Assuming that $\left \lfloor \frac{\sN}{\nn} \right \rfloor, \left \lfloor \frac{\tN}{\nn} \right \rfloor > 0$ (recall also that $\transMarginCoeff \leq 1$), we get that $\mathbb{E}_{\tJoinProb}[A(\epsilon,\featVar)]$ is at most 
\begin{align*}
 & \leq \tCovCoeff' \int_{\epsilon}^{\diamDom^{\holderExp}} \left( \transMarginCoeff t^{(\tCovDim + \transMarginExp)/\holderExp} /(2 \diamDom)^{\transMarginExp} \left \lfloor \frac{\sN}{\nn} \right \rfloor + t^{ \tCovDim / \holderExp} \left \lfloor \frac{\tN}{\nn} \right \rfloor \right)^{-1} \diff t \\
& \leq \tCovCoeff' \int_{\epsilon}^{\diamDom^{\holderExp}} \min \left( t^{-(\tCovDim + \transMarginExp)/\holderExp} \left(\left \lfloor \frac{\sN}{\nn} \right \rfloor \transMarginCoeff /(2 \diamDom)^{\transMarginExp} \right)^{-1}, t^{- \tCovDim / \holderExp} \left \lfloor \frac{\tN}{\nn} \right \rfloor^{-1} \right) \diff t.
\end{align*}

Assume $\tCovDim > \holderExp$. Switching integral and min, the above is upper-bounded by 
\begin{align} \label{eq:lemKulkarniIntegral}
& \tCovCoeff' \min \left( \frac{\holderExp}{\transMarginExp + \tCovDim - \holderExp} \left(\left \lfloor \frac{\sN}{\nn} \right \rfloor \frac{\transMarginCoeff}{(2 \diamDom)^{\transMarginExp}} \right)^{-1} \epsilon^{-\frac{\tCovDim + \transMarginExp}{\holderExp}+1} \right., 
\left. \frac{\holderExp}{\tCovDim - \holderExp} \left \lfloor \frac{\tN}{\nn} \right \rfloor^{-1} 
\epsilon^{- \frac{\tCovDim}{\holderExp}+1} \right)   \\
& \leq 2 \frac{\tCovCoeff'}{\transMarginCoeff} \frac{\holderExp}{\tCovDim - \holderExp} ((2 \diamDom)^{\transMarginExp}\vee 1) \left( \left \lfloor \frac{\sN}{\nn} \right \rfloor  \epsilon^{(\tCovDim  + \transMarginExp) / \holderExp - 1} + \left \lfloor \frac{\tN}{\nn} \right \rfloor \epsilon^{ \tCovDim / \holderExp - 1} \right)^{-1}.\nonumber
\end{align}
When either $\left \lfloor \frac{\tN}{\nn} \right \rfloor= 0$ or $\left \lfloor \frac{\sN}{\nn} \right \rfloor= 0$, the same inequality (or even tighter) is obtained {\color{black} more directly}. 
When $\tCovDim = \holderExp$, the r.h.s. of (\ref{eq:lemKulkarniAssumption3}) can be bounded as:
\begin{align*}
\mathbb{E}_{\tProb}[A(\epsilon,\featVar)] & \leq \upConst \left(1 \vee (\log(\diamDom^{\holderExp}) - \log(\epsilon))\right) \min \left(\left \lfloor \frac{\sN}{\nn} \right \rfloor^{-1} \epsilon^{-\transMarginExp/\holderExp}, \left \lfloor \frac{\tN}{\nn} \right \rfloor^{-1} \right) \\
& \leq 2 \upConst \left(1 \vee (\log(\diamDom^{\holderExp}) - \log(\epsilon))\right) \left(\left \lfloor \frac{\sN}{\nn} \right \rfloor \epsilon^{\transMarginExp/\holderExp} + \left \lfloor \frac{\tN}{\nn} \right \rfloor\right)^{-1},
\end{align*}
where $\upConst = \frac{\tCovCoeff'}{\transMarginCoeff}  ((2 \diamDom)^{\transMarginExp} \vee 1) \left( 1 +\frac{\holderExp}{\transMarginExp} \mathbbm{1} \{\transMarginExp > 0 \} \right)$. Now, the case where either $\left \lfloor\frac{\tN}{\nn}\right \rfloor= 0$ or $\left \lfloor\frac{\sN}{\nn}\right \rfloor= 0$ is handled similarly. Finally for $\transMarginExp = \infty$ {\color{black} and $\left \lfloor \frac{\tN}{\nn} \right \rfloor > 0$}, by equation (\ref{eq:lemKulkarniInfinity}), and following the above intermediary steps, we get that $\mathbb{E}_{\tJoinProb_X}[A(\epsilon,\featVar)]$ is at most 
\begin{align*}
\tCovCoeff' \int_{\epsilon}^{\diamDom^{\holderExp}} t^{-\tCovDim /\holderExp} \left \lfloor \frac{\tN}{\nn} \right \rfloor^{-1} \diff t 
\leq \left \{
\begin{matrix}
\tCovCoeff' \frac{\holderExp}{\tCovDim - \holderExp} \left \lfloor \frac{\tN}{\nn} \right \rfloor^{-1} \epsilon^{ -\tCovDim/\holderExp + 1} , \quad & \text{ for } \holderExp < \tCovDim, \\
\tCovCoeff'  (\log(\diamDom^{\holderExp}) -\log(\epsilon) ) \left \lfloor \frac{\tN}{\nn} \right \rfloor^{-1}, \quad & \text{ for } \holderExp = \tCovDim.
\end{matrix}
\right.
\end{align*}

{\color{black} The bound for the case where $\transMarginExp = \infty$ and $\left \lfloor \frac{\tN}{\nn} \right \rfloor = 0$ is direct.}
\end{proof}

\section{Adaptive Rates}
\label{app:adaRates}

Algorithm \ref{alg:adaptiveKNN} works by considering the intersection of confidence sets (on the regression function $\regFct(x)$) for increasing values of $k$, and stops when confidence sets no longer intersect; this is an indication of having reached a good choice of $k$ which approximately balances regression bias and variance at a point $x$. While the basic Lepski's approach usually appears in the literature for adaptation to smoothness $\alpha$ -- as applied to \emph{kernel} regression type procedures, we will show that we also automatically get adaptation to $\gamma, \beta, d$ in our classification setting under transfer. 

Moreover, such adaptation extends beyond i.i.d. input $\featVectBase'$ (e.g. $\featVectBase' = \featVectBase$) to \emph{large} 
$k$-$2k$ covers $\featVectBase_{\coverSet}$ of $\featVectBase$. This is because such covers maintain useful properties of the original i.i.d. $\featVectBase$ as outlined in the Section \ref{sec:propofk2k} below. 

Therefore the proof of Theorem \ref{thm:genericadaptivity} can be given for a generic $\featVectBase' = \featVectBase_{\coverSet}$ as done in Theorem \ref{thm:genericadaptivity2} of this section. 

\subsection{Useful Properties of $k$-$2k$ covers $\featVectBase_{\coverSet}$ w.r.t. the Original Sample $\featVectBase$}
\label{sec:propofk2k} 
The first result below is an adaptation of Lemma 1 from \cite{berlind2015active} and gives a bound on the distance to nearest neighbors from a $\nn$-$2\nn$ cover. In particular, the result below holds simultaneously over any neighbor index $i \in [\nn]$, rather than only for $i = \nn$.

\begin{lemma}[Relating NN distances] \label{lem:berlindLemma}
Let $x \in \spa$, $1 \leq \nn \leq (\sN + \tN)/2$ and consider $\coverSet \subset [\sN + \tN]$ such that $\featVectBase_{\coverSet}$ is a $\nn$-$2\nn$ cover of $\featVectBase$. Let $\nnFeat{i}{\coverSet}$ denote the $i$-th NN of $x$ from $\featVectBase_{\coverSet}$, while as before, we let $\nnFeat{i}{}$ denote the same from $\featVectBase$: 
\begin{equation*}
\forall i \in [\nn], \quad \dist(\nnFeat{i}{\coverSet},x) \leq 3 \dist(\nnFeat{i + \nn}{},x).
\end{equation*}
\end{lemma}
\begin{proof}
We proceed by contradiction. Assume $\exists i \in [\nn]$ such that
\begin{equation*}
\dist(\nnFeat{i}{\coverSet}, x) > 3 \dist(\nnFeat{i+\nn}{},x).
\end{equation*}

It means that in the ball $\cBall{x}{3 \dist (\nnFeat{i+\nn}{},x)}$ there are strictly less than $i$ observations from $\featVectBase_{\coverSet}$ and in the ball $\cBall{x}{\dist(\nnFeat{i+\nn}{},x)}$ there are at least $i + \nn$ observations from $\featVectBase$. Therefore, there exists $x' \in \featVectBase \backslash \featVectBase_{\coverSet}$ such that $x' \in \cBall{x}{\dist(\nnFeat{i+\nn}{}, x)}$. We have for such $x'$:
\begin{equation*}
\cBall{x}{\dist(\nnFeat{i+\nn}{},x)} \subset \cBall{x'}{2 \dist(\nnFeat{i+\nn}{},x)} \subset \cBall{x}{3 \dist(\nnFeat{i + \nn}{},x)}.
\end{equation*}

Thus $\cBall{x'}{2 \dist(\nnFeat{i+\nn}{},x)}$ contains strictly less than $i$ elements from $\featVectBase_{\coverSet}$ but at least $\nn + i$ elements from $\featVectBase$, meaning that it contains at least $k+1$ elements from $\featVectBase \backslash \featVectBase_{\coverSet}$ while having less than $\nn$ elements from $\featVectBase_{\coverSet}$.

Therefore, among the $2 \nn$ nearest neighbors of $x'$ from $\featVectBase$, there are strictly less than $\nn$ elements from $\featVectBase_{\coverSet}$, this is in contradiction with the definition of a $\nn$-$2\nn$ cover (see Definition \ref{def:k2kCover}).
\end{proof}

The next result relates the size of a $k$-2$k$ cover with that of the dataset $\featVectBase$. 

\begin{myproposition}
Suppose $\featVectBase_{\coverSet}$, contains $\featVectBase_P$, and is a $\nn$-$2\nn$ cover of $\featVectBase$\, 
for some $\nn\geq (\sN \vee \tN)/4$. 

Let $\nR = |\coverSet|$, we have: $\sN + \tN \geq \nR \geq (\sN \vee \tN)/4 \geq (\sN + \tN) / 8.$
\end{myproposition}
\begin{proof}
Either $\featVectBase_Q\subset \featVectBase_R$, or some $x\in \featVectBase_Q$ has at least $k$ neighbors in $\featVectBase_R$. 
\end{proof}

\subsection{Obtaining Theorem \ref{thm:genericadaptivity}}
The proof of the theorem requires the following lemma -- presented  without proof -- due to \cite{SK:78}, which bounds in high probability, and uniformly over $x\in \cal X$, the error of a $\nn$-NN regression estimator. We state it in a generic way which applies beyond i.i.d. labeled data $\sampleBase'$. 

\begin{lemma}[Lemma 3 from \cite{SK:78}] \label{lem:regFctBiasVarBound} Assume that class $\ballColl$ of all the balls in $(\spa, \dist)$ has finite VC-dimension $\vcDim$. Consider a sample $\sampleBase'$ of size $n$, where the $\labVar_{i}$'s are conditionally independent (conditioned on $\featVectBase' \doteq \{X_i'\}_{i = 1}^n$), with the conditional distribution $\regFct(x) = \mathbb{P}(Y_i' = 1 | X_i' = x), \,\, \forall i$. Assume $\eta$ is $(\holderCoeff, \holderExp)$-H\"{o}lder as in Definition \ref{def:smoothness}. Define the $\nn$-NN regression estimate $\empReg(x) \doteq \frac{1}{\nn} \sum_{i=1}^{\nn} Y_i'$ where $Y_{(i)}'$ is the label of the $i$-th NN of $x$ in $\featVectBase'$. Then for any $0< \delta <1$, we have with probability at least $1-\delta$ that:
\begin{equation*}
\forall x \in \spa, \forall \nn \in [n], \quad \left| \empReg(x) - \regFct(x) \right| \leq \sqrt{\frac{\vcDim \log (2n/\delta) + 8}{\nn}}+ \frac{\holderCoeff}{\nn} \sum_{i=1}^{\nn} \dist^{\holderExp}(X_{(i)}', x).
\end{equation*}
\end{lemma}

%


Theorem \ref{thm:genericadaptivity} follows from Theorem \ref{thm:genericadaptivity2} below by letting $\featVectBase_{\coverSet} = \featVectBase$. 

\begin{theorem}[Generic analysis of Algorithm \ref{alg:adaptiveKNN}] \label{thm:genericadaptivity2}
Let Assumption \ref{ass:boundedVC} hold, and let $\family$ denote $\dmFam$ or $\bcnFam$. For $\family = \bcnFam$ assume further that $\alpha < d$. Suppose Algorithm \ref{alg:adaptiveKNN} takes as input $\sampleBase' \doteq \sampleBase_{\coverSet}$, where $\featVectBase_{\coverSet}$ is a $\nn$-$2\nn$ cover of $\featVectBase$ for all $\nn \in \mathcal{K} \doteq \{2^{i} \nnZero: i \in \{ 0, \ldots, \lfloor \log_{2}((\sN \vee \tN)/2\nnZero) \rfloor \} \}$, 
where we let $\nnZero \doteq \Theta (\vcDim \log(\sN + \tN) )$\footnote{The asymptotic $\Theta$ is in $\sN \vee \tN$.}.
Suppose Algorithm \ref{alg:adaptiveKNN} outputs $\hHat_{\coverSet}$. We have: 
\begin{equation*}
\sup_{(\sJoinProb, \tJoinProb) \in \family} \mathbb{E}_{\sample}[\exErr(\hHat_{\coverSet})] \leq \upConst \left( \frac{\nnZero\cdot \log (2(\sN + \tN))}{\sN ^{\rates / (\rates + \transMarginExp / \holderExp)} + \tN }\right)^{ (\tTsyExp + 1) / \rates}, 
\end{equation*}
where $\rates = 2  + \tCovDim / \holderExp$ when $\family = \dmFam$, and $\rates = 2 + \tTsyExp + \tCovDim / \holderExp$ when $\family = \bcnFam$.

When $\family = \bcnFam$ with $\alpha = d$, replace $\upConst$ above with 
$\upConst\cdot \log(2(n_P + n_Q))$. 
\end{theorem}

\begin{proof}
Let $\nR \doteq |\coverSet|$, and for fixed $x$, let $X_{(i)}^R$ denote the $i$th NN of $x$ in $\featVect_R$. 
Let $\delta' \in (0,1)$, and let $A_{\delta'}$ denote the event that, $\forall x \in \spa, \forall \nn \in \mathcal{K}$
\begin{equation*}
  \left| \empReg_{\nn, \coverSet}(x) - \regFct(x) \right| \leq \sqrt{\frac{\vcDim \log(2\nR/\delta') + 8}{\nn}}+ \frac{\holderCoeff}{\nn} \sum_{i=1}^{\nn} \dist^{\holderExp}(\nnFeat{i}{\coverSet}, x).
\end{equation*}
{Note that, by definition of a $k$-$2k$ cover, the indices in $\coverSet$ depend only on the marginals $\featVect$, implying that $\sampleBase_{\coverSet}$ must satisfy the conditions of Lemma \ref{lem:regFctBiasVarBound}. It follows that the event $A_{\delta'}$ has probability at least $1-\delta'$.}

Pick $\delta' = (\sN + \tN)^{-(\tTsyExp +1)/\rates}$. We can see that $\exists N_{1} = N_{1}(\family)$ such that whenever $\sN+\tN \geq N_{1}$ we have:
\begin{equation*}
\sqrt{\vcDim \log(2 \nR / \delta') + 8} \leq \frac{1}{2}\sqrt{\vcDim} \log \nR.
\end{equation*}


Hence, assume the event $A_{\delta'}$ holds, and $\sN + \tN$ sufficiently large. We then have that, $\forall x \in \spa, \forall \nn \in \mathcal{K}$
\begin{equation*}
 \left| \empReg_{\nn, \coverSet}(x) - \regFct(x) \right| \leq \max \left(\log (\nR) \sqrt{\frac{\vcDim}{\nn}}, \frac{2\holderCoeff}{\nn} \sum_{i=1}^{\nn} \dist^{\holderExp}(\nnFeat{i}{\coverSet}, x) \right).
\end{equation*}

Notice that $\nn^{-1/2}$ is decreasing in $k$, while $\nn^{-1} \sum_{i=1}^{\nn} \dist^{\holderExp}(\nnFeat{i}{\coverSet}, x)$ is non-decreasing. Therefore, it makes sense to define:
\begin{equation*}
\nn^{*} \doteq \max \left \{ \nn \in \mathcal{K} \cap [\nR]: \log(\nR)\sqrt{\frac{\vcDim}{\nn}} \geq \frac{2\holderCoeff}{\nn} \sum_{i=1}^{\nn} \dist^{\holderExp}(\nnFeat{i}{\coverSet}, x) \right \}.
\end{equation*}
Assume $\sN+\tN \geq N_{3}$ for some $N_{3} = N_{3}(\family)$, so that $\nn^{*}\geq \nn_0$ is well defined simultaneously for all $x$, that is, the set being maximized over is non-empty, i.e., contains at least $\nn_0$ since $\dist(\nnFeat{i}{\coverSet},x) \leq \diamDom$ for all $i$ and $x$. Now, $\forall \nn \leq \nn^{*}$, we have: 
\begin{equation}
\regFct(x) \in [\lowerBd_{\nn}, \upperBd_{\nn}] = \bigcap_{k' \in {\cal K}, \, k' \leq k} \left[ \empReg_{\nn'}(x) - \log(\nR)\sqrt{\frac{\vcDim}{\nn}}, \empReg_{\nn'}(x) + \log(\nR)\sqrt{\frac{\vcDim}{\nn}} \right], \label{eq:smallK}
\end{equation}
in other words, we have that $\lowerBd_{\nn} \leq \upperBd_{\nn}$. 

Let $\nn_{\text{stop}}$ denote the choice of $\nn$ at which Algorithm \ref{alg:adaptiveKNN} stops, that is:
\begin{equation*}
\nn_{\text{stop}} \doteq \min \left \{ \nn \in \mathcal{K}: \quad \nn > \frac{1}{2}\nR \text{  or  } \lowerBd_{2\nn} > \upperBd_{2\nn} \text{  or  } \upperBd_{\nn} < \frac{1}{2} \text{  or  } \lowerBd_{\nn} > \frac{1}{2} \right \}.
\end{equation*}

We consider two cases as to how $\nn_{\text{stop}}$ relates to $\nn^*$:

\underline{Case 1:} $\nn_{\text{stop}} < \nn^{*}$

We have $\nn_{\text{stop}} < \nn^* \leq \nR / 2$.  Also, as per (\ref{eq:smallK}), $\regFct(x) \in [\lowerBd_{2\nn_{\text{stop}}}, \upperBd_{2\nn_{\text{stop}}}]$, that is $\lowerBd_{2\nn_{\text{stop}}} \leq \upperBd_{2\nn_{\text{stop}}}$. Therefore, either $\upperBd_{\nn_{\text{stop}}} < 1/2$ or $\lowerBd_{\nn_{\text{stop}}} > 1/2$. Since we also have $\regFct(x) \in [\lowerBd_{\nn_{\text{stop}}}, \upperBd_{\nn_{\text{stop}}}]$, this implies:
\begin{equation*}
\hStar(x) \doteq \mathbbm \{ \regFct(x) \geq 1/2 \} = \mathbbm \{ \est_{\nn_{\text{stop}}}(x) \geq 1/2 \} = \hHat_{\coverSet}(x),
\end{equation*}
that is, the excess error at $x$ is equal to $0$. 


\underline{Case 2:} $\nn_{\text{stop}} \geq \nn^{*}$
We now have that $[\lowerBd_{\nn_{\text{stop}}}, \upperBd_{\nn_{\text{stop}}}]\subset [\lowerBd_{\nn^{*}}, \upperBd_{\nn^{*}}]$,  while $\regFct(x) \in [\lowerBd_{\nn^{*}}, \upperBd_{\nn^{*}}]$, and $\est_{\nn_{\text{stop}}} \in [\lowerBd_{\nn_{\text{stop}}}, \upperBd_{\nn_{\text{stop}}}]$. In other words, $\regFct(x)$ and $\est_{\nn_{\text{stop}}}$ cannot be far from each other:
\begin{equation}
|\est_{\nn_{\text{stop}}} - \regFct(x)| \leq \upperBd_{\nn^{*}} - \lowerBd_{\nn^{*}}  \leq 
2 \log(\nR)\sqrt{\frac{\vcDim}{\nn^*}}. \label{eq:case2kstop}
\end{equation}

We now proceed to bounding the above right hand side . 
Since $\nn \mapsto \nn^{-1/2}$ is decreasing and $\nn \mapsto \nn^{-1} \sum_{i = 1}^{\nn} \dist^{\holderExp}(\nnFeat{i}{\coverSet})$ non-decreasing and the following minimum $\mu$ is reached at either $\nn^{*}$ or $2\nn^{*}$:
\begin{equation*}
\mu \doteq \min_{\nn \in \mathcal{K}} \left(\max \left( \log(\nR)\sqrt{\frac{\vcDim}{\nn}}, \frac{2 \holderCoeff}{\nn} \sum_{i = 1}^{\nn} \dist^{\holderExp}(\nnFeat{i}{\coverSet},x) \right) \right).
\end{equation*}
We now argue that this implies:
\begin{equation} \label{eq:boundByTheMin}
\log(\nR)\sqrt{\frac{\vcDim}{\nn^{*}}} \leq \sqrt{2} \mu.
\end{equation} 

Indeed the inequality is direct when the minimum $\mu$ is reached at $\nn^{*}$. If instead the minimum is reached at $2 \nn^{*}$, we have that:
\begin{equation*}
\log(\nR)\sqrt{\frac{\vcDim}{2\nn^{*}}} \leq \frac{2 \holderCoeff}{2\nn^{*}} \sum_{i = 1}^{2\nn^{*}} \dist^{\holderExp}(\nnFeat{i}{\coverSet},x) = \mu.
\end{equation*}

It follows that, by \eqref{eq:case2kstop} we have 
\begin{equation*}
|\est_{\nn_{\text{stop}}} - \regFct(x)| \leq 2\sqrt{2} \mu.
\end{equation*}

Now, consider the following choice of $\nn$ (yielding the optimal rates of Theorem \ref{thm:expErrRates}, up to log terms):
\begin{equation*}
\nn(\sN, \tN) \doteq \left \lceil {\nnZero}\log(\sN + \tN)(\sN^{\rates/(\rates + \transMarginExp / \holderExp)} + \tN)^{2/\rates} \right \rceil.
\end{equation*}
Since $\rates > 2$, $\exists N_{2} = N_{2}(\family)$, such that for $\sN + \tN \geq N_{2}$, we have:
\begin{equation*}
\nnZero \leq \nn(\sN, \tN) \leq \frac{\sN + \tN}{8} \leq \nR.
\end{equation*}

Therefore there exists $\nn \in \mathcal{K}$ such that $\nn \leq \nn(\sN, \tN) \leq 2\nn$. Hence we have:
\begin{align*}
\mu &\leq \log(\nR)\sqrt{\frac{\vcDim}{\nn}} + \frac{2 \holderCoeff}{\nn} \sum_{i = 1}^{\nn} \dist^{\holderExp}(\nnFeat{i}{\coverSet},x) \\
&\leq \log(\nR)\sqrt{\frac{2\vcDim}{\nn(\sN,\tN)}} + \frac{2 \holderCoeff}{\nn(\sN, \tN)} \sum_{i = 1}^{\nn(\sN,\tN)} \dist^{\holderExp}(\nnFeat{i}{\coverSet},x).
\end{align*}

Thus the excess error at $x$ is bounded as follows:
\begin{align*}
& 2 \left| \regFct(x) - \frac{1}{2} \right| \mathbbm{1}\{ \hStar(x)\neq \hHat_{\coverSet}(x)\} 
 \leq 2 \left| \regFct(x) - \frac{1}{2} \right| \mathbbm{1} \left\{ \left| \regFct(x) - \frac{1}{2} \right| \leq 2\sqrt{2} \mu \right\} \\
& \leq 2 \left| \regFct(x) - \frac{1}{2} \right| \mathbbm{1} \left\{ \left| \regFct(x) - \frac{1}{2} \right| \leq 2\sqrt{2}  \log(\nR)\sqrt{\frac{2\vcDim}{\nn(\sN,\tN)}} \right. \\
& \qquad \qquad \qquad \qquad + \left. \frac{4\sqrt{2} \, \holderCoeff}{\nn(\sN, \tN)} \sum_{i = 1}^{\nn(\sN,\tN)} \dist^{\holderExp}(\nnFeat{i}{\coverSet},x) \right\} .
\end{align*}

We recognize a bias-variance bound that is similar to the one we obtained in our proofs of Theorem \ref{thm:expErrRates} and Proposition \ref{prop:biasVarianceDecomp}. Except for log terms and some additional constant factors, the only difference is that the bias term depends on a cover $\featVect_{\coverSet}$ instead of the full sample $\featVect$. Lemma \ref{lem:berlindLemma} 
implies that in fact the bias term is of similar order as that from Theorem \ref{thm:expErrRates}. Indeed we have:
\begin{align*}
\frac{1}{\nn(\sN, \tN)} \sum_{i = 1}^{\nn(\sN, \tN)} \dist^{\holderExp}(\nnFeat{i}{\coverSet}, x) 
\leq \frac{3}{2\nn(\sN,\tN)} \sum_{i = 1}^{4\nn(\sN,\tN)} \dist^{\holderExp}(\nnFeat{i}{}, x).
\end{align*}

The rest of the proof therefore consists of the same arguments as for the upper bound proof (see Theorem \ref{thm:expErrRates} and Proposition \ref{prop:biasVarianceDecomp} and the proofs in Appendix \ref{app:upperBound}), as $\nn(\sN, \tN)$ is indeed the optimal oracle choice of Theorem \ref{thm:expErrRates}. Finally, all the arguments above hold under the event $A_{\delta'}$ whose  complement has probability at most $(\sN + \tN)^{-(\tTsyExp + 1)/\rates}$, i.e., is of the right rate order. We can then conclude with the statement of the theorem.
\end{proof}

\section{Adaptive Labeling Results} 
\label{app:adaResults}



\subsection{Obtaining Theorem \ref{thm:labelingComplexity}}
The first part of Theorem \ref{thm:labelingComplexity} follows from Theorem \ref{thm:genericadaptivity2}. 
For the second part, we use Lemma 1 from \cite{SK:78}. It is the direct consequence of some known result in VC-theory (see \cite{VC:72}). We restate it below without proof.

\begin{lemma}[Lemma 1 from \cite{SK:78}] \label{lem:skLem1} Let $\ballColl$ denote the class of all the balls in $(\spa, \dist)$, and let $D$ be a distribution over $\spa$. Let $\hat{D}$ be the empirical distribution of $D$ from $n$ i.i.d.~realizations of $D$. For $\delta \in (0,1)$, define $\alpha_{n} = (\vcDim \log(2n) + \log(6/\delta))/n$. With probability at least $1-\delta$ over the $n$ i.i.d.~samples drawn from $D$, we have simultaneously $\forall B \in \ballColl, \forall a \geq \alpha_{n} $:
\begin{equation*}
\hat{D}(B) \geq 3a \, \implies \, D(B) \geq a,  \quad D(B) \geq 3a \, \implies \, \hat{D}(B) \geq a.
\end{equation*}
\end{lemma}

We now turn to the proof of Theorem \ref{thm:labelingComplexity}. The proof is based on a similar intuition as used in Theorem 2 of \cite{berlind2015active}: namely that there is no label request at a target sample $X_i\sim Q_X$ if 
the distances to its nearest neighbor in $\mathbf{X}_R$ is of similar order as the distance to its nearest neighbor in 
$\mathbf{X}_P$. However, their theorem is only shown for a fixed $\nn$, while we need this result to hold simultaneously for several values of $\nn$ for our iterative construction of the cover in Algorithm \ref{alg:queryLabels}. We therefore present a new analysis below that simultaneously considers multiple values of $k$, and also manages to remove some extraneous log-terms present in their earlier result. 
 

\begin{proof}[Proof of Theorem \ref{thm:labelingComplexity}]
The first part of Theorem \ref{thm:labelingComplexity} follows from Theorem \ref{thm:genericadaptivity2}. The label complexity result is obtained as follows. 

Fix a point $x \in \tFeatVect$, and let $\nnFeat{i}{\sJoinProb}$ and $\nnFeat{i}{\tJoinProb}$ denote the $i$-th NN of $x$ from $\sFeatVect$ and $\tFeatVect$.
First notice that $\dist(x, \nnFeat{2\nn}{}) \geq \min(\dist(x, \nnFeat{\nn}{\sJoinProb}), \dist(x, \nnFeat{\nn}{\tJoinProb}))$. Hence, Algorithm \ref{alg:queryLabels} won't query the label at $x$ if $\dist(x, \nnFeat{\nn}{\sJoinProb}) \leq \dist(x, \nnFeat{\nn}{\tJoinProb})$, as this implies $\dist(x, \nnFeat{\nn}{\sJoinProb}) \leq \dist(x, \nnFeat{2\nn}{})$. 

Define $\hat{\sJoinProb}_{\featVar}$ and $\hat{\tJoinProb}_{\featVar}$ as the empirical distributions on the samples $\sFeatVect$ and $\tFeatVect$. 
From Lemma \ref{lem:skLem1}, we have with probability at least $1 - \delta$, that 
$\forall x \in \tDom, \forall \nn \in \mathcal{K}$,
\begin{align}
\nn_0 \leq \nn &\leq \tN \hat{\tJoinProb}_{\featVar}(\cBall{x}{\dist(x, \nnFeat{\nn}{\tJoinProb})}) \leq 3 \tN  \tProb(\cBall{x}{\dist(x, \nnFeat{\nn}{\tJoinProb})}) \label{eq:firstkBoundQ}\\
&\leq 3 \frac{\tN}{\sN} \transMarginCoeff \left( \frac{\dist(x, \nnFeat{\nn}{\tJoinProb})}{\diamDom} \right)^{-\transMarginExp} \cdot \sN \sProb(\cBall{x}{\dist(x, \nnFeat{\nn}{\tJoinProb})}). \label{eq:firstkBoundP}
\end{align}

Now, notice that, following from \eqref{eq:firstkBoundQ}, $\dist(x, \nnFeat{\nn}{\tJoinProb})  \geq r_Q(x; \frac{k_0}{n_Q})$. Therefore, if 
\begin{align}
9\frac{\tN}{\sN} \transMarginCoeff \left(\frac{r_Q(x; \frac{k_0}{n_Q})}{\diamDom}\right)^{-\transMarginExp} \leq 1, \label{eq:noSamplingCondition}
\end{align}
it follows from \eqref{eq:firstkBoundP} and another application of Lemma \ref{lem:skLem1} that, with probability at least $1-2\delta$, for any $x \in {\cal X}_Q$ satisfying \eqref{eq:noSamplingCondition} above, we have: 
\begin{align*} 
\nn &\leq 9 \frac{\tN}{\sN} \transMarginCoeff \left( \frac{\dist(x, \nnFeat{\nn}{\tJoinProb})}{\diamDom} \right)^{-\transMarginExp} \cdot \sN \hat{\sJoinProb}_{\featVar}(\cBall{x}{\dist(x, \nnFeat{\nn}{\tJoinProb})}) \\
&\leq \sN \hat{\sJoinProb}_{\featVar}(\cBall{x}{\dist(x, \nnFeat{\nn}{\tJoinProb})}).
\end{align*} 

This last inequality implies that for any such $x$, we have $\dist(x, \nnFeat{\nn}{\sJoinProb}) \leq \dist(x, \nnFeat{\nn}{\tJoinProb})$. 

Now, for the last part of the theorem statement, notice that, if as per assumption, 
$\tProb(\cBall{x}{r)} \leq \tDmCoeff' (r/ \diamDom)^{\tDmDim}$ for all $x\in {\cal X}_Q$ and $r>0$, then 
$\tProb(\cBall{x}{r)} \geq \alpha$ implies $r\geq \tDmCoeff'^{-1/d}\cdot \diamDom \cdot \alpha^{1/d}$. It follows that 
$$\forall x \in {\cal X}_Q, \quad r_Q\left (x; \frac{k_0}{n_Q}\right ) \geq \tDmCoeff'^{-1/d}\cdot \diamDom \cdot n_Q^{-1/d}.$$ Plugging into the condition of \eqref{eq:noSamplingCondition} yields the final result. 
\end{proof}

\section{Extensions}
 \label{app:extensions}

We give in this section a few extensions of our results to more general settings.

\subsection{Exponential Rates for $\dm$ with $\beta = \infty$ } 
Super fast exponential rates are known in passive learning (for local-polynomial classifiers) for the case 
$\beta = \infty$ under $\dm$ \cite{audibert2007fast}. Here we show that this is also the case in Transfer with $k$-NN. Our results also imply exponential rates for vanilla $k$-NN simply by setting $n_P$ to $0$. 

\begin{theorem} [$\beta = \infty$]\label{theo:betainfinity}
Let $\hPQ$ as given in Definition \ref{def:kNNClass}. 
Let $(\sJoinProb, \tJoinProb) \in \dmFam$ with $\gamma < \infty$, and $\beta = \infty$, i.e., $\exists \, 0< \tau \leq 1/2$, $Q_X\left (0 <   |\eta(X) - 1/2| \leq \tau \right) = 0$. Then there exists $C_1, C_2$, depending only on $\dmFam$, such that

$$ \mathbb{E}_{\sample}[\exErr(\hPQ)] \leq C_1 \exp\left( -C_2  \left( n_P\cdot \tau^{\frac{2\alpha + d+ \gamma}{\alpha}} \vee n_Q\cdot \tau^{\frac{2\alpha + d}{\alpha}}\right) \right),$$ 
for a setting of $k = \Theta \left( n_P\cdot \tau^{\frac{d+ \gamma}{\alpha}} \vee n_Q\cdot \tau^{\frac{d}{\alpha}}\right) $.
\end{theorem}

The proof relies on Lemma 3.6. of \cite{audibert2007fast} which states that, for any plug-in classifier 
$\hat h = \mathbbm{1}\{\hat \eta \geq 1/2\}$ relying on a sample $\sample$ we have 
$$\mathbb{E}_{\sample} \, \exErr(\hat h) \leq \mathbb{P}_{\sample, X}\left( |\hat \eta (X) - \eta (X)| > \tau \right).$$
A simple inspection of the proof shows that the result applies even in our case where $\sample$ is not i.i.d.. We therefore proceed by bounding the probability on the right. 

To this end, recall the decomposition in \eqref{eq:decompStep2} that 
$|\empReg_k(x)  - \regFct(x) |$ is at most 
\begin{align*}
 \underbrace{\frac{1}{\nn}\left| \sum_{i = 1}^{\nn} \nnLab{i}{} - \regFct(\nnFeat{i}{}) \right|}_{G_1(x)}
+ \underbrace{\frac{\holderCoeff}{\nn} \sum_{i = 1}^{\nn} \left(\dist( \tilde{X}_i , x )^{\holderExp} - \Expectation_{\tilde{X}_1} \dist( \tilde{X}_1 , x )^{\holderExp}\right)}_{G_2(x)} + 
\underbrace{\holderCoeff\Expectation_{\tilde{X}_1} \dist( \tilde{X}_1 , x )^{\holderExp}}_{G_3(x)}. 
\end{align*} 

Our previous results bounding those three terms immediately yield the following lemma of independent interest (for vanilla $k$-NN, set $n_P = 0$). 

\begin{lemma}[Concentration for $k$-NN regression estimate]\label{lem:kNNconcentration}
Fix any $x \in {\cal X}_Q$, and define $r_k = \left(\left ( \frac{\sN}{k} \right )^{\holderExp/(\tDmDim +\transMarginExp)}\wedge \left ( \frac{\tN}{k}  \right )^{\holderExp / \tDmDim} \right)^{-1}$. Pick any $0< \epsilon < 1$, and suppose $r_k \leq \epsilon/3c_1$ for some $c_1 = c_1(\dmFam)$. 

$$ \mathbb{P}_{\sample}\left( |\hat \eta_k (x) - \eta (x)| > \epsilon \right) 
\leq 4 \exp\left (- C k \epsilon^2  \right),$$
for some $C = C(\dmFam)$. 
\end{lemma}
\begin{proof} 
First, by \eqref{eq:expecBias}, $G_3(x) \leq c_1 \cdot r_k \leq \epsilon/3$. Next, by \eqref{eq:varhighprob} and \eqref{eq:biashighprob}, 
$$\mathbb{P}\left ( G_1(x) + G_2(x) > 2\epsilon/3 \right) \leq 4 \exp (- C k \epsilon^2).$$
The result then follows by \eqref{eq:decompStep2}. 
\end{proof}

\begin{proof}[Proof of Theorem \ref{theo:betainfinity}]
The result follows from Lemma \ref{lem:kNNconcentration} above (combined with Lemma 3.6 of \cite{audibert2007fast}), first by noticing that the concentration bound is independent of $x$, and that the 
setting of $k =\Theta \left( n_P\cdot \tau^{\frac{d+ \gamma}{\alpha}} \vee n_Q\cdot \tau^{\frac{d}{\alpha}}\right) $, ensures that $r_k \leq \tau/3c_1$. 
\end{proof}

\subsection{Localizing the Transfer Exponent}
The proposition below gives what we should expect as rates of convergence in the situation where the support of $\sJoinProb$ doesn't include the one of $\tJoinProb$, but are in some sense close to each other, allowing some amount of transfer.

\begin{myproposition} [Generalized transfer exponent]
Let $\epsilon \in (0,3/4]$. Assume that the region $\tDom^{\transMarginExp}$ from Definition \ref{def:transferCoefficient} is such that $\tProb(\tDom^{\transMarginExp}) \geq 1 - \epsilon$, instead of $\tProb(\tDom^{\transMarginExp}) = 1$. Then the optimal minimax rates are reached by a mixture of $\nn$-NN classifiers $\hHat(x) \doteq \mathbbm{1}\{x \in \tDom^{\transMarginExp} \} \hHat_{\nn_{1}}(x) + \mathbbm{1}\{x \notin \tDom^{\transMarginExp} \} \hHat_{\nn_{2}}(x)$, where $\nn_{1} = \Theta ( \sN ^{\rates / (\rates + \transMarginExp/ \holderExp)} + \tN )^{2 /\rates}$ and $\nn_{2} = \Theta ( 1 + \tN )^{2 /\rates}$, where $\rates$ is defined below. These rates are as follows. Let $\family$  denote either $\dmFam$ or $\bcnFam$ and for $\family = \bcnFam$ assume further that $\alpha < d$. There exist constants $\upConst_{1}, \upConst_{2}$ depending only on $\family$, such that:
\begin{align*}
\sup_{(\sJoinProb, \tJoinProb) \in \family} \mathbb{E}_{\sample}[\exErr(\hHat)] &\leq \upConst_{1} \left( \sN ^{\rates / (\rates + \transMarginExp / \holderExp)} + \tN \right)^{ -(\tTsyExp + 1) / \rates} \\
&\quad + \epsilon \wedge \left( \upConst_{2} (1+\tN)^{ -(\tTsyExp + 1) / \rates}  \right), 
\end{align*}
where $\rates = 2  + \tCovDim / \holderExp$ when $\family = \dmFam$, and $\rates = 2 + \tTsyExp + \tCovDim / \holderExp$ when $\family = \bcnFam$.

For the case $\family = \bcnFam$ with $\alpha = d$, $\upConst_{1}$ is replaced with 
$\upConst_{1}\cdot \log(2(n_P + n_Q))$, and $\upConst_{2}$ with $\upConst_{2}\cdot \log(2(1 + n_Q))$
\end{myproposition}
\begin{proof}[Proof Outline]
Start by dividing $\mathbb{E}_{\sample}[\exErr(\hHat)]$ into two parts:
$$\mathbb{E}_{\sample, \featVar}[\exErr(\hHat_{\nn_{1}})(\featVar) \mathbbm{1}\{ \featVar \in \tDom^{\transMarginExp} \}] + \mathbb{E}_{\sample, \featVar}[\exErr(\hHat_{\nn_{2}})(\featVar) \mathbbm{1}\{ \featVar \notin \tDom^{\transMarginExp} \}],$$ where $\featVar \sim \tProb$ is independent of the data $\sample$ and $\exErr(\hHat)(x) \doteq \left|\regFct(x)-\frac{1}{2}\right| \cdot\mathbbm{1}\{\hHat(x) \neq \hStar(x)\}$ is the excess error at point $x$. The L.H.S.~rates are obtained by bounding $\mathbb{E}_{\sample, \featVar}[\exErr(\hHat_{\nn_{1}})(\featVar) \mathbbm{1}\{ \featVar \in \tDom^{\transMarginExp} \}]$ following similar lines as in our proof of the upper bounds in Theorem \ref{thm:expErrRates} (see Proposition \ref{prop:biasVarianceDecomp} and subsequent lemmas and proofs in Appendix \ref{app:upperBound}). Indeed, we can redo all theses proofs by restricting the integral to the set $\tDom^{\transMarginExp}$ as we just need in this case that the condition of Definition \ref{def:transferCoefficient} to be satisfied only on this subset for some $\transMarginExp$. Finally, the R.H.S.~rates are simply obtained because $\mathbb{E}_{\sample, \featVar}[\exErr(\hHat_{\nn_{2}})(\featVar) \mathbbm{1}\{ \featVar \notin \tDom^{\transMarginExp} \}]$ is simultaneously bounded by $\epsilon$ and by the rate of convergence in the worst case scenario of $\transMarginExp = \infty$.
\end{proof}

\subsection{Relaxing Covariate-Shift}
The last proposition treats the case where the covariate-shift assumption is not verified, that is, there are two different regression functions $\regFct_{P}$ and $\regFct_{Q}$ though close to each other.

\begin{myproposition} [Minimax rates without covariate-shift] 
Assume that $\sJoinProb$ and $\tJoinProb$ have respective regression functions $\regFct_{P}$ and $\regFct_{Q}$ such that $\| \regFct_{\sJoinProb} - \regFct_{\tJoinProb}\|_{\infty} \leq \epsilon$, for $\epsilon \in [0,1]$. Let $\family$ denote either $\dmFam$ or $\bcnFam$, where here we added the previous assumption on $\regFct_{P}$ and $\regFct_{Q}$ to the definitions of these classes of distribution tuples. For $\family = \bcnFam$ assume further that $\alpha < d$. There exists a constant $\upConst = \upConst(\family)$ such that, for a $k$-NN classifier $\hPQ$ we have 
\begin{equation*}
\sup_{(\sJoinProb, \tJoinProb) \in \family} \mathbb{E}_{\sample}[\exErr(\hPQ)] \leq \upConst \left( \sN ^{\rates / (\rates + \transMarginExp / \holderExp)} + \tN \right)^{ -(\tTsyExp + 1) / \rates} + 2 \tTsyCoeff (2 \epsilon)^{\tTsyExp+1}, 
\end{equation*}
for a choice of $k = \Theta \left( \sN ^{\rates / (\rates + \transMarginExp/ \holderExp)} + \tN \right)^{2 /\rates}$, where $\rates = 2  + \tCovDim / \holderExp$ when $\family = \dmFam$, and $\rates = 2 + \tTsyExp + \tCovDim / \holderExp$ when $\family = \bcnFam$.

For $\family = \bcnFam$ with $\alpha = d$, $\upConst$ above is replaced with 
$\upConst\cdot \log(2(n_P + n_Q))$. 
\end{myproposition}
\begin{proof}
Note that in this setting we have:
\begin{equation*}
\exErr(\hPQ) \leq 2 \mathbb{E}_{\tJoinProb} \left|\regFct_{Q}(\featVar)-\frac{1}{2}\right|  \cdot \mathbbm{1} \left \{\left| \regFct_{Q}(\featVar)-\frac{1}{2} \right| \leq \left| \empReg_k(\featVar) - \regFct_{Q}(\featVar) \right| \right \}.
\end{equation*}
We then have the following bound on the regression error:
\begin{align*} 
\left| \empReg_k(x) - \regFct_Q(x) \right| 
 &\leq \frac{1}{\nn}\left| \sum_{i = 1}^{\nn} \nnLab{i}{} - \mathbb{E}[\nnLab{i}{} | \nnFeat{i}{}]  \right| + \frac{1}{\nn}\left| \regFct_{Q}(\nnFeat{i}{}) - \mathbb{E}[\nnLab{i}{} | \nnFeat{i}{}]  \right| \\
 & \quad + \frac{\holderCoeff}{\nn} \sum_{i = 1}^{\nn} \dist( \nnFeat{i}{} , x )^{\holderExp}. 
\end{align*}

Note that by assumption the middle term is bounded as follows:
\begin{align*} 
 \frac{1}{\nn}\left| \regFct_{Q}(\nnFeat{i}{}) - \mathbb{E}[\nnLab{i}{} | \nnFeat{i}{}]  \right| \leq \epsilon.
\end{align*}

Hence, by using $\mathbbm{1}\{x \leq a + b\} \leq \mathbbm{1}\{x \leq 2a\} +  \mathbbm{1}\{x \leq 2b\}$, and using low noise assumption, we get:
\begin{equation*}
\exErr(\hPQ) \leq 2 \mathbb{E}_{\tJoinProb} \left|\regFct_{Q}(\featVar)-\frac{1}{2}\right|  \cdot \mathbbm{1} \left \{\left| \regFct_{Q}(\featVar)-\frac{1}{2} \right| \leq 2A \right \}  + 2 \tTsyCoeff (2 \epsilon)^{\tTsyExp+1}, 
\end{equation*}
where $A =  \frac{1}{\nn}\left| \sum_{i = 1}^{\nn} \nnLab{i}{} - \mathbb{E}[\nnLab{i}{} | \nnFeat{i}{}]  \right| + \frac{\holderCoeff}{\nn} \sum_{i = 1}^{\nn} \dist( \nnFeat{i}{} , x )^{\holderExp}. $

To bound $A$, we just follow the lines of our previous upper-bound analysis.
\end{proof}


\end{document}